\documentclass[12pt]{article}

\usepackage{newtxtext,newtxmath}

\usepackage{graphicx}
\usepackage{longtable}

\usepackage[margin=1in,letterpaper]{geometry}

\renewenvironment{abstract}
	{\quotation}
	{\endquotation}

\date{}

\makeatletter
\renewcommand{\fnum@figure}{\textbf{Fig. \thefigure}}
\renewcommand{\fnum@table}{\textbf{Table \thetable}}
\makeatother

\usepackage{scicite}

\usepackage[hyphens]{url}
\usepackage{xcolor}
\newtheorem{theorem}{Theorem}
\newtheorem{lemma}{Lemma}
\newtheorem{proposition}{Proposition}
\newtheorem{corollary}{Corollary}
\newtheorem{assumption}{Assumption}

\newcommand{\qedsymbol}{\ensuremath{\square}}
\newenvironment{proof}{\par\noindent\textit{Proof.}\ }{\hfill\qedsymbol\par}

\usepackage{xcolor}

\newcommand{\nl}[1]{{\color{black} #1}}
\newcommand{\vs}[1]{{\color{black} #1}}

\newcommand{\edit}[1]{\textcolor{black}{#1}}

\usepackage[hidelinks,hypertexnames=false]{hyperref}

\newcommand{\scititle}{Neuromorphic Decision Dynamics for Deadlock-free Safe Navigation in Congested Environments}
\title{\bfseries\boldmath\scititle}
\author{
Himani~Sinhmar$^{1\ast}$,
Vaibhav~Srivastava$^{2}$,
Naomi~Ehrich~Leonard$^{1}$\and
\small$^{1}$Department of Mechanical and Aerospace Engineering, Princeton University, NJ, USA.\and
\small$^{2}$Department of Electrical and Computer Engineering, Michigan State University, MI, USA.\and
\small$^\ast$Corresponding author. Email: himani.sinhmar@princeton.edu
}

\begin{document}
\maketitle

\noindent\textbf{Short title:} Safe navigation in congested spaces


\begin{abstract}\bfseries\boldmath
Safe navigation in congested environments requires both collision avoidance and coordinated passage. We introduce neuromorphic decision dynamics (NDD), a decentralized, event-based controller, that provides navigation in congested spaces with guarantees on safety, liveness, robustness, and scalability. NDD makes \edit{safe-and-live deconfliction} a controllable, locally tunable emergent property of local agent interactions, without requiring communication, preassigned priorities, models of neighbors' objectives or future actions, or estimates of group size or global network properties. By design, NDD couples the dynamics of internal decision and attention states with changing physical states and locally observed relative motion. Each agent continuously modulates its speed along its prescribed path, 
responding as soon and as much 
as conflict resolution requires. Agents self-organize into disjoint, locally engaged groups that resolve conflicts concurrently. When any group's conflict becomes safety-critical, decision-action feedback destabilizes collective indecision, and coordinated conflict-resolution emerges 
that stabilizes a safe traversal order within the group. 
Order persists as conflicts clear, but groups may split or merge as agents autonomously adapt to new conflicts and those with resolved conflicts resume their nominal motion.  
We derive explicit conditions on sensing range, feedback gains, response time scales, and actuation limits that guarantee collision avoidance and traversal-order formation in finite time with no deadlock or livelock for finite, heterogeneous agent populations of arbitrary, changing size. Computational cost scales with the number of local interactions sensed by each agent. {Large-scale simulations, comparisons with existing approaches, and hardware experiments demonstrate NDD in complex congested scenarios and 
validate its scalability.}


\end{abstract}

\noindent\textbf{One-sentence summary:} Neuromorphic decision dynamics enable robust safe-and-live multi-robot navigation with tunable emergent collective behavior.




\section*{Introduction}
Autonomous systems are increasingly operating in environments populated by other intelligent agents where, as a single agent decides and acts, it simultaneously shapes the choices available to others.
For moving agents navigating shared physical spaces, any one agent's movement decisions may influence and be influenced by the movement decisions of other nearby agents.
Moreover, bottlenecks will create local competition over limited available space, such as when movers (pedestrians, vehicles, robots) heading in different directions, and without a common protocol, seek to safely pass by one another through narrow passages, intersections, and busy, congested spaces.

Resolving the competition is challenging  because it requires coordination, not only to ensure safe (collision-free) passage, but also to avoid deadlocks and livelocks. Deadlocks can {manifest}  as mutual yielding~\cite{grover2020symmetry} or motion freezing when an agent deems every forward motion unsafe~\cite{trautman2010unfreezing}. Livelock occurs when there is repeated motion adjustment without progress~\cite{berg2008reciprocal}.
{When agents fail to form a \textit{traversal order} through the contested region, they can land in a deadlock or livelock.} Even immediate locally feasible and safe actions can eventually lead to  deadlock or livelock~\cite{zhou2020distributed}.

\textit{\edit{Safe-and-live deconfliction}} refers to collective behavior that preserves both safety and liveness in contested environments. To the best of our knowledge, a general decentralized design of \edit{safe-and-live deconfliction} that jointly guarantees safety and liveness without preassigned traversal priorities remains challenging, even when agents can communicate.
{In decentralized settings, coordination requires each agent to account for \textit{interaction structure} including:} (i) the set of others with which the agent is in conflict, (ii) how much each conflict matters, and (iii) which agent involved in the conflict should pass through the conflict region first. \textit{In this paper, we address decentralized \edit{safe-and-live deconfliction} with guarantees by designing neuromorphic decision dynamics that leverage interaction structure on-the-fly.} Our solution addresses the harder problem in which agents have no communication, only partial observations of their neighbors' actions, limited actuation, and the number of agents is 
finite but otherwise arbitrary and time-varying.

Existing coordination methods~\cite{trautman2010unfreezing,sadigh2016planning,schwarting2019social,roh2021multimodal,nakao2025winding}  often seek to simplify the problem by separating the assessment of interactions from movement decisions: an agent \textit{first predicts} the intended behavior of its neighbors and \textit{then selects} an action conditioned on those predictions. This reduces the interactive multi-agent problem to a single-agent optimal control problem. However, the approach rests on the strong assumption that the interaction structure needed for coordination can be determined \emph{before} an action is selected and remains valid while the agents react to one another. Congested environments violate this assumption because the same agent can be involved in multiple conflicts simultaneously, which makes the interaction structure an outcome of the movement decisions it is meant to inform. \edit{For example, an agent cannot predict a neighbor's response from their shared interaction alone because that response may depend on the agent's own action, which has yet to be selected.}
{As a result, the traversal orders agents arrive at need not agree, leading to mutual yielding or repeated readjustment.} Moreover, if an agent predicts conservatively and plans against the worst admissible behavior of every neighbor, it may find no forward motion safe and freeze. Near-symmetric encounters introduce additional ambiguity because several choices remain equally likely based on the available observations. Anticipating the resulting configurations and planning for an appropriate response become combinatorially intractable~\cite{hopcroft1984complexity,yu2013structure}.

The behavior of human crowds shows that coordination can emerge from \textit{fast, differentiated responses to local motion cues}~\cite{helbing1995social,moussaid2011simple}. In busy spaces, people usually keep moving along their preferred routes while making small speed adjustments: they slow slightly to let another person pass, accelerate into an opening, or time their motion through a merge or unsignalized intersection. They rarely need to stop completely or make a discrete change of direction.
{Instead, continuous speed modulations create gaps, resolve local conflicts, and allow traversal order to emerge~\cite{helbing1995social,moussaid2011simple}.}
{Inspired by these studies, we define \textit{coordination cues} to be local motion cues, such as speeding up, slowing down, and small differences in projected arrival times to contested regions.}

We are thus motivated by the following scientific question: \emph{Can a fully decentralized feedback controller utilize transient, endogenous coordination cues to generate provably-robust \edit{safe-and-live deconfliction} in interacting groups of arbitrary and changeable size?} Our thesis is that to successfully exploit coordination cues that emerge from interaction, \edit{safe-and-live deconfliction} must be designed to emerge from the closed loop between the decisions and actions of agents and the transient cues they generate.
Deadlocks and livelocks should not be treated as configurations that agents must detect and repair. Instead, the controller should, by construction, make deadlocks and livelocks unstable states of evolving interactions,  so they are autonomously destabilized by agents' movements.

Recent work on collision-avoidance control policies based on control barrier functions (CBF) also identified unstable equilibria as the key mechanism behind deadlock-free behavior in multi-agent navigation~\cite{jankovic2023instinctive,jankovic2022instability}.  However, the analysis is descriptive rather than constructive, limited to small groups, and certified only with centralized information or by co-optimizing agents' actions. Under purely local, model-free control policies, the deadlock equilibria remain stable, so the mechanism does not carry over to the decentralized setting considered here. According to~\cite{jankovic2022instability} it is an open problem ``how to design a method that creates unstable equilibria and controls how agents depart from unresolved configurations.''

We address the open problem by introducing \emph{Neuromorphic Decision Dynamics} (NDD), a dynamical system in continuous time for decision formation that synthesizes controlled instability by coupling decisions and actions across the interacting group. NDD \emph{constructs} and \emph{certifies} \edit{safe-and-live deconfliction} as a \emph{controllable emergent property} of local interactions. \vs{Its activation--persistence--recovery cycle mirrors the response of an excitable neural dynamical system, motivating the term \emph{neuromorphic decision dynamics}.} Each agent makes decisions locally using anonymous, egocentric measurements of nearby agents' relative positions and velocities.
{Whereas communication- and schedule-based approaches coordinate through shared information~\cite{jager2001decentralized,gregoire2013priority,dresner2008multiagent,divaio2019distributed}, and prediction- and game-theoretic approaches model neighbors' behavior~\cite{trautman2015robot,sadigh2016planning,fridovich2019efficient,le2022algames}, NDD couples agents through their observed motion.} It requires no communication, common clock, synchronized updates, global information, or models of neighboring agents' control objectives or future actions. NDD can interface with different robotic platforms and low-level tracking controllers that satisfy the stated actuation and tracking conditions. \edit{NDD preserves each agent's preferred path, continuously modulating only its traversal speed as conflict evidence arises and decays.}

{For an interacting group of arbitrary and time-varying size, we show analytically that NDD induces a traversal order {and sufficient temporal spacing} by means of a controlled group-level instability.} Locally tunable control parameters enable the emergence of group-level positive feedback
through which agents' speed changes alter the conflict evidence driving their decisions. As conflicts become safety-critical, this positive feedback 
overcomes decision damping, destabilizing the collective-indecision equilibrium in which the group has not selected a traversal order.
{Infinitesimally small endogenous asymmetries in local motion cues can then select a stable \emph{differentiated} motion pattern that establishes this order (Fig.~\ref{fig:ndd_deconfliction}, A to F).}
Decision values rapidly depart from a resting state near zero and remain differentiated as long as the conflict persists. The resulting motion pattern creates temporal separation by accelerating earlier-arriving agents, slowing later-arriving agents, or coordinating both responses. As this separation grows, the evidence of conflict weakens, and the decision values return to rest, restoring nominal motion. This allows agents to \emph{adapt} continuously as new local evidence arises in rapidly changing environments.
{The response requires no discrete proceed-or-yield decisions or rule-based switching logic.}

We derive explicit conditions on sensing range, decision and control gains, and actuation limits that ensure the controller is tuned to exploit emerging coordination cues before they decay and is robust under bounded sensing noise, uncertainty in predicted relative motion, and actuation error. We validate deployability of our framework in hardware experiments and demonstrate scalability in large-scale simulations of complex multi-agent scenarios.

\section*{Results}
Notation is collected in Supplementary Tables~\ref{tab:principal_notation}--\ref{tab:notation_spectral}.
\subsection*{Neuromorphic decision dynamics and decentralized speed control}

Each agent augments its physical state with an internal decision state and an attention state. Let $\boldsymbol  q_i$ denote the planar position of agent $i$, $\boldsymbol  v_i$ its physical velocity, and $v_i$ its realized traversal speed along its nominal path. For a path with unit tangent $ \boldsymbol  t_i(\boldsymbol  q_i)$, $\boldsymbol  v_i=v_i\boldsymbol  t_i(\boldsymbol  q_i)$. We assume each nominal path is smooth, i.e., has a continuously varying tangent direction and no abrupt changes in heading.

The real-valued decision state $z_i\in\mathbb R$ encodes the instantaneous tendency of agent $i$ to speed up ($z_i>0$) or slow down ($z_i<0$). The attention state $u_i\geq 0$ sets the gain applied to the input that drives $z_i$. Agent $i$ forms the time-varying observation state $\mathcal X_i$ from its velocity $\boldsymbol  v_i$, egocentric measurements of the relative positions $\boldsymbol  q_j-\boldsymbol  q_i$ and relative velocities $\boldsymbol  v_j-\boldsymbol  v_i$ of each neighbor $j\in\mathcal N_i(t)$, where $\mathcal N_i(t)$ is the set of agents in agent $i$'s sensing range at time $t$. Each sensed neighbor defines one directed interaction. Let $\mathcal E_i(t)=\{e=(i,j):j\in\mathcal N_i(t)\}$ be the interaction set of agent $i$ at time $t$, and let $\mathcal X_i^e$ contain the pairwise observations associated with interaction $e$. An interaction is \emph{conflicting} at time $t$ if, should the agents continue with their current motion, their spatial separation can fall below the required minimum $d^{\mathrm{safe}}>0$. Agent $i$ assigns to each {conflicting} interaction a signed evidence contribution whose magnitude reflects the spatial and temporal urgency of the potential conflict, weighted relative to its other simultaneous interactions, and whose sign indicates whether the interaction favors proceeding or yielding. Aggregating the contributions from all sensed {conflicting} interactions gives agent $i$'s scalar conflict evidence $h(\mathcal X_i)$.

\edit{Each agent's sensing radius must meet the lower bound in Proposition~\ref{prop:fixed_component_edge_safety}. We conservatively analyze a common radius equal to this bound, ensuring reciprocal sensing; agents with larger radii can begin responding earlier.}

\subsubsection*{Neuromorphic decision dynamics}
NDD builds on the nonlinear opinion dynamics (NOD) framework~\cite{gray2018multiagent,bizyaeva2023nonlinear,leonard2023fast}. The defining advance of NDD is the tight coupling of decision and physical dynamics: the nonlinear input $h(\mathcal X_i)$ driving each agent's decision dynamics depends implicitly on the decision states of other agents through the motion they generate, rather than through an explicit communication network. Denote $\dot x:= dx/dt$.  Agent $i$ evolves its decision and attention states 
continuously over time according to
\begin{subequations}
\label{eq:ndd}
\begin{align}
\tau_z\dot z_i
&=
-\rho_z z_i+S\bigl(u_i\,h(\mathcal X_i)\bigr),
\label{eq:ndd:decision}\\
\tau_u\dot u_i
&=
-\rho_u u_i+u_0+\kappa_u z_i^2.
\label{eq:ndd:attention}
\end{align}
\end{subequations}
Here, $S$ is a bounded odd sigmoid with $S'(0)=1$; we use $S(x)=\tanh(x)$. The parameters $\tau_z$ and $\tau_u$ set the decision and attention time scales, and $\rho_z$ and $\rho_u$ set their decay rates. The ratio $u_0/\rho_u$ is the basal attention level. The coefficient $\kappa_u$ sets the strength of decision-dependent self-reinforcement during decision formation. 

\subsubsection*{Decentralized control policy}
Agent $i$'s decision state adjusts its commanded traversal speed without changing its path:
\begin{subequations}
\label{eq:speed}
\begin{align}
v_i^{\mathrm{cmd}}
&= v_i^{\mathrm{nom}} + V(z_i),
\qquad V(0)=0,
\label{eq:cmd_speed}\\
\tau_v\dot{v}_i
&= f_{v}\left(v_i,v_i^{\mathrm{cmd}}\right),
\label{eq:speed_response}
\end{align}
\end{subequations}
where $v_i^{\mathrm{nom}}$ is agent $i$'s nominal speed provided by the upstream planner and $v_i$ is its realized traversal speed along the nominal path. The map $V$ is continuous, bounded, monotonically increasing, and sign-preserving, so positive decisions raise the commanded speed and negative decisions lower it (see the decision and speed trajectories in Fig.~\ref{fig:ndd_deconfliction}, C and D, respectively). Equation~\eqref{eq:speed_response} represents the low-level speed tracking, with time scale $\tau_v>0$ and agent-specific tracking dynamics $f_v$ that account for actuation limits. This formulation accommodates heterogeneous platforms and low-level controllers, provided they track admissible speed commands with a known bounded transient time. Because the decision input and speed command depend only on locally available observations, the feedback control policy remains fully decentralized.

{Figure~\ref{fig:overview} illustrates the closed-loop architecture, which composes planning, sensing, embodied decision dynamics, and physical control across distinct time scales. Within this hierarchy, NDD serves as a modular, event-based dynamic controller that resolves generic encounters using speed modulation along the prescribed path, invoking computationally intensive path replanning only sparsely---such as for direct head-on encounters. Much like human pedestrians who maintain their routes while adjusting pace, NDD bypasses the inefficiency and intractability of jointly replanning trajectories for every encounter (see Supplementary Text S1 for a practical implementation guide).}

\subsection*{Conflict evidence from local observations}

\paragraph{Contested region of an interaction}
For each interaction $e=(i,j)\in\mathcal E_i(t)$, agent $i$ uses its 
 observation $\mathcal X_i^e(t)$ of its position and velocity and agent $j$'s relative position and velocity at time $t$ to construct the pairwise contested region at time $t$.  
 The instantaneous projected heading of agent $i$ (agent $j$) is the ray that originates at agent $i$'s (agent $j$'s) position and extends along its velocity vector at time $t$, {also given by the path tangent $\mathbf{t}(\mathbf{q}_i)$ $(\mathbf{t}(\mathbf{q}_j))$.}
 The intersection of these rays defines the geometric crossing point $\boldsymbol  q_{\mathrm c}^e$ and crossing angle $\alpha^e$ (Fig.~\ref{fig:contested_regions}, A).
The contested segment on each ray is the set of points along the ray that are a distance less than $d^{\mathrm{safe}}$ from the other ray.
This conservative construction includes every point where the projected paths could violate the required spatial separation. The two segments form the contested region of interaction $e$ at time $t$. 

For each agent $k\in\{i,j\}$, upstream and downstream endpoints, $k^-$ and $k^+$, lie at a distance $r_{\mathrm c}^e=d^{\mathrm{safe}}/\sin\alpha^e$ from $\boldsymbol  q_{\mathrm c}^e$ and define the entry and exit boundaries, respectively (Fig.~\ref{fig:contested_regions}, B). The along-ray distance and speed of agent $k$ determine its projected entry and exit times, $t_{k^-}^{e}$ and $t_{k^+}^{e}$. These times define the occupancy interval $\mathcal I_k^e=[t_{k^-}^{e},t_{k^+}^{e}]$. The interaction is conflicting when the two occupancy intervals overlap, $\mathcal I_i^e\cap\mathcal I_j^e\neq\varnothing$. Overlap indicates that the agents can simultaneously occupy their contested segments and may violate the required clearance. Disjoint occupancy intervals instead define a safe temporal traversal order. Shallower crossing angles increase $r_{\mathrm c}^e$ and therefore lengthen both the contested segments and their projected occupancy intervals.

By assumption, if $e=(i,j) \in \mathcal E_i(t)$, then $\bar e = (j,i) \in\mathcal E_j(t) $, and $j$ will calculate the same contested region as $i$ at time $t$  from its observation $\mathcal X_j^{\bar e}(t)$. 

At a given time $t$, an agent may participate in several interactions and therefore encounter several pairwise contested regions along its projected path.

\paragraph{Local conflict evidence}
At time $t$, agent $i$ computes its conflict evidence by aggregating  signed contributions from its simultaneous interactions:
\begin{equation}
h(\mathcal X_i)
=
\sum_{e\in\mathcal E_i(t)}
P(\mathcal X_i^e)\,
W(\mathcal X_i^e;\mathcal X_i)\,
G(\mathcal X_i^e).
\label{eq:ndd:evidence}
\end{equation}
Each factor in Equation~\eqref{eq:ndd:evidence} has a distinct role. {See Materials and Methods (\textit{Smooth speed-control and conflict-evidence design}) for details on constructing these factors from local kinematics.}
 
 The conflict-pressure function $P(\mathcal X_i^e)\in\mathbb R_{\geq0}$ quantifies the {severity} of interaction $e$. It increases when the projected closest-approach separation decreases or when the temporal overlap of the occupancy intervals grows, thereby distinguishing weak from safety-critical interaction regimes (\edit{Lemmas~\ref{prop:uniform_pressure_bands} and~\ref{prop:pressure_response_tuning}}).

The pressure-weight function $W(\mathcal X_i^e;\mathcal X_i)\in[0,1]$ assigns relative importance to agent $i$'s simultaneous interactions. It normalizes the weights so that $\sum_{e\in\mathcal E_i(t)}W(\mathcal X_i^e;\mathcal X_i)=1$. A softmax over the conflict pressures is a candidate choice for $W$.

The arrival-order function $G(\mathcal X_i^e)\in[-1,1]$ assigns the sign. 
It maps the entry-time difference $t_{j^-}^{e}-t_{i^-}^{e}$ to a positive value when the instantaneous kinematics favor agent $i$ traversing first, a negative value when they favor agent $j$, and zero at a tie. Choosing $G$ to be odd gives $G(\mathcal X_j^{\bar e})=-G(\mathcal X_i^e)$ for $\bar e=(j,i)$.

We define the interaction weight as $\omega_i^e(t)=P(\mathcal X_i^e)\,W(\mathcal X_i^e;\mathcal X_i)$, combining the severity of conflict $e$ with its relative importance among agent $i$'s simultaneous sensed interactions. The product $\omega_i^e(t)G(\mathcal X_i^e)$ is the weighted arrival-order cue contributed by interaction $e$. Summing these signed contributions gives agent $i$'s conflict evidence, $h(\mathcal{X}_i)=\sum_{e\in\mathcal E_i(t)}\omega_i^e(t)G(\mathcal X_i^e)$.

\subsection*{Graph formulation of \edit{safe-and-live deconfliction}}
\label{sec:results_graph_decomposition}

{Physical conflicts} define the time-varying weighted directed conflict graph
$$\mathcal G(t)=(\mathcal V(t),\mathcal E^{\mathrm{conf}}(t),\boldsymbol \omega(t)),$$
where $\mathcal V(t)$ contains the agents in conflicts, $\mathcal E^{\mathrm{conf}}(t)$ contains the directed interactions associated with those conflicts, and $\boldsymbol \omega(t)$ collects the weights $\omega_i^e(t)$ for all directed interactions $e=(i,j)\in\mathcal E^{\mathrm{conf}}(t)$. Under the convention $e=(i,j)$, an edge $i\rightarrow j$ identifies a physical conflict for which the evidence of agent $i$ is influenced by the motion of agent $j$. The graph records the interactions that currently impose traversal-order constraints (see Figure~\ref{fig:graph_hierarchy}). 

\paragraph{Asymmetric self-allocation of responsibility for conflict resolution}
Agents self-allocate conflict-resolution responsibility through their locally normalized interaction weights, which can be asymmetric even at a pairwise arrival-time tie ($t_{i^-}^{e}=t_{j^-}^{e}$).  Each physical conflict generates reciprocal interactions $e=(i,j)$ and $\bar e=(j,i)$ with shared geometry and pressure. Away from a tie, the reciprocal arrival-order values have opposite signs. A tilted softmax (Equations~\eqref{eq:tilted_pressure}--\eqref{eq:pressure_softmax}) further biases responsibility toward the later-arriving agent. At a pairwise tie, both arrival-order values are zero, but the weights $\omega_i^e$ and $\omega_j^{\bar e}$ can still differ because each agent normalizes against its own simultaneous interactions. Thus, the same conflict can receive different levels of responsibility from the two agents without communication, preassigned priorities, or global normalization. 
In Fig.~\ref{fig:graph_hierarchy}, B, agent~7 weights its dominant conflict with agent~2 highly, whereas agent~2 prioritizes conflicts with agents~1 and~3.

\paragraph{Self-organization into locally engaged groups}
Every conflict edge in $\mathcal G(t)$ has its reciprocal, and weights on edges between a pair are zero only when the pair has no conflict. Thus, $\mathcal G(t)$ will be the union of strongly connected components, each
defining a \emph{locally engaged group} (see Fig.~\ref{fig:graph_hierarchy}, C). \edit{Strong connectivity means that every agent in the group can influence every other agent through a directed path.} 
Distinct groups have no physical conflicts between them, although they may be weakly coupled through sensed nonconflicting interactions, including those involving other agents.
Under the admissible weak intergroup-influence conditions of Supplementary Sec.~\ref{sec:supp_spectral_tuning}, groups can negotiate and resolve their conflicts concurrently without predefined group membership or external coordination.

A \emph{group tie} occurs when every {physical conflict} in the group is pairwise tied, so all {conflict} contributions vanish. 
The Results subsection \textit{Collective interaction dynamics} shows how NDD destabilizes a group tie and stabilizes a safe traversal order.

\paragraph{Desired emergent \edit{safe-and-live deconfliction} behavior}
The desired emergent behavior has three features: 1) every {conflicting pair} remains collision-free and establishes a strict, robust traversal order through its contested region; (2) the resulting pairwise orders are mutually consistent within each locally engaged group, meaning that the traversal-order graph formed by \edit{orienting each physical-conflict pair from the agent selected to pass first to the agent that follows} is acyclic; and (3) \edit{all conflicts in each admissible episode clear in finite time}. The safe-speed and safe-decision sets defined in Materials and Methods (\textit{Formulation of \edit{safe-and-live deconfliction}}) formalize the pairwise temporal-clearance requirement, while the following analysis shows how local NDD interactions establish the required traversal orders, make them mutually consistent within each group, and ensure finite progress across the population.

\subsection*{Collective interaction dynamics} 
\label{sec:results_collective_dynamics}
For a locally engaged group of $N$ agents, let $\boldsymbol  q$, $\boldsymbol  v$, $\boldsymbol  z$, $\boldsymbol  u$, and $\boldsymbol  h$ denote the stacked positions, realized speeds, decision states, attention states, and conflict evidence. The collective closed-loop dynamics are
\begingroup
\begin{equation}
\label{eq:stacked_ndd}
\begin{aligned}
\tau_z\dot{\boldsymbol  z}
&=
-\rho_z\boldsymbol  z
+
S\left(
\boldsymbol  u\odot
\boldsymbol  h\left(\mathcal X(\boldsymbol  q,\boldsymbol  v)\right)
\right),
\\
\tau_u\dot{\boldsymbol  u}
&=
-\rho_u\boldsymbol  u
+u_0\boldsymbol  1
+\kappa_u\left(\boldsymbol  z\odot\boldsymbol  z\right),
\\
\tau_v\dot{\boldsymbol  v}
&=
\boldsymbol  f_v\left(
\boldsymbol  v,
\boldsymbol  v^{\mathrm{nom}}+\boldsymbol  V(\boldsymbol  z)
\right),
\end{aligned}
\end{equation}
\endgroup
where $\odot$ denotes the Hadamard product, $S$ and $V$ act componentwise, and $\boldsymbol  f_v$ stacks the agent-specific tracking dynamics $f_v$. The commanded speed vector is $\boldsymbol  v^{\mathrm{nom}}+\boldsymbol  V(\boldsymbol  z)$, which the realized speed vector $\boldsymbol  v$ tracks. The positions evolve along the nominal paths according to $\dot{\boldsymbol  q} = \operatorname{col}\left( v_1{\boldsymbol  t}_1(\boldsymbol {q}_1), \ldots, v_N{\boldsymbol  t}_N(\boldsymbol  q_N) \right)$.

{These equations form a multitime-scale continuous-time decision ($\boldsymbol z$)--action ($\boldsymbol v$)--perception ($\mathcal X$) closed feedback loop. Decisions change commanded speeds through $\boldsymbol  V(\boldsymbol  z)$, the tracking dynamics produce realized motion, and that motion changes the observations and conflict evidence driving subsequent decisions. Attention regulates the response to this evidence. Agents are coupled through locally sensed motion without exchanging internal states or intended actions. For a sensed interaction $e=(i,j)$, this coupling follows} \( (z_i,z_j)\longrightarrow(v_i,v_j)\longrightarrow\mathcal X_i^e \longrightarrow h_i(\mathcal X_i)\longrightarrow z_i, \) with $(u_i,u_j)$ regulating the decision responses.

\paragraph{Collective recurrent-network interpretation} The closed loop can therefore be viewed as an embodied continuous-time recurrent network. Each agent is a dynamical node with state $(z_i,u_i,v_i)$, and {sensed interactions} form directed edges whose strengths and connectivity evolve with agent motion. The leaky decision and attention dynamics and stable speed tracking retain and gradually attenuate the effects of past evidence and commands, giving the network a fading memory. The same instantaneous observation can therefore produce different responses depending on the internal state inherited from the interaction history.

\paragraph{Time-scale separation and emergence of traversal order} The decision--action feedback must establish a traversal order before the agents traverse an appreciable fraction of the distance to their contested segments. The characteristic response time of the coupled decision, attention, and speed dynamics is
$T_{\mathrm{fast}}=\max\{ \tau_z/\rho_z, \tau_u/\rho_u, \tau_v\}$. Let $\mathcal I^{\mathrm{neg}}=[t_0,t_0+T^{\mathrm{neg}}]$ denote the negotiation interval, from conflict detection at $t_0$ until a traversal order is established at $t_0+T^{\mathrm{neg}}$. Successful negotiation requires
\begin{equation}
T_{\mathrm{fast}}
\lesssim
T^{\mathrm{neg}}
\ll
T_{\mathrm{trav}},
\label{eq:time_scale_separation}
\end{equation}
where $T_{\mathrm{trav}}$ is the conservative travel time available from conflict detection before any agent in the group enters its contested segment. 
The first relation expresses a response-time lower bound on the negotiation duration $T^{\mathrm{neg}}$, during which the agents establish a traversal order
through decision-induced speed adjustments and feedback from the resulting arrival-order cues. 
The second relation ensures that this order is established before agents travel an appreciable fraction of the distance to their contested segments.
\edit{Actuation limits constrain how quickly speed differences develop, so slower actuation requires earlier detection and a larger sensing range.} \edit{Lemmas~\ref{prop:pressure_response_tuning} and~\ref{prop:sensing_range_detection_margin}} quantify these joint response and sensing requirements.


To show how traversal-order forms on the fast time scale, we hold the traversal distances and the set of sensed interactions fixed, while retaining the dependence of conflict evidence on the evolving realized speeds and, through the decision-to-speed map, on decisions. 
Under the conditions of \edit{Lemma~\ref{prop:pressure_response_tuning}}, 
pressure on conflicting edges remains high and varies little throughout the negotiation time interval. This sustained pressure supports the decision--action feedback as decision, attention and speed dynamics evolve and agents develop differentiated projected entry times into their contested regions. 

We next derive the interaction-sensitivity matrix, which quantifies how decision-induced speed changes affect conflict evidence $h(\mathcal X)$. We then analyze the equilibria of the collective interaction dynamics to show how agents in a group establish a safe traversal order, where safety is ensured by sufficiently large temporal spacing between successive agent entry times into the contested region.

\paragraph{Interaction sensitivity} {Given an observation snapshot $\mathcal X$, let $\mathcal X_{\boldsymbol  z}$ denote the observations at realized speeds $\boldsymbol  v^{\mathrm{nom}}+\boldsymbol  V(\boldsymbol  z)$.}
The Jacobian of the conflict evidence with respect to the decision states defines the \emph{interaction-sensitivity matrix}
\begin{equation}
\boldsymbol  A(\mathcal X,\boldsymbol  z)
:=
D_{\boldsymbol  z}
\boldsymbol  h(\mathcal X_{\boldsymbol  z})
=
D_{\boldsymbol  v}
\boldsymbol  h(\mathcal X_{\boldsymbol  z})
D_{\boldsymbol  z}\boldsymbol  V(\boldsymbol  z).
\label{eq:evidence_jacobian}
\end{equation}
Because $\boldsymbol  V$ acts componentwise, $D_{\boldsymbol  z}\boldsymbol  V(\boldsymbol  z)=\operatorname{diag}\left(V'(z_1),\ldots,V'(z_N)\right)$, and the individual sensitivities are
\begingroup
\begin{equation*}
\begin{aligned}
A_{ik}(\mathcal X,\boldsymbol  z)
&=
V'(z_k)
\frac{\partial h(\mathcal X_i)}{\partial v_k}=
V'(z_k)
\sum_{e\in\mathcal E_i}
\left(
G(\mathcal X_i^e)
\frac{\partial\omega_i^e}{\partial v_k}
+
\omega_i^e
\frac{\partial G(\mathcal X_i^e)}{\partial v_k}
\right).
\end{aligned}
\end{equation*}
\endgroup
The diagonal entry $A_{ii}$ is the \emph{proprioceptive sensitivity} of agent $i$'s evidence to its own decision-induced motion. An off-diagonal entry $A_{ik}$, $k\neq i$, is the \emph{exteroceptive sensitivity} to agent $k$'s decision-induced motion. Because the observations, pressures, and weights evolve with the collective state, the interaction-sensitivity matrix is itself state-dependent.

\paragraph{Destabilization of collective indecision and stabilization of traversal order}
When a locally engaged group reaches a group tie, the arrival-order evidence provides no preference for either speeding up or slowing down. The state $\boldsymbol  z=\boldsymbol 0$, at which the agents retain their nominal speeds, is therefore an equilibrium that we call \emph{collective indecision}. Inherited nonzero decisions decay toward this equilibrium while it is stable. However, as conflict pressure rises above an implicit threshold, determined by design parameters, collective indecision loses stability and equilibria with differentiated decision states emerge through a local bifurcation. The resulting differentiated decisions drive differentiated speed changes, which separate the agents' projected entry times and establish a {traversal order} through the contested region with safe spacing. The order persists during safety-critical interactions. Once the agents clear their contested regions, the conflict evidence decreases, the decision states relax toward zero, and the agents return to their nominal speeds.
 
This behavior prevents two unresolved collective responses: \emph{deadlock}, which occurs when all agents uniformly yield and progress stops, and \emph{livelock}, which occurs when the agents continue adjusting their motion without sustaining a traversal order. A stable traversal order, established by the differentiated equilibrium of the collective interaction dynamics, ensures \textit{liveness} by excluding both behaviors, while sufficient relative temporal separation ensures \textit{safety} throughout traversal.

We determine when collective indecision loses stability and stable differentiated decision patterns emerge using a local bifurcation analysis of the collective interaction dynamics in Equation~\eqref{eq:stacked_ndd}. Let $\boldsymbol \Theta$ collect the tunable parameters of the conflict-evidence functions, decision and attention dynamics, and decision-to-speed map. For a fixed observation snapshot $\mathcal X$, the equilibria satisfy
\begin{equation}
\boldsymbol  0
=
\boldsymbol  F(\boldsymbol  z;\mathcal X,\boldsymbol \Theta)
:=
-\rho_z\boldsymbol  z
+
S\left(
\left(
\frac{u_0}{\rho_u}\boldsymbol  1
+
\frac{\kappa_u}{\rho_u}
(\boldsymbol  z\odot\boldsymbol  z)
\right)
\odot
\boldsymbol  h(\mathcal X_{\boldsymbol  z})
\right).
\label{eq:reduced_steady_state}
\end{equation}
At a group tie $\boldsymbol  h=\boldsymbol 0$, so the collective-indecision equilibrium $\left( \boldsymbol  z^\ast, \boldsymbol  u^\ast, \boldsymbol  v^\ast \right) = \left( \boldsymbol  0, \frac{u_0}{\rho_u}\boldsymbol  1, \boldsymbol  v^{\mathrm{nom}} \right)$ is a solution of~Equation~\eqref{eq:reduced_steady_state}. Varying $\boldsymbol \Theta$ changes the locations and stability of the equilibria. {Therefore, identifying the parameter values at which collective indecision changes stability provides a basis for tuning NDD so that weak interactions preserve collective indecision and safety-critical interactions trigger decision-making.} {Local stability of collective indecision is determined by the eigenvalues of the Jacobian evaluated at collective indecision:}
\begin{equation}
D_{\boldsymbol  z}\boldsymbol  F(\boldsymbol  0;\mathcal X,\boldsymbol \Theta)
=
-\rho_z\boldsymbol  I
+
S'(0)\frac{u_0}{\rho_u}\boldsymbol  A(\mathcal X_z,\boldsymbol  0).
\label{eq:steady_state_jacobian}
\end{equation}
{The Jacobian in Equation~\eqref{eq:steady_state_jacobian} is an affine function of the interaction-sensitivity matrix $\boldsymbol  A$ computed at $\boldsymbol  z^\ast = \boldsymbol  0$, so its eigenvalues follow directly from those of $\boldsymbol  A$. A simple eigenvalue Equation~\eqref{eq:steady_state_jacobian} crossing zero marks the stability boundary above which collective indecision is destabilized, and equilibrium branches with differentiated decision states can emerge.} If $\lambda$ is an eigenvalue of $\boldsymbol  A$ with right eigenvector $\boldsymbol \phi$, the corresponding eigenvalue of~Equation~\eqref{eq:steady_state_jacobian} is $-\rho_z+S'(0)(u_0/\rho_u)\lambda$, which crosses zero at
\begin{equation}
\lambda^\ast
:=
\frac{\rho_z\rho_u}{S'(0)u_0}.
\label{eq:formation_spectral_condition}
\end{equation}
{For reciprocal directed interactions, the interaction-sensitivity matrix $\boldsymbol  A(\mathcal X_z,\boldsymbol  0)$ generically has a simple dominant eigenvalue $\lambda_{\max}$ (\edit{Proposition~\ref{prop:interaction_sensitivity_spectrum}}), i.e., $\lambda_{\max}$ will be the first  eigenvalue of $\boldsymbol  A(\mathcal X_z,\boldsymbol  0)$ to  reach $\lambda^\ast$. For $\lambda_{\max}<\lambda^\ast$, decision damping dominates the first-order {(positive)} interaction feedback. At $\lambda_{\max}=\lambda^\ast$, the two effects balance, causing the steady-state Jacobian $D_{\boldsymbol  z}\boldsymbol  F(\boldsymbol  0)$ to become singular and marking the bifurcation point.}

{ To describe the equilibria near the bifurcation point, we define the parameter $\mu:=\lambda_{\max}-\lambda^\ast$ and let \(\bar{\boldsymbol \tau}\) collect the projected entry-time offsets at collective indecision. The nearby steady-state condition is then $\boldsymbol  F(\boldsymbol  z;\mu,\bar{\boldsymbol \tau})=\boldsymbol 0$. At the bifurcation point, differentiated equilibrium branches emerge tangent to the right eigenvector $\boldsymbol \phi_{\max}$ associated with $\lambda_{\mathrm max}(\boldsymbol  A)$. The Lyapunov--Schmidt reduction in Supplementary Sec.~\ref{sec:supp_ls_split} characterizes these solution branches near the collective-indecision equilibrium \(\boldsymbol  z^\ast=\boldsymbol 0\) as
\begin{equation}
\boldsymbol  z^\dagger(\mu,\bar{\boldsymbol \tau})
=
a^\dagger(\mu,\bar{\boldsymbol \tau})
\boldsymbol \phi_{\max}
+\mathcal O\!\left((a^\dagger)^2+\|\bar{\boldsymbol\tau}\|\right),
\label{eq:main_differentiated_branch}
\end{equation}
where $a^\dagger$ is an algebraic function directly related to $\mu$ and $\bar \tau$. The higher-order terms in~\eqref{eq:main_differentiated_branch} govern the shape, saturation, and local stability of the emerging branches.
\edit{Proposition~\ref{prop:interaction_sensitivity_spectrum}} establishes that $\boldsymbol \phi_{\max}$ has both positive and negative components. Equation~\eqref{eq:main_differentiated_branch} therefore implies that, for sufficiently small nonzero \(a^\dagger\), the emerging decision state $\boldsymbol  z^{c,\dagger}$ is differentiated. \edit{To leading order, $a^\dagger\phi_{\max,i}$ determines agent $i$'s decision: positive and negative values command speeds above and below nominal, respectively, and the absolute value sets the decision magnitude.} For $\lambda_{\max}>\lambda^\ast$, the collective-indecision equilibrium $\boldsymbol z^\ast=\boldsymbol  0$ is unstable, and the differentiated equilibrium branch $\boldsymbol  z^\dagger \neq \boldsymbol  0$ is locally attracting (see the local bifurcation analysis in Supplementary Sec.~\ref{sec:supp_local_bifurcation}).}

{Generic closed-loop asymmetries---such as unequal projected distances or entry times, measurement variations, asynchronous updates, or differences in tracking transients---perturb the fast state away from the collective indecision equilibrium $\boldsymbol z^\ast=\boldsymbol  0$, select an attracting differentiated branch with equilibrium $\boldsymbol  z^\dagger \neq \boldsymbol  0$, and bring the fast state into its attraction neighborhood within $T^{\mathrm{neg}}$. Only an exactly symmetric collective state initialized precisely at $\boldsymbol  z^\ast=\boldsymbol  0$ would remain at the unstable indecision equilibrium, and this is practically impossible since any small asymmetry or variation drives it away.} Through $\boldsymbol  V$, a collective decision state $\boldsymbol  z^\dagger \neq \boldsymbol  0$ produces unequal commanded speeds. \edit{Propositions~\ref{prop:strict_acyclic_traversal_order} and~\ref{prop:fixed_component_edge_safety}} show that these differences produce a mutually consistent traversal order with safe temporal separation, while Theorem~\ref{thm:fixed_component_safe_order} establishes finite progress without deadlock or livelock.

{As the interaction evolves, changes in $\mathcal X_{\boldsymbol  z}$ alter the pairwise pressures and normalized weights, and hence $\boldsymbol  A$ and its leading eigenpair $(\lambda_{\max},\boldsymbol \phi_{\max})$. We tune the pressure, weighting, and basal-attention parameters, just once and in advance, so that safety-critical interactions autonomously drive $\lambda_{\max}$ across $\lambda^\ast$, destabilizing collective indecision whenever conflict resolution requires differentiation (Theorem~\ref{thm:autonomous_spectral_tuning}).}

\subsection*{Finite-time guarantees for decentralized \edit{safe-and-live deconfliction}}
\label{sec:results_safe_deconfliction}
\edit{The following theorem states the safety and finite-progress guarantees for populations in which every agent implements NDD.}

\paragraph{Admissible conflict episodes and scope}
\edit{We consider a cooperative population in which each agent follows its prescribed path and local control law, without inter-agent communication or access to neighbors' internal states. A conflict episode---a maximal time interval with at least one active physical conflict --- is \emph{admissible} if its conflicts are physically resolvable by speed modulation and each agent senses a uniformly bounded number of neighbors at any time.} Admissibility also requires: 1) crossing angles bounded away from zero, to exclude head-on encounters; 2) bounded tracking errors, actuation delays, and stopping distances; 3) a positive minimum traversal speed attainable by every leading mover; and 4) finitely many conflict events without Zeno switching (Assumption~\ref{ass:episode}). 
Finally, traversal-order negotiation also requires a 
small symmetry-breaking cue---arising naturally from asynchronous updates, heterogeneous actuation or tracking, measurement noise, or geometric asymmetries---to generate the entry-time offsets that select a differentiated equilibrium branch.

\begin{theorem}\textnormal{\textbf{Decentralized continuous-time globally \edit{safe-and-live deconfliction}.}}
\label{thm:finite_time_safe_deconfliction}
\edit{For any generic admissible conflict episode and the joint design conditions of Supplementary Sec.~\ref{sec:supp_joint_design} for a finite, possibly time-varying population of heterogeneous agents evolving in continuous time, there exist finite sensing radii and fixed local controller parameters, independent of population and engaged-group sizes, that guarantee:}
\begin{enumerate}
\item \emph{Decentralized and scalable operation:} \edit{each agent uses only its state and local relative-motion measurements, with computation linear in its bounded number of sensed interactions. No communication, centralized scheduling, online group-state or interaction-structure estimation, or combinatorial traversal-order assignment is required.}

\item \emph{Continuous safety:} all agents maintain the required separation for all time:
$\|\boldsymbol  q_i(t)-\boldsymbol  q_j(t)\|\geq d^{\mathrm{safe}}$,
for all $t$ and for all $\{i,j\}$.

\item \emph{\edit{Emergence of traversal orders:}}
\edit{For each unresolved group whose physical-conflict set remains fixed, the group establishes, during the time interval $T^{\mathrm{neg}}$, a strict acyclic traversal order with the required temporal-clearance margins. Each pairwise traversal order persists until its conflict clears. Physical clearance and group splitting preserve established orders on remaining conflicts. After new conflict activations, including those involving agents entering the population or merging groups, the emergent traversal order continuously adapts while preserving safety; the same safe-and-live deconfliction guarantees apply to the updated conflict set under the same conditions.}

\item \emph{Finite progress and termination:} \edit{All conflicts in the episode clear in finite time. Distinct groups and leading movers that share no physical conflict may progress concurrently.}

\item \emph{Release to nominal motion:} \edit{Once all of agent $i$'s physical conflicts in the episode clear, its states converge as residual evidence vanishes, provided no new conflicts involving $i$ arise:}
\[
z_i(t)\to0,
\qquad
u_i(t)\to\frac{u_0}{\rho_u},
\qquad
v_i(t)\to v_i^{\mathrm{nom}}.
\]
\end{enumerate}
\end{theorem}

The proof is given in Supplementary Sec.~\ref{sec:supp_proof_global}. 

\paragraph{Roadmap of guarantees and \edit{mechanisms underlying decentralized} \edit{safe-and-live deconfliction}}

\textit{Negotiation for local safe traversal-order formation:} Safe traversal-order formation requires enough time to negotiate while preserving collision avoidance. (\edit{Lemma~\ref{prop:sensing_range_detection_margin}}). 
\edit{Lemma~\ref{prop:pressure_response_tuning}} gives conditions for choosing the spatial and temporal pressure gains to maintain sustained, slowly varying safety-critical pressure \edit{during negotiation.} With the basal attention tuning of Theorem~\ref{thm:autonomous_spectral_tuning}, this pressure drives $\lambda_{\max}(\boldsymbol A)$ above the threshold $\lambda^\ast$ in Equation~\eqref{eq:formation_spectral_condition}, destabilizing collective indecision.

As collective indecision destabilizes, small asymmetries can grow through decision--action feedback. Even an infinitesimal projected-entry-time difference produces pairwise evidence that favors yielding by the later-arriving agent. This does not yet establish a safe traversal order: the agents' projected occupancy intervals can still overlap. However, under the yielding and stopping conditions of \edit{Proposition~\ref{prop:following_decision_invariance}}, the instantaneous follower remains outside its contested segment until safe passage is possible, with a certified stop if needed. Several agents may therefore yield simultaneously in response to their respective conflicts. In practice, agents yield by different amounts, creating or modifying projected entry-time offsets. \edit{Developing timing offsets can then unfold} the bifurcation and bias branch selection (Fig.~\ref{fig:supp_cubic_bifurcation_comparison}), while infinitesimal decision asymmetries with a nonzero component along the critical eigenmode can select a branch even in the absence of timing bias. Under the conditions of Theorem~\ref{thm:autonomous_spectral_tuning}, the \edit{fast subsystem state} reaches the attraction neighborhood of the selected differentiated branch within $T^{\mathrm{neg}}$. The resulting steady-state speed differences define a strict acyclic traversal order with safe temporal separation before any agent enters its contested segment (\edit{Propositions~\ref{prop:strict_acyclic_traversal_order} and~\ref{prop:fixed_component_edge_safety}}).
Figure~\ref{fig:ndd_deconfliction}, C and D, show the transition from initial yielding to differentiated decisions and speeds. Figure~\ref{fig:ndd_deconfliction}, E plots the differentiated decision magnitude $\|\widetilde{\boldsymbol z}\|=\bigl[\sum_{i=1}^{6}(z_i-\bar z)^2\bigr]^{1/2}$, where $\bar z=\frac{1}{6}\sum_{i=1}^{6}z_i$. Subtracting the mean isolates differences among the six agents' decisions from a common response such as collective yielding during negotiation (panel~F). \edit{Together, panels E and F show an emergent population-level spiking--hold--recovery response: the mean decision returns to near zero during hold even as the differences among agents' decisions persist.}
When decision formation is faster than speed tracking, decisions may briefly continue to differentiate before realized speeds establish temporal separation, producing the peak in $\|\widetilde{\boldsymbol z}\|$ in Fig.~\ref{fig:ndd_deconfliction}, E.
\newline

\textit{Persistence of local safe traversal orders}: Conflict evidence is continuously recomputed from local observations, allowing the closed loop to adjust decision differentiation as conflicts evolve. At the selected attracting equilibrium, the attention-modulated aggregate decision input balances decision damping. As projected conflict severity decreases, lower pressure can weaken the decision input and allow speeds to relax toward nominal values. \edit{If this relaxation increases projected conflict severity}, pressure can rise again (\edit{Lemma~\ref{prop:uniform_pressure_bands}}), reinforcing decision differentiation along the selected branch. These corrective responses autonomously regulate the magnitude of differentiation around the selected equilibrium, which shifts along the attracting branch as the agents move and the interaction weights evolve slowly. The \edit{fast subsystem state} tracks this moving equilibrium while preserving the selected traversal order
(\edit{Lemma~\ref{prop:smooth_dominant_eigenvector_branch}}; Theorem~\ref{thm:autonomous_spectral_tuning}). \edit{Proposition~\ref{prop:fixed_component_edge_safety}} supplies the corresponding temporal-clearance conditions for safe passage. The response adjusts smoothly without abrupt speed switching, while nonconflicting agents and disjoint groups progress concurrently (Theorem~\ref{thm:fixed_component_safe_order}).  The nonzero hold-phase plateau in Figure~\ref{fig:ndd_deconfliction} E illustrates sustained decision differentiation with subsequent gradual decline in differentiated response as conflicts clear physically. \newline

\textit{Continuous-time graph reconfiguration:}
Once the leader clears its contested segment, the pair associated with that segment no longer needs to maintain a traversal order, and \edit{its reciprocal conflict edges are removed} from $\mathcal G(t)$. If the agents remain within sensing range, they still contribute to each other's conflict evidence, but this contribution weakens smoothly as pressure falls into the release band (\edit{Lemma~\ref{prop:uniform_pressure_bands}}). \edit{Removing a conflict's reciprocal edges} from $\mathcal G(t)$ may split the group into smaller groups with no unresolved conflicts between groups. The agents' positions, speeds, decision states, and attention states remain continuous, and the traversal orders and temporal-clearance margins on all remaining conflicts are preserved (Theorem~\ref{thm:graph_reconfiguration_induction}). Without new conflicts, each clearance removes an ordering constraint, reducing the acyclic graph.
{Figure~\ref{fig:bifurcation_graph_reconfiguration} illustrates traversal-order formation and adaptation as the interaction network evolves in a three-agent engaged group. Initially, all three pairs are in conflict (panel~A), and the differentiated equilibrium $\boldsymbol z^{\dagger,-}$ gives the order $z_1>z_3>z_2$, with agent~1 as the leading mover. Clearance of the $1$--$3$ conflict \edit{removes its reciprocal conflict edges (panel~B)}. As this pair's residual influence decays, the equilibrium shifts smoothly to $\boldsymbol z^{\dagger,+}$. The decision amplitude decreases, and agent~3's decision changes from negative to positive. The orders on the remaining conflicts persist, allowing agents~1 and~3 to lead concurrently. Supplementary Sec.~\ref{sec:numerical_bifurcation_illustration} provides the calculations.}

\edit{A new physical conflict adds reciprocal edges to $\mathcal G(t)$, and the affected group's traversal order adapts as existing agents' states evolve continuously. The new conflict may connect agents within an existing group or merge otherwise separate groups. Conflicts involving an agent entering the population may join that agent to one group or merge several groups. An agent entering the population without a new physical conflict does not by itself require a change in an established traversal order. 
} With limited sensing ranges, groups may establish orders independently before detecting conflicts that merge them. Their continuously evolving \edit{decisions, attention states, and speeds} provide the initial conditions for coordination over the slowly evolving conflict graph. These inherited decision differences and projected entry-time offsets select an attracting differentiated branch for the updated group. Small entry-time offsets ($\eta_\tau \neq 0$) can further unfold the bifurcation and bias branch selection (Fig.~\ref{fig:supp_cubic_bifurcation_comparison}). \edit{If the updated physical-conflict set remains fixed for $T^{\mathrm{neg}}$, a strict acyclic traversal order is established within that interval}
(Theorem~\ref{thm:graph_reconfiguration_induction}). The lower bound on sensing radius ensures that each new conflict is necessarily detected early enough for the agents to respond safely throughout this adaptation (\edit{Lemma~\ref{prop:sensing_range_detection_margin} and Proposition~\ref{prop:fixed_component_edge_safety}}). \newline

\textit{Release to nominal motion:} \edit{After all of an agent's physical conflicts in the current episode clear, residual sensed pressure decays through the release band and its states relax toward nominal values. If no new conflict involving that agent activates, decision damping gives $z_i\to0$, attention returns to $u_0/\rho_u$, and speed tracking restores $v_i^{\mathrm{nom}}$ as residual evidence vanishes.} {Figure~\ref{fig:ndd_deconfliction}, C to E show staggered recovery and the eventual loss of decision differentiation: agents return to nominal motion as their own conflicts clear, without waiting for all others.}
\edit{With finitely many physical-conflict activations, safe adaptation to new conflicts and successive clearances resolve all conflicts in the episode in finite time, while distinct groups can progress concurrently (Theorem~\ref{thm:finite_time_safe_deconfliction}).}

{Thus, a locally engaged group undergoes a self-organized \edit{spiking--hold--recovery cycle} analogous to that of an excitable neural system~\cite{wongwang2006recurrent,nagumo1962active,izhikevich2006dynamical}.}
\newline

\textit{Scalability and plasticity:} {NDD achieves scalability through local design and computation. Under the prescribed uniform bounds on local congestion, geometry, speed-response heterogeneity, sensing and tracking errors, and actuation limits, one controller design accommodates any arbitrary finite population and group size (\edit{Lemma~\ref{lem:uniform_spectral_bands}} and Theorem~\ref{thm:finite_time_safe_deconfliction}). Each agent evaluates one projection per sensed interaction, keeping computation proportional to its local interaction degree. NDD achieves plasticity by reorganizing coordination as agents enter or leave and conflicts change. Engaged groups form, merge, split, and dissolve on the slower physical time scale, retaining traversal orders on remaining conflicts after clearance and smoothly adapting when new conflicts arise (Theorem~\ref{thm:graph_reconfiguration_induction}).} \newline

\textit{Sensing range and the safety--throughput trade-off:} A larger sensing range provides more time to establish the required temporal gap across a contested region, permitting smaller speed differences and motion closer to the nominal group speed. By contrast, a shorter sensing range requires stronger differentiation, more pronounced yielding, or, in the limiting case, certified stopping (\edit{Proposition~\ref{prop:fixed_component_edge_safety}}, particularly Equation~\eqref{eq:required_clearance_time}). Because safety depends on relative rather than absolute speeds, both agents may slow down or speed up while preserving the required clearance margin. Concurrency further enhances throughput: leading movers that share no physical conflict can clear independent conflicts simultaneously.

\subsection*{Simulation benchmarks}
We evaluated NDD against optimal reciprocal collision avoidance (ORCA)~\cite{berg2011reciprocal,alonsomora2013optimal}, LivePoint~\cite{chen2025livepoint}, and centralized and decentralized control barrier function quadratic programs (CBF-QP)~\cite{ames2016control,wang2017safety} using homogeneous Dubins-car agents. All methods used the same 50 randomized start--goal scenarios for each $N\in\{4,8,12,24,32\}$. Across all 250 NDD rollouts, every NDD agent reached its goal without a safety violation, pairwise separation stayed above the safety threshold, and agents retained their prescribed paths (Fig.~\ref{fig:simulation_benchmarks}, A to C). LivePoint achieved the highest baseline safe-agent success, although not even close to NDD's 100\% success and with performance that declined with group size. ORCA and both CBF-QP variants had lower success and mean minimum separation below the safety threshold. Two highly congested 15-agent scenarios---a near-antipodal exchange creating a central bottleneck and intersecting groups of nearly parallel paths---highlight this contrast: NDD maintained collision-free passage along the prescribed paths, whereas all three displayed baselines incurred collisions in the contested region (Fig.~\ref{fig:simulation_trajectories}, A to H).

\subsection*{Hardware Experiments}

We evaluated NDD using six TurtleBot3 differential-drive robots on prescribed straight routes in an indoor workspace (Fig.~\ref{fig:hardware_experiments}). A Vicon motion-capture system supplied positions and orientations through ROS; planar velocities were estimated from successive position measurements and low-pass filtered. Each robot's controller ran in a separate ROS node with asynchronous pose updates and no inter-robot communication, modulating forward speed around the nominal $0.2\,\mathrm{m/s}$. \edit{Two antipodal repetitions with starting positions differing by only $0.8$--$2.5$~cm per robot produced different traversal orders: robot~2 (orange) passed first in one trial and last in the other (Fig.~\ref{fig:hardware_experiments}, A to F). These contrasting outcomes illustrate that traversal order emerges from coupled motion without preassigned priorities. In NDD, small asymmetries in local motion cues bias decisions; the resulting speed changes alter neighbors' cues and amplify differentiation across the group. Orthogonal and mixed-lane encounters further show adaptive yielding: robots adjust their speeds throughout passage, and their decisions relax toward zero as they complete their crossings (Fig.~\ref{fig:hardware_experiments}, G to L).}

\section*{Discussion}

NDD establishes a constructive design principle for controlling desired emergent behavior in embodied multi-agent systems. \nl{No communication, centralized scheduling, or combinatorial traversal-order assignment is required.} Each agent extracts task-relevant relational information from local observations, constructs continuous-time interaction evidence, acts on that evidence, and uses the resulting motion cues to update its decision. Weak evidence keeps the group near collective indecision. For \edit{safe-and-live deconfliction}, safety-critical conflict drives the group through a bifurcation that amplifies even infinitesimally small asymmetries in motion into rapid differentiation in decisions. The tightly coupled decision--action--perception dynamics then create the temporal separation required for collision avoidance and the strict traversal order required for progress, allowing safety and liveness to emerge from a single closed loop. Agents recover nominal motion as conflicts clear. We derive explicit conditions on sensing, actuation, and feedback design parameters that set the commitment threshold and guarantee finite-time, robust, \nl{safe, deadlock-free, livelock-free} deconfliction  under heterogeneous actuation, sensing noise, and time-varying interaction topologies. \nl{The Supplementary Text provides complete certification details. 
Applying the NDD design with a certified guarantee requires only one choice of controller design parameters that satisfy the five compatibility conditions listed in Supplementary Section~\ref{sec:supp_joint_design}. 
This choice is independent of the total population and realized sizes of engaged groups.}

\subsection*{Related work}

\edit{Safe-and-live deconfliction} methods broadly comprise reactive avoidance and safety filters, explicit liveness mechanisms, and game-, optimization-, or learning-based coordination. Separating safety from liveness risks rejecting progress-enabling motion or disrupting the cues that establish traversal order. 
Methods that address safety and liveness within a single formulation account for coupled interactions, but shift the burden to representing those interactions in the decision problem. They must model or learn how each action propagates across connected conflicts before selecting a response. As the conflict group grows or changes, real-time computation can become infeasible or intractable. 
The unresolved challenge is therefore to allow a mutually consistent traversal order to emerge and adapt as interactions unfold. 

\paragraph{Reactive Collision Avoidance and Deadlock Recovery}
Reactive collision-avoidance methods compute motion commands from geometry or relative velocity using velocity obstacles~\cite{fiorini1998motion}, reciprocal variants~\cite{berg2008reciprocal,berg2011reciprocal,alonsomora2013optimal}, and potential fields or navigation functions~\cite{khatib1986real,rimon1992exact,loizou2008navigation}. Safety filters provide stronger guarantees by modifying nominal inputs that would violate safety constraints. Control barrier functions establish forward invariance of a safe set~\cite{ames2016control,borrmann2015control,wang2017safety}, while data-driven variants learn agent dynamics, barrier certificates, or distributed safety controllers from data~\cite{taylor2020learning,lindemann2021learning,wabersich2023data,zhang2025gcbfplus}. These mechanisms do not generally determine which agent should pass first through shared space: agents may yield simultaneously or repeatedly revise their motion. Geometric symmetry can force reactive controllers into deadlock even with heterogeneous control gains~\cite{grover2020symmetry}.

Deadlock-avoidance methods add recovery or preventive mechanisms to preserve liveness when reactive avoidance is insufficient. These include formation-level escape motions certified for two- and three-robot formations~\cite{grover2020deadlock,grover2023before}, trajectory exchange and replanning~\cite{jager2001decentralized}, leader-guided recovery~\cite{garg2024foundation}, deadlock-free liveness sets~\cite{chandra2025deadlock}, point-cloud barrier functions and speed regulation~\cite{chen2025livepoint}, or model predictive control (MPC)-based recovery and preventive risk conditions~\cite{chen2024deadlock,zhang2026adaptive}. These methods enforce liveness through explicit recovery maneuvers, planning rules, or constraints designed to exclude deadlock, rather than through the continuous formation of a traversal order during local interactions. \edit{Recovery can require communication~\cite{jager2001decentralized,garg2024foundation} or departures from nominal paths~\cite{grover2023before}, limiting applicability when communication is unavailable or paths must remain fixed.}

\paragraph{Prediction- and Planning-Based Coordination}
Planning methods encode passage alternatives through mixed-integer or rule-based trajectory branches~\cite{schouwenaars2001mixed,richards2002aircraft}, topological predictions~\cite{roh2021multimodal}, learned winding strategies~\cite{nakao2025winding}, or opinion-guided planning and game reduction~\cite{alghamdi2025opinion}. Multirobot planners establish precedence through optimized planning sequences~\cite{bennewitz2001optimizing}, priority graphs or search~\cite{cap2015prioritized,gregoire2013priority,ma2018searching}. When multiple conflicts become coupled in real time, local priorities selected independently from partial observations may not define a consistent group-wide traversal order.

Nonlinear opinion dynamics enable rapid commitment from social cues~\cite{cathcart2023proactive} or dynamic-game values~\cite{hu2023emergent}; spiking extensions permit fast switching as environmental information changes~\cite{cathcart2024spiking}. Social-cue approaches increase attention as a human approaches without requiring a predictive model of human behavior, while game-derived opinions guide trajectory optimization. \edit{Hu et al.~\cite{hu2023emergent} establish stability conditions for the game-induced opinion equilibria of two agents with two options each.} These methods resolve two-option intent ambiguity, but do not establish group-wide safe traversal orders for arbitrary time-varying conflict networks. Interaction-aware prediction and games model reciprocal responses through cooperative crowd forecasts~\cite{trautman2010unfreezing,trautman2015robot}, influence-aware planning~\cite{sadigh2016planning}, coupled strategies~\cite{fridovich2019efficient,le2022algames}, social preferences~\cite{schwarting2019social}, bounded rationality~\cite{mehr2023maxent}, or mixed strategies~\cite{muchen2024mixed}. They require models or data for agents' dynamics, objectives, and responses, and their coupled decision problem grows with the interacting group.

Trajectory optimization and MPC incorporate dynamics, collision constraints, and progress objectives~\cite{augugliaro2012generation,schulman2013finding,zhang2017optimization,ferranti2023distributed,zhu2020trajectory}, including centralized nonlinear MPC with collision- and deadlock-free motion guarantees for nonholonomic robots~\cite{lafmejani2021nonlinear}. Growing conflict coupling challenges tractability and feasibility. Pairwise collision constraints characterize unsafe states, whereas deadlock and livelock depend on sequences of closed-loop responses and require an explicit liveness construction.

Learning-based methods encode interaction reasoning in contextual crowd models~\cite{liu2025emergent} or policies mapping observations to actions~\cite{chen2017decentralized,everett2018motion,chen2019crowd,fan2020fully}. They capture patterns difficult to specify analytically, but depend on training coverage. Changing group composition and conflict geometry complicate safety and liveness certification.

\paragraph{Urban Transportation and Connected Vehicle Systems}
Urban transportation research coordinates connected and automated vehicles at intersections and merges by solving for access order and timing. Approaches include space--time reservations~\cite{dresner2004multiagent,dresner2008multiagent}, auctions~\cite{carlino2013auction}, conflict-point optimization~\cite{levin2017conflict}, and polling~\cite{miculescu2020polling}. Networked-control approaches compute vehicle trajectories after the crossing order or safety constraints are established. Decentralized optimal-control and distributed schemes coordinate crossing times and vehicle controls while enforcing rear-end and lateral safety constraints, including in experimental implementations~\cite{zhang2019decentralized,hult2019optimal,divaio2019distributed}. Recent mixed-traffic formulations jointly optimize traffic signals and connected automated vehicle trajectories using distributed optimization~\cite{le2024distributed}, or learn human-driver behavior online to generate robust, time-optimal merging trajectories under prediction uncertainty~\cite{le2025stochastic}. 

These methods exploit structured lanes and shared conflict zones. Reservations and auctions rely on intersection managers and vehicle-to-infrastructure communication; distributed controllers still exchange states, schedules, or crossing orders~\cite{divaio2019distributed}. Dense, unstructured robot navigation may offer only local observations, without communication or a fixed coordinator. Traversal order must therefore form and adapt through local interaction.

\subsection*{Importance and future directions}
Our multitime-scale autonomy architecture embeds NDD's fast, continuous-time decision dynamics as a lightweight interaction layer within a hierarchical autonomy stack (Fig.~\ref{fig:overview}). 
NDD resolves crossings and merges through local speed modulation while preserving nominal steering along the planned path, including turns. The planning module replans paths when speed modulation alone is insufficient, as with static obstacles or direct head-on encounters.
Future work could extend NDD to modulate nominal steering and handle static obstacles within the local interaction layer, potentially improving efficiency.
Alternative interaction-evidence and decision-to-control maps could extend this architecture to other tasks.
For example the decision state could regulate aerial-vehicle altitude or manipulator speed and task timing, or incorporate traffic conventions. 
An inference layer could add intent, semantic context, communication, and cooperation estimates. Although fixed local NDD parameters remain valid as groups form, split, or merge, adaptation to congestion could improve performance. Future work should establish safety and liveness conditions for adapting controller gains, attention parameters, and evidence weights through local observations or online learning. NDD provides a general, resource-aware foundation for shaping and certifying collective behavior through local physical interactions.

\subsection*{Limitations}
\edit{Our guarantees assume that all agents cooperate and use NDD. Extending them to mixed crowds requires analysis of other decision policies; potentially noncooperative agents also require real-time cooperation inference. Nonlinear opinion dynamics~\cite{bizyaeva2023nonlinear} can support this inference and inform NDD responses. Developing the NOD--NDD coupling and establishing mixed-population guarantees remain future work.}

\section*{Materials and Methods}

\subsection*{Formulation of \edit{safe-and-live deconfliction}}
\label{sec:mathematical_formulation}

Consider a finite population of agents moving along prescribed paths, with position $\boldsymbol  q_i$ and realized traversal speed $v_i$ for each agent $i$. At time $t$, let $\mathcal P(t)$ contain the unordered pairs $\{i,j\}$ in physical conflict. Each pair generates the directed interactions $e=(i,j)$ and $\bar e=(j,i)$. Define $ \mathcal E^{\mathrm{conf}}(t) := \bigcup_{\{i,j\}\in\mathcal P(t)}\{(i,j),(j,i)\}. $ With directional weights $\omega_i^e(t):=p^e(t)w_i^e(t)$, define the weighted directed conflict graph $ \mathcal G(t):=(\mathcal V(t),\mathcal E^{\mathrm{conf}}(t),\boldsymbol \omega(t)). $

Let $\mathcal C(t)$ index the strongly connected components of $\mathcal G(t)$. For $c\in\mathcal C(t)$, let $\mathcal V^c(t)$ be its agent set, $ \mathcal E^{\mathrm{conf},c}(t) := \{e\in\mathcal E^{\mathrm{conf}}(t):e\text{ has both endpoints in }\mathcal V^c(t)\}, $ and let $ \mathcal P^c(t):=\{\{i,j\}\in\mathcal P(t):i,j\in\mathcal V^c(t)\} $ be its physical-conflict pairs. We write $ \mathcal G^c(t):=(\mathcal V^c(t),\mathcal E^{\mathrm{conf},c}(t),\boldsymbol \omega^c(t)) $ for the corresponding strongly connected component. \edit{Each component is a locally engaged group. 
}

For $\{i,j\}\in\mathcal P^c(t)$ and $e=(i,j)$, define the temporal-clearance gaps $ \pi_{i\prec j}^e(\boldsymbol  q,\boldsymbol  v):=t^e_{j^-}-t^e_{i^+}, $ $ \pi_{j\prec i}^e(\boldsymbol  q,\boldsymbol  v):=t^e_{i^-}-t^e_{j^+}. $ The relation $i\prec_e j$ means that agent $i$ clears its contested segment before agent $j$ \edit{enters its contested segment}; $j\prec_e i$ denotes the reverse order. 
\nl{\textit{An engaged group $c$ is temporally safe} when every physical conflict pair $\{i,j\}\in\mathcal P^c(t)$ admits either $i\prec_e j$ or $j\prec_e i$ \edit{with positive temporal separation}}. Define
\begin{align}
\mathcal V_{\mathrm{safe}}^c(\boldsymbol  q)
&=
\left\{
\boldsymbol  v:
\max\left\{
\pi_{i\prec j}^e(\boldsymbol  q,\boldsymbol  v),
\pi_{j\prec i}^e(\boldsymbol  q,\boldsymbol  v)
\right\}
\edit{>0}
\quad
\forall\{i,j\}\in\mathcal P^c(t)
\right\},
\label{eq:safe_speed_set}
\\
\mathcal Z_{\mathrm{safe}}^c(\boldsymbol  q)
&=
\left\{
\boldsymbol  z:
\boldsymbol  v^{\mathrm{nom}}+\boldsymbol  V(\boldsymbol  z)
\in
\mathcal V_{\mathrm{safe}}^c(\boldsymbol  q)
\right\}.
\label{eq:safe_decision_set}
\end{align}

\edit{Collision avoidance alone can mean agents remaining stopped or repeatedly adjusting their motion without clearing a conflict. Thus, finite-time progress must also be ensured, and this relies on the emergence of leading movers: agents that advance while leading all their unresolved conflicts.} For agent $i\in\mathcal V^c(t)$, define its set of conflicting neighbors as $\mathcal N_i^c(t):=\{j\in\mathcal V^c(t):\{i,j\}\in\mathcal P^c(t)\}$. Component $c$'s leading-mover set is then
\begingroup
\begin{equation*}
\mathcal M^c(t):=
\left\{
i\in\mathcal V^c(t):
\mathcal N_i^c(t)\neq\varnothing,
\ v_i(t)\geq\underline v,
\ \pi_{i\prec j}^e(\boldsymbol  q(t),\boldsymbol  v(t))\edit{>0}
\ \forall j\in\mathcal N_i^c(t)
\right\},
\qquad
\underline v>0.
\end{equation*}
\endgroup
Every agent in $\mathcal M^c(t)$ moves with positive speed and leads each of its conflicts with \edit{positive temporal clearance}. 
\edit{A group is \emph{unresolved} when $\mathcal P^c(t)\neq\varnothing$. \emph{Negotiation} establishes a safe traversal order. During \emph{persistence}, each selected pairwise order remains until physical clearance, provided no new physical conflict activates. Each interaction begins release at clearance.
Under the conditions of Theorem~\ref{thm:fixed_component_safe_order}, a safe acyclic traversal order provides a leading mover that clears a conflict in finite time.} 
\edit{A \emph{conflict activation} adds a previously absent physical-conflict pair to $\mathcal P(t)$ and its reciprocal edges to $\mathcal G(t)$; physical clearance removes them. These \emph{graph events} are treated in Theorem~\ref{thm:graph_reconfiguration_induction} which extends safe-and-live deconfliction under time-varying graph reconfiguration and ensures finite-time clearance of every conflict in an admissible episode.}

\subsection*{Smooth speed-control and conflict-evidence design}
\label{sec:admissible_maps}

For the speed command in~Equation~\eqref{eq:cmd_speed}, we choose the saturating law
\begin{equation}
v_i^{\mathrm{cmd}}
=
v_i^{\mathrm{nom}}
\left(1+\kappa_V\tanh(\kappa_z z_i)\right),
\qquad
\kappa_V>0,
\quad
\kappa_z>0.
\label{eq:tanh_speed_map}
\end{equation}
\edit{Thus $V_i(z_i)=v_i^{\mathrm{nom}}\kappa_V\tanh(\kappa_z z_i)$ in Equation~\eqref{eq:cmd_speed}. The command equals nominal speed at $z_i=0$ and increases with $z_i$. The gain $\kappa_V$ sets the maximum fractional-speed departure; $\kappa_z$ sets the sensitivity to the decision state. Lemma~\ref{lem:decision_range_and_speed_extrema} gives gains spanning each agent's admissible speed range within any prescribed local decision neighborhood.}

\edit{For each directed interaction $e=(i,j)$, the pairwise observation $\mathcal X_i^e$ gives the closest-approach distance $d_{\mathrm{ca}}^{e}$ and time $t_{\mathrm{ca}}^{e}$, and each agent's contested-segment entry and exit times $t^e_{k^-},t^e_{k^+}$, $k\in\{i,j\}$ (Fig.~\ref{fig:contested_regions}, A). Agent $i$ evaluates its sensed interaction set $\mathcal E_i(t)$. The pressure factors below use the logistic sigmoid $\sigma(x;k):=(1+\exp(-kx))^{-1}$, $k>0$.}

At each time, \nl{agent $i$ computes} conflict pressure \nl{for interaction $e=(i,j)$} 
from current relative-position and relative-velocity measurements and combines spatial and temporal relevance. The spatial factor depends on the predicted distance at closest approach, rather than the agents' current separation, while the temporal factor measures whether closest approach occurs before the exit times of each agent. Their product \nl{of spatial and temporal factors} provides a graded measure of conflict severity, particularly at shallow crossing angles, where long contested segments can produce substantial occupancy overlap even when the predicted closest-approach distance remains above $d^{\mathrm{safe}}$. Because pressure depends on the projected relative geometry, it may remain constant as the agents move if that geometry does not change.

Set $d_{\mathrm{prox}}=k_{\mathrm{prox}}d^{\mathrm{safe}}$, where $k_{\mathrm{prox}}>1$, so that pressure begins to rise outside the required clearance boundary. \edit{The spatial-proximity factor is $P_{\mathrm d}(\mathcal X_i^e)=\sigma(d_{\mathrm{prox}}-d_{\mathrm{ca}}^{e};k_{\mathrm d})$, with $k_{\mathrm d}>0$.}
The closest approach is temporally relevant only if it occurs before both projected exit times, that is, $t_{\mathrm{ca}}^{e}<\min\{t^e_{i^+},t^e_{j^+}\}$. To retain the interaction near either exit boundary, we introduce a temporal margin $t_{\mathrm{prox}}>0$. The resulting temporal-proximity factor is the smooth conjunction
\begin{equation*}
P_{\mathrm t}(\mathcal X_i^e)
=
\sigma\left(
t_{\mathrm{prox}}
+
t^e_{i^+}
-
t_{\mathrm{ca}}^{e};
k_{\mathrm t}
\right)
\sigma\left(
t_{\mathrm{prox}}
+
t^e_{j^+}
-
t_{\mathrm{ca}}^{e};
k_{\mathrm t}
\right),
\qquad
k_{\mathrm t}>0.
\end{equation*}
Each factor decays as the projected closest-approach time moves beyond the corresponding buffered exit time. The total pressure is
\begin{equation}
P(\mathcal X_i^e)
=
P_{\mathrm d}(\mathcal X_i^e)
P_{\mathrm t}(\mathcal X_i^e).
\label{eq:conflict_pressure}
\end{equation}
\edit{The pressure lies in $(0,1)$ and decreases when either spatial or temporal relevance falls.}

\edit{The arrival-order function maps projected entry-time differences to signed evidence. An admissible map $\gamma:\mathbb R\rightarrow[-1,1]$ is smooth, bounded, and odd, with $\gamma(0)=0$ and $s\gamma(s)>0$ for $s\neq0$. We use}
\begin{equation}
G(\mathcal X_i^e)
=
\tanh\left(
k_{\mathrm g}
\left(
t^e_{j^-}
-
t^e_{i^-}
\right)
\right),
\qquad
k_{\mathrm g}>0.
\label{eq:arrival_order_function}
\end{equation}
\edit{The value is positive when agent $i$ is projected to enter its contested segment before agent $j$ enters its own, negative for the reverse order, and zero at a pairwise tie.} Write $p_i^e:=P(\mathcal X_i^e)$ and $g_i^e:=G(\mathcal X_i^e)$. \edit{Reciprocal interactions share the closest-approach geometry and reverse the entry-time difference, so $p_i^e=p_j^{\bar e}=:p^e$ and $g_j^{\bar e}=-g_i^e$.}

\edit{To weight simultaneous conflicts, agent $i$ tilts the pressure of each directed interaction $e=(i,j)$ according to its signed arrival-order evidence:}
\begin{equation}
\edit{\hat{p}^e}
=
p_i^e
\left(
1-
\frac{k_{\mathrm{tilt}}g_i^e}{1-p_i^e}
\right),
\qquad
k_{\mathrm{tilt}}>0.
\label{eq:tilted_pressure}
\end{equation}
Here $0<p_i^e\leq\bar p<1$, so the denominator remains bounded away from zero. \edit{Lemma~\ref{prop:uniform_pressure_bands}} establishes this bound and gives conditions for tuning the pressure gains, \nl{$k_t$, $k_{\mathrm{tilt}}$, $k_g$,} to separate the interaction regimes. The normalized weight in~Equation~\eqref{eq:ndd:evidence} is
\begin{equation}
W(\mathcal X_i^e;\mathcal X_i)
=
\frac{
\exp\left(\edit{\hat{p}^e}/T_{\mathrm{sm}}\right)
}{
\displaystyle
\sum_{e'\in\mathcal E_i(t)}
\exp\left(\edit{\hat{p}^{e'}}/T_{\mathrm{sm}}\right)
},
\qquad
0<T_{\mathrm{sm}}<1.
\label{eq:pressure_softmax}
\end{equation}
\edit{The tilt lowers the score for $G>0$ and raises it for $G<0$, assigning greater responsibility to the agent projected to enter its contested segment later. The weights satisfy $\sum_{e\in\mathcal E_i(t)}W(\mathcal X_i^e;\mathcal X_i)=1$ for nonempty $\mathcal E_i(t)$; otherwise, $h_i=0$.}

\edit{Define $w_i^e:=W(\mathcal X_i^e;\mathcal X_i)$ and $\omega_i^e:=p^e w_i^e$, giving $h_i^e=\omega_i^e g_i^e$. Each agent normalizes over its own sensed set, so generally $w_i^e\neq w_j^{\bar e}$ and $\omega_i^e\neq\omega_j^{\bar e}$ despite the shared pressure. With fixed sensed sets, positive traversal speeds, and crossing angles bounded away from zero, the kinematic quantities and functions in Equations~\eqref{eq:conflict_pressure}--\eqref{eq:pressure_softmax} are smooth in the observations. The evidence is therefore piecewise smooth throughout an episode.}

\edit{Supplementary Text, Section~\ref{sec:supp_pressure_regimes}, gives conditions for alternative pressure functions.}

\clearpage
\bibliographystyle{sciencemag}
\bibliography{references}

\subsection*{Supplementary materials}
Figure~\ref{fig:supp_cubic_bifurcation_comparison}: Bifurcation diagrams and unfolding by timing bias.\\
{Supplementary Text (Sections~\ref{sec:supp_pressure_regimes}--\ref{sec:numerical_bifurcation_illustration})}\\
Table~\ref{tab:principal_notation}: Notation: agent states and controller parameters\\
Table~\ref{tab:notation_interactions}: Notation: interaction structure and conflict evidence\\
Table~\ref{tab:notation_geometry}: Notation: conflict geometry, timing, and safety\\
Table~\ref{tab:notation_spectral}: Notation: spectral and bifurcation quantities\\
{Captions for Movies~S1--S9}\\
References \cite{kato1966perturbation,khalil2008nonlinear,golubitsky1985singularities}

\section*{Acknowledgments}
\edit{We acknowledge the Leonard Lab at Princeton University for providing the facilities used for the hardware experiments.}
\paragraph*{Funding:} \edit{This work was supported by the U.S. Army Research Office through a grant to N.E.L. under Grant Number W911NF2410126, and the Office of Naval Research through a grant to V.S. under Grant Number N00014-22-1-2813. The views and conclusions contained in this document are those of the authors and should not be interpreted as representing the official policies, either expressed or implied, of the Army Research Office, the Office of Naval Research, or the U.S. Government. The U.S. Government is authorized to reproduce and distribute reprints for Government purposes notwithstanding any copyright notation herein.} 
\paragraph*{Author contributions:} \edit{H.S., V.S., and N.E.L. contributed to developing the ideas presented in this work and prepared the theoretical proofs. H.S. wrote the manuscript, and all authors reviewed and edited it. H.S. performed the hardware experiments and simulations. N.E.L. and V.S. secured funding.}
\paragraph*{Competing interests:} \edit{H.S., V.S., and N.E.L. are inventors on a patent application being prepared by Princeton University covering the neuromorphic decision dynamics described in this work.}
\paragraph*{Data, code and materials availability:} 
Videos accompanying this work are available at {\urlstyle{same}\url{https://drive.google.com/drive/folders/1OCUdXAbFT4UJu-GTchzBHUohFRTxB_vd?usp=drive_link}}.

\clearpage
\begin{figure}[p]
    \centering    \includegraphics[width=\textwidth,height=0.42\textheight,keepaspectratio]{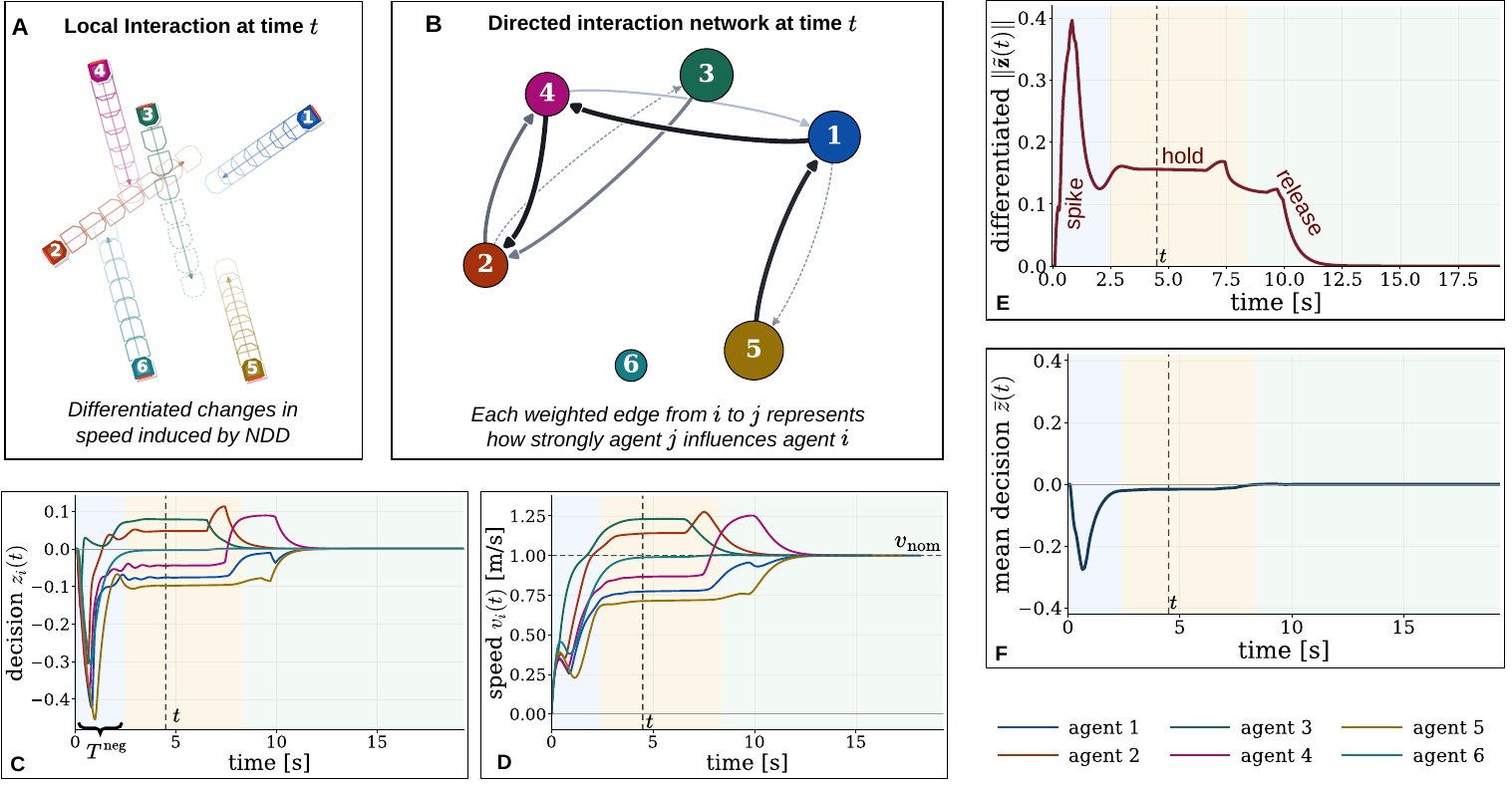}
\caption{\textbf{Emergence of \edit{safe-and-live deconfliction} under an NDD-driven controller.}
\edit{\textbf{(A)} Colored bodies show six agents at time $t$; faded poses project motion along prescribed paths at equally spaced future times, assuming current speeds persist. Closer spacing indicates slower motion. Dotted portions show motion after projected conflict clearance. \textbf{(B)} An arrow $i\to j$ means agent~$j$'s motion influences agent~$i$'s conflict evidence. Darker edges indicate stronger influence; node size represents aggregate conflict-evidence magnitude $|h_i(t)|$. \textbf{(C)} Decisions $z_i(t)$ initially indicate collective yielding, then differentiate. The bracket marks negotiation time $T^{\mathrm{neg}}$; shading identifies negotiation, hold, and release. \textbf{(D)} Realized speeds $v_i(t)$ track decision-dependent commands. Positive and negative decisions raise and lower commanded speeds, respectively, relative to nominal $v^{\mathrm{nom}}=1\,\mathrm{m/s}$ (horizontal dashed line). \textbf{(E)} The differentiated decision magnitude $\|\widetilde{\boldsymbol z}(t)\|=\sqrt{\sum_{i=1}^{6}(z_i(t)-\bar z(t))^2}$, with mean $\bar z(t)=\frac{1}{6}\sum_{i=1}^{6}z_i(t)$, spikes during negotiation, persists during hold, and vanishes during release. \textbf{(F)} The mean decision becomes negative during collective yielding, then returns to near zero during hold while decision differences persist. After release, decisions vanish and speeds return to nominal. Vertical dashed lines in C--F mark snapshot time $t$.}}
\label{fig:ndd_deconfliction}
\end{figure}

\clearpage
\begin{figure}[p]
    \centering
\includegraphics[width=\textwidth,height=0.42\textheight,keepaspectratio]{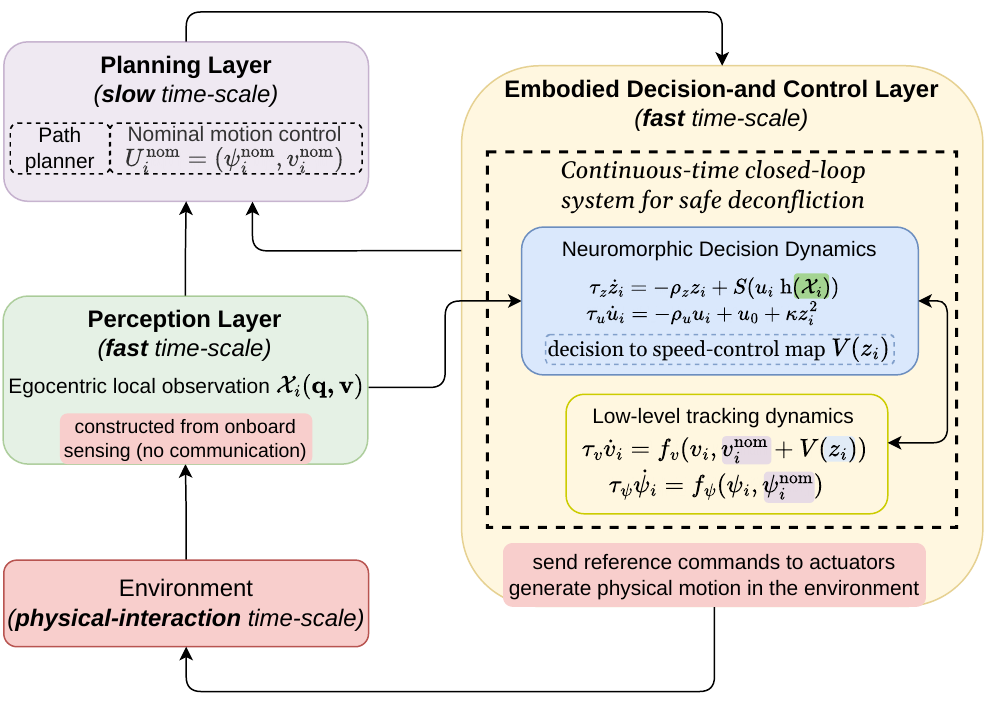}
    \caption{\textbf{Multitime scale embodied control architecture for \edit{safe-and-live deconfliction}.}
\edit{All layers operate concurrently at distinct characteristic time scales. The slow planning layer supplies a prescribed path and nominal motion command $U_i^{\mathrm{nom}}=(\psi_i^{\mathrm{nom}},v_i^{\mathrm{nom}})$. At a fast sensing and processing rate, the perception layer constructs the egocentric observation $\mathcal X_i(\boldsymbol q,\boldsymbol v)$ from onboard measurements of the agent's own motion and its neighbors' relative motion, without communication. These observations drive the fast decision and control layer: conflict evidence $h(\mathcal X_i)$ drives decision state $z_i$, while attention state $u_i$ regulates the response. The map $V(z_i)$ continuously adjusts commanded speed above or below nominal while retaining the nominal angular-rate command. Low-level tracking dynamics execute these evolving references through the actuators, so decision formation, attention adaptation, and physical action co-evolve. The resulting motion changes inter-agent geometry on the slower physical-interaction time scale. These changes generate new measurements that feed back into the ongoing decision dynamics, closing the continuous decision-action-perception loop.}}
\label{fig:overview}
\end{figure}

\clearpage
\begin{figure}[p]
\centering
\includegraphics[width=\textwidth,height=0.42\textheight,keepaspectratio]{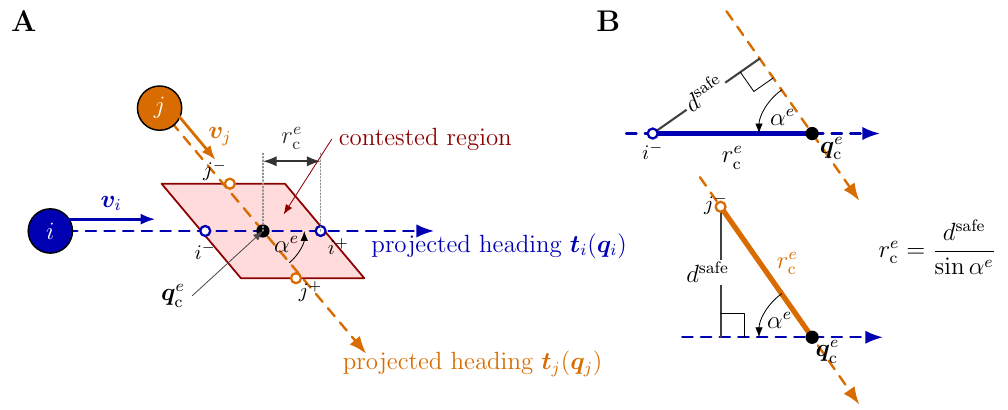}
\caption{ \textbf{Contested-region construction from instantaneous local kinematics.} \textbf{(A)} The dashed rays extend the currently measured velocity directions and are instantaneous heading projections; the physical paths may subsequently smoothly curve. The  intersection of the rays defines the geometric crossing point $\boldsymbol  q_{\mathrm c}^{e}$ and crossing angle $\alpha^e$. On each projected ray, the markers $i^{-},i^{+}$ and $j^{-},j^{+}$ delimit the segment over which simultaneous occupancy can reduce inter-agent separation below $d^{\mathrm{safe}}$. Together, these segments define the contested region. The current along-ray distances and speeds determine the projected entry and exit times and hence the occupancy intervals of the two agents. Conflict evidence depends on possible temporal overlap of these intervals and does not require equal arrival at $\boldsymbol  q_{\mathrm c}^{e}$. \textbf{(B)} The right-triangle constructions locate the upstream segment boundaries. The perpendicular clearance satisfies $d^{\mathrm{safe}}=r_{\mathrm c}^{e}\sin\alpha^e$, giving $r_{\mathrm c}^{e}=d^{\mathrm{safe}}/\sin\alpha^e$. Shallower crossing angles therefore produce longer contested segments and, for a fixed speed, longer projected occupancy intervals. }
\label{fig:contested_regions}
\end{figure}

\clearpage
\begin{figure}[p]
\centering
\includegraphics[width=\textwidth,height=0.42\textheight,keepaspectratio]{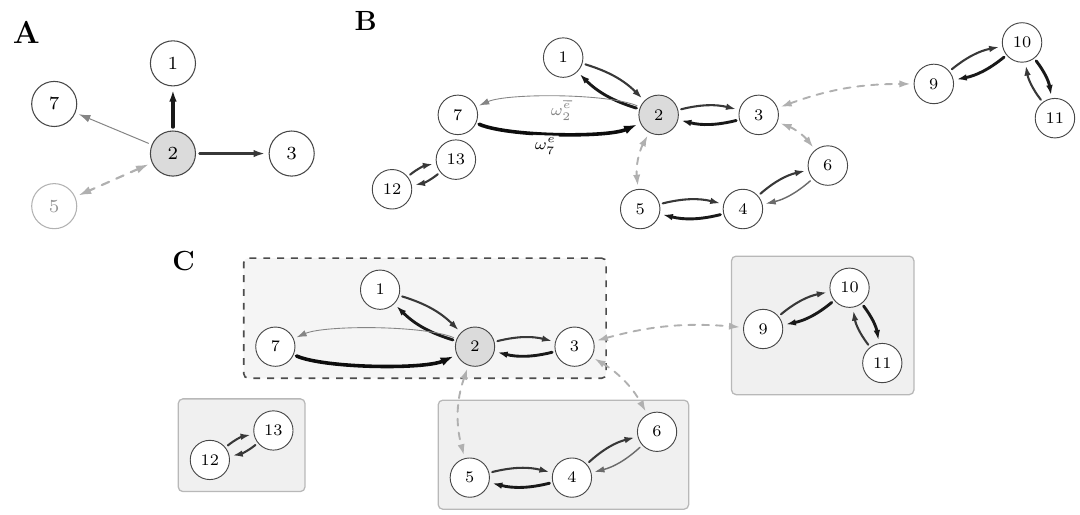}
\caption{
\textbf{Graph hierarchy from local sensing to locally engaged groups.} 
\textbf{(A)} Agent~2 evaluates the egocentric sensed-interaction set $\mathcal E_2(t)$. Under the convention $e=(i,j)$, an arrow $i\to j$ indicates that agent~$j$'s motion enters agent~$i$'s evidence. Solid arrows are the current conflict interactions $\mathcal E_2(t)\cap\mathcal E^{\mathrm{conf}}(t)$. Dashed gray double-headed links denote mutually sensed but nonconflicting pairs--the dashed gray edge from agent~2 belongs to $\mathcal E_2(t)\setminus\mathcal E^{\mathrm{conf}}(t)$. \textbf{(B)} The weighted directed conflict graph $\mathcal G(t)=(\mathcal V(t),\mathcal E^{\mathrm{conf}}(t),\boldsymbol \omega(t))$ contains the reciprocal pair $(i,j)$ and $(j,i)$ generated by every physical-conflict pair $\{i,j\}$. For $e=(i,j)$, arrow thickness represents the directional weight $\omega_i^e(t)$, so reciprocal directions may have unequal weights. The $2$--$7$ conflict illustrates pronounced asymmetry: taking $e=(7,2)$ and $\bar e=(2,7)$, the graph depicts $\omega_7^e(t)\gg\omega_2^{\bar e}(t)>0$. Dashed gray double-headed links are mutually sensed nonconflicting pairs shown for context and do not belong to $\mathcal G(t)$. \textbf{(C)} The strongly connected components $\{\mathcal G^c(t)\}_{c\in\mathcal C(t)}$ of $\mathcal G(t)$ are the locally engaged groups. The dashed box highlights the component containing agent~2. Sensed nonconflicting interactions do not merge these components. }
\label{fig:graph_hierarchy}
\end{figure}

\clearpage
\begin{figure}[p]
    \centering    \includegraphics[width=\textwidth,height=0.42\textheight,keepaspectratio]{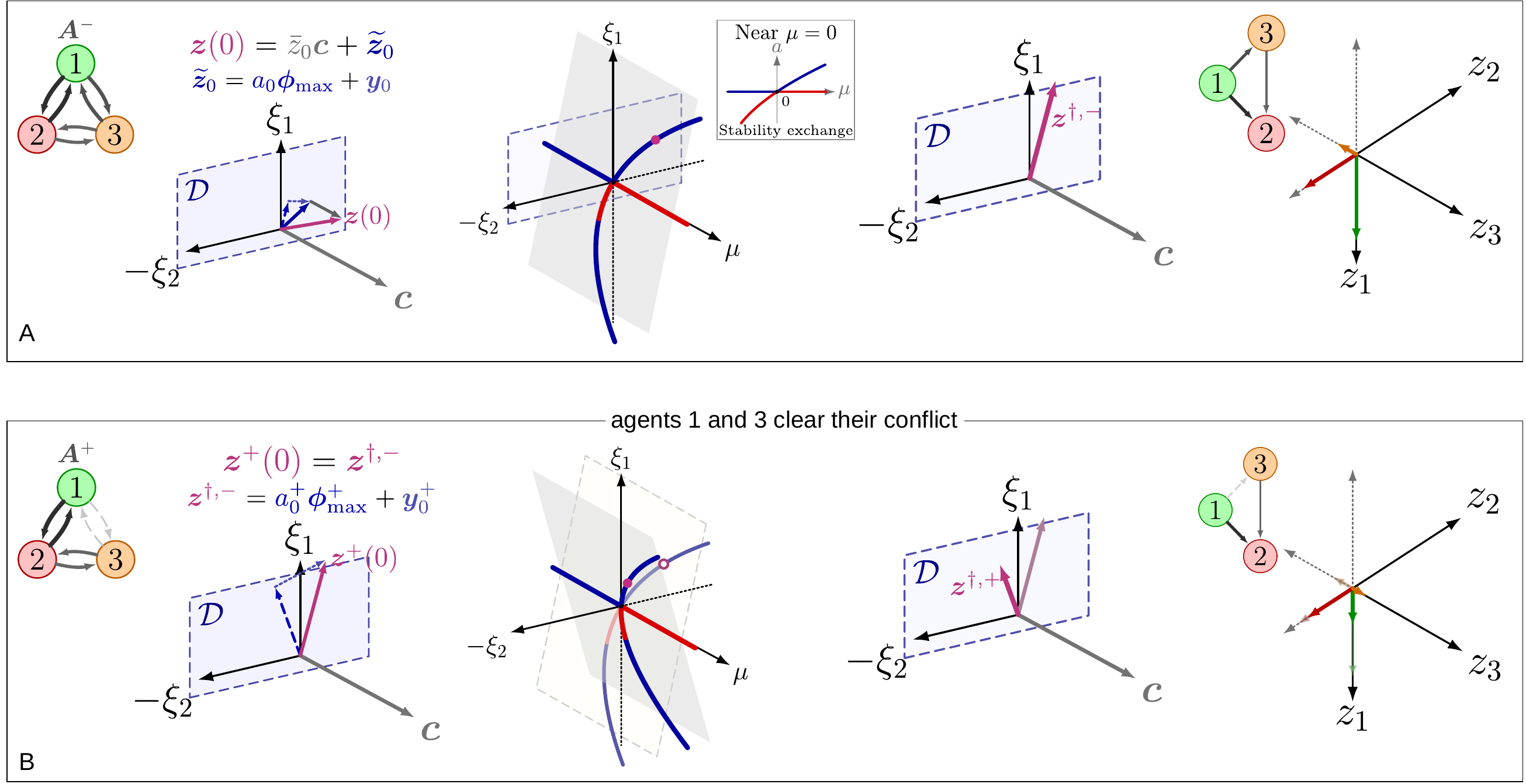}
    \caption{\textbf{Conflict release smoothly reconfigures the selected differentiated response.}
    \edit{Left graphs show reciprocal interaction sensitivities; right graphs show traversal orders, with arrows from the agent selected to pass first to its follower. Coordinates $\xi_1=(z_1-z_2)/\sqrt{2}$ and $\xi_2=(z_1+z_2-2z_3)/\sqrt{6}$ describe decision differences in $\mathcal D=\boldsymbol c^\perp$; $\boldsymbol c=(1,1,1)^{\mathsf T}/\sqrt{3}$ represents equal decision changes. The unit dominant eigenvector $\boldsymbol\phi_{\max}$ of the interaction-sensitivity matrix $\boldsymbol A$ determines the decision pattern; $a$ is its signed amplitude and $\mu=\lambda_{\max}-\lambda^\ast$ measures offset from the instability threshold.
    \textbf{(A)} Before clearance ($-$), all pairs are in conflict. Initial decisions separate into common and differentiated components, $\bar z_0\boldsymbol c$ and $\widetilde{\boldsymbol z}_0$. Positive initial amplitude $a_0$ selects the upper attracting branch. Blue and red curves denote attracting and unstable equilibria, respectively; the inset shows their stability exchange at $\mu=0$.
    \textbf{(B)} After conflict $1$--$3$ clears ($+$), the inherited state $\boldsymbol z^+(0)=\boldsymbol z^{\dagger,-}$ decomposes into $a_0^+\boldsymbol\phi_{\max}^+$ along the updated dominant eigenvector and the perpendicular remainder $\boldsymbol y_0^+$. States evolve continuously as the dominant eigenvector rotates and the equilibrium amplitude decreases. Agent~3's decision becomes positive; orders $1\prec2$ and $3\prec2$ permit agents~1 and~3 to progress concurrently. Pale elements show the previous configuration; the dashed $1$--$3$ interaction indicates residual sensitivity. Supplementary Sec.~\ref{sec:numerical_bifurcation_illustration} provides the matrices, amplitude model, and calculations.}}
    \label{fig:bifurcation_graph_reconfiguration}
\end{figure}

\clearpage
\begin{figure}[p]
    \centering
    \includegraphics[width=\textwidth,height=0.42\textheight,keepaspectratio]{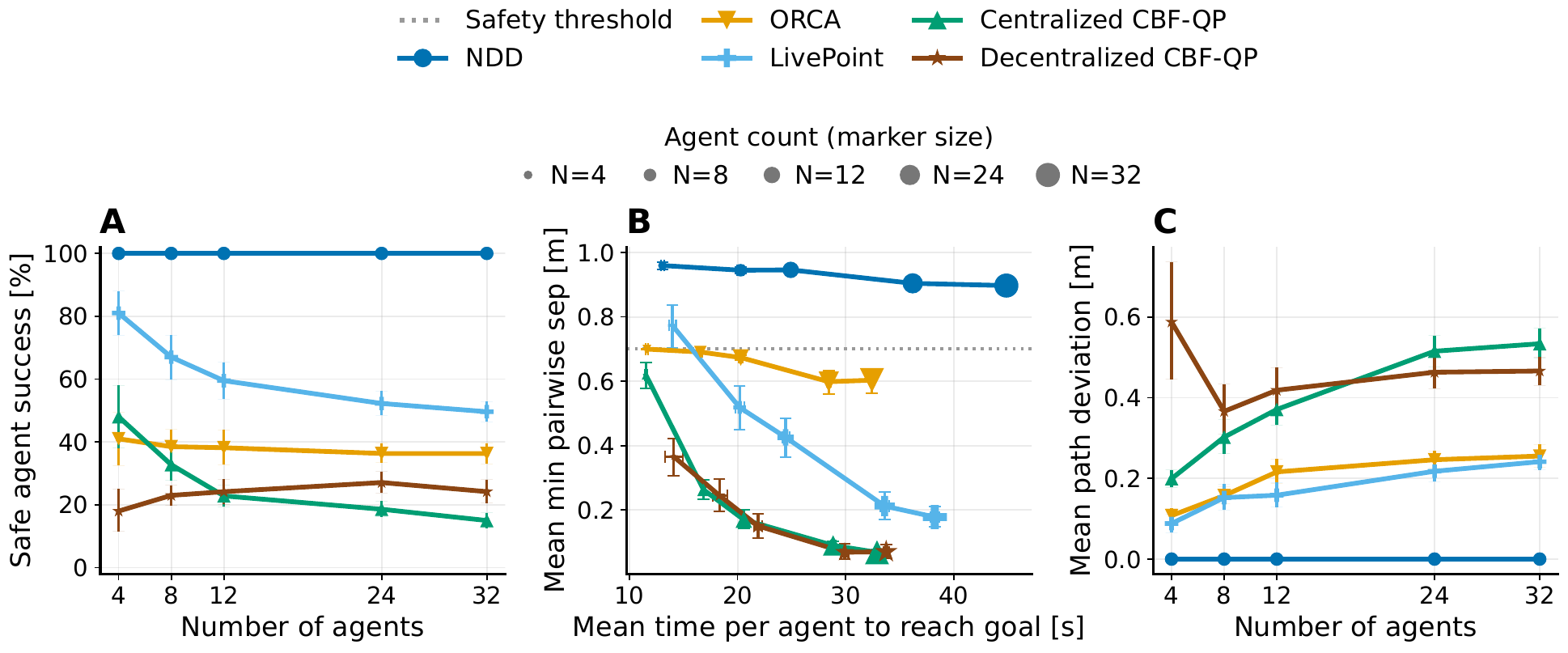}
    \caption{ {\textbf{NDD remains safe and on-path as congestion increases.} Methods are compared using homogeneous Dubins-car agents in the same 50 randomized start--goal scenarios per population size, $N\in\{4,8,12,24,32\}$ (minimum crossing angle $\alpha^e = 70^\circ$). Points are means over $n=50$ scenarios for each method and population size; error bars are $95\%$ bootstrap confidence intervals. \textbf{(A)} Safe-agent success is the percentage of agents reaching their goals without a safety violation. NDD maintains $100\%$ success, whereas the success rate of LivePoint, the strongest baseline by this measure, remains below $100\%$ and declines as congestion increases. \textbf{(B)} Mean minimum pairwise separation versus mean time per agent to reach its goal. Marker size denotes population size; the dotted line marks the $d^{\mathrm{safe}}=0.7$~m safety threshold. NDD retains a safe separation margin across all group sizes, although its completion time increases with population size. Shorter baseline times should not be interpreted as faster \edit{safe-and-live deconfliction} when separation falls below the safety threshold. \textbf{(C)} Mean deviation from prescribed paths. NDD maintains zero deviation, resolving conflicts through speed modulation rather than spatial detours. The panels distinguish safe completion from goal attainment alone and show that baseline path deviation does not ensure safety in these tests.} }
    \label{fig:simulation_benchmarks}
\end{figure}

\clearpage
\begin{figure}[p]
    \centering
    \includegraphics[width=\textwidth]{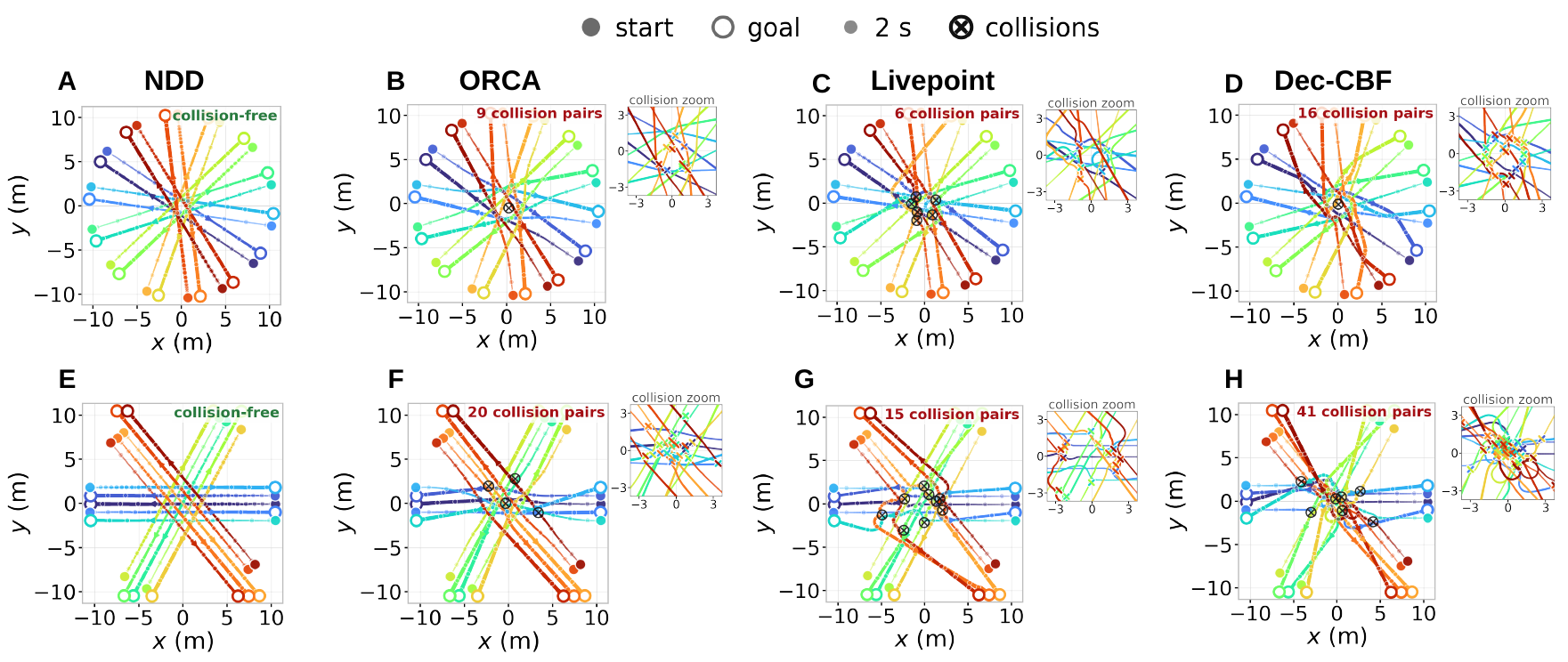}
     \caption{ {\textbf{NDD enables \edit{safe-and-live deconfliction} through shared crossing regions.} \edit{Panels show Dubins-car agents with unlimited sensing range and nominal speed $1$~m/s.} \textbf{(A to D)} Near-antipodal exchange, with agents starting around a circle and moving toward roughly opposite goals. \textbf{(E to H)} Intersecting groups of approximately parallel paths. Columns compare NDD (A and E), ORCA (B and F), LivePoint (C and G), and decentralized CBF-QP (D and H) using the same start--goal configuration within each row. Colors identify agents; filled and open circles mark starts and goals, respectively. Small dots mark positions at $2$~s intervals, so closer spacing indicates slower motion. Crossed circles mark pairs whose separation fell below the $d^{\mathrm{safe}}=0.7$~m safety distance; insets enlarge the corresponding regions. NDD-controlled agents remain collision-free in both scenarios while following their prescribed straight paths: spatial path intersections need not cause collisions because agents can traverse the same region at different times by modulating their speeds to create temporal separation. In contrast, ORCA, LivePoint, and decentralized CBF-QP record 9, 6, and 16 colliding pairs in the top row and 20, 15, and 41 in the bottom row, respectively.} }
    \label{fig:simulation_trajectories}
\end{figure}

\clearpage
\begin{figure}[p]
    \centering\includegraphics[width=\textwidth,height=0.72\textheight,keepaspectratio]{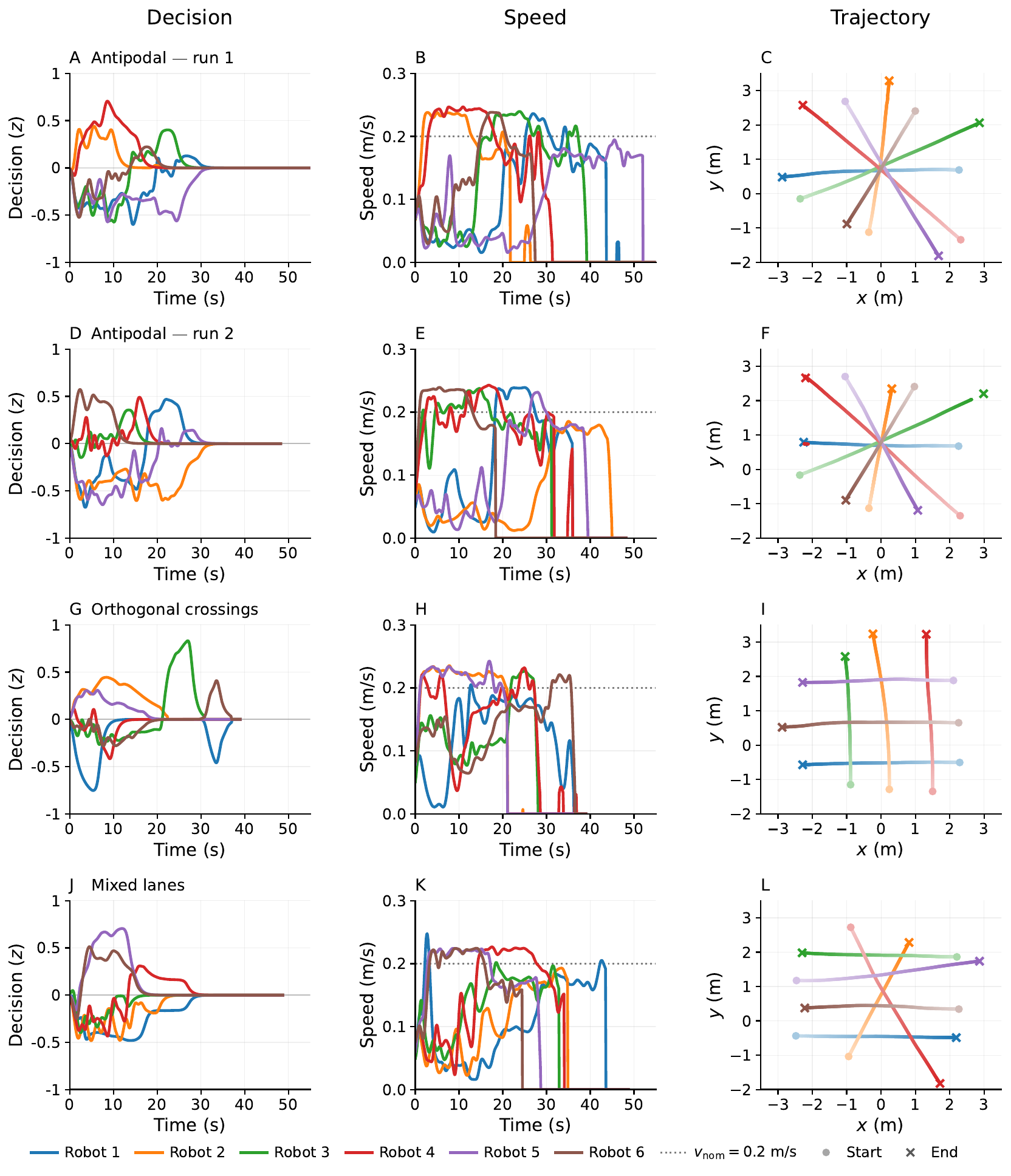}
    \caption{\small\textbf{Emergent traversal order and continuous adaptation in six-robot experiments.} Six TurtleBot3 robots follow prescribed paths without inter-robot communication. Columns show decisions $z_i$, measured speeds $v_i$, and planar trajectories; colors identify robots. Positive and negative decisions raise and lower commanded speeds relative to nominal. \textbf{(A to F)} Two antipodal repetitions with initial positions differing by $0.8$--$2.5$~cm per robot produce traversal orders $2\prec4\prec6\prec3\prec1\prec5$ (A to C) and $6\prec3\prec4\prec1\prec5\prec2$ (D to F). Robot~2 (orange) passes first in one trial and last in the other, illustrating how small motion asymmetries can be collectively amplified into distinct traversal orders without preassigned priorities. \textbf{(G to I)} In orthogonal crossings, robots~3 and~6 switch from yielding to speeding up at different stages of the encounter, illustrating continuous adaptation of their responses. \textbf{(J to L)} Mixed lanes show differentiated speed adjustments across multiple intersections. Across trials, decisions relax toward zero as robots complete their crossings. Gray dotted lines mark nominal speed $v^{\mathrm{nom}}=0.2\,\mathrm{m/s}$. Trajectories darken with distance traveled; circles and crosses mark recorded starts and ends.}
    \label{fig:hardware_experiments}
\end{figure}

\clearpage
\renewcommand{\thefigure}{S\arabic{figure}}
\renewcommand{\thetable}{S\arabic{table}}
\renewcommand{\theequation}{S\arabic{equation}}
\renewcommand{\thepage}{S\arabic{page}}
\renewcommand{\thesubsection}{S\arabic{subsection}}
\renewcommand{\thesubsubsection}{\thesubsection.\arabic{subsubsection}}
\renewcommand{\thetheorem}{S\arabic{theorem}}
\renewcommand{\thelemma}{S\arabic{lemma}}
\renewcommand{\theproposition}{S\arabic{proposition}}
\renewcommand{\thecorollary}{S\arabic{corollary}}
\setcounter{figure}{0}
\setcounter{table}{0}
\setcounter{equation}{0}
\setcounter{page}{1}
\setcounter{subsection}{0}
\setcounter{subsubsection}{0}
\setcounter{theorem}{0}
\setcounter{lemma}{0}
\setcounter{proposition}{0}
\setcounter{corollary}{0}

\begin{center}
\section*{Supplementary Materials for\\ \scititle}
Himani~Sinhmar$^{\ast}$,
Vaibhav~Srivastava,
Naomi~Ehrich~Leonard\\
\small$^\ast$Corresponding author. Email: himani.sinhmar@princeton.edu
\end{center}

\subsubsection*{This PDF file includes:}
Supplementary Text (Sections~\ref{sec:supp_practical_guide}--\ref{sec:numerical_bifurcation_illustration})\\
Figure~\ref{fig:supp_cubic_bifurcation_comparison}: Bifurcation diagrams and unfolding by timing bias.\\
Table~\ref{tab:principal_notation}: Notation: agent states and controller parameters\\
Table~\ref{tab:notation_interactions}: Notation: interaction structure and conflict evidence\\
Table~\ref{tab:notation_geometry}: Notation: conflict geometry, timing, and safety\\
Table~\ref{tab:notation_spectral}: Notation: spectral and bifurcation quantities\\
{Captions for Movies~S1--S9}\\
References \cite{kato1966perturbation,khalil2008nonlinear,golubitsky1985singularities}

\subsubsection*{{Other Supplementary Materials for this manuscript:}}
{Movies~S1--S9}

\clearpage

\section*{Supplementary Text}
\setlength{\emergencystretch}{2em}

\subsection{Practical guide to implementing NDD for a given mobile robot}
\label{sec:supp_practical_guide}

NDD is an event-triggered dynamic controller designed to sit between a robot's existing global path planner and its nominal motion controller. During nominal execution, the robot follows a planned straight or smoothly curved path in continuous time. When NDD predicts a conflict—evaluated in closed form using only local relative-motion observations—it autonomously modulates the robot's traversal speed to resolve the encounter.
If speed adaptation alone is insufficient (e.g., encountering a static obstacle blocking the path or a head-on conflict in a narrow corridor), NDD signals the architecture to trigger path replanning (Fig.~\ref{fig:overview}). Below, we outline a parameter-tuning workflow to implement NDD on any mobile robot platform, enabling real-time safe-and-live deconfliction under strict compute budgets.

\paragraph{Robot platform information.}
Start with the admissible speed range $[\underline v_i,\overline v_i]$, nominal speed $v_i^{\mathrm{nom}}$, speed-response time $\tau_v$, and stopping distance $D_i^{\mathrm{stop}}$ from the largest operating speed, including actuator delay. Also use an estimate of sensing/processing delay $T_{\mathrm{delay}}$ and an estimate of largest number $M_{\max}$ of neighbors the robot can simultaneously sense. Choose safety distance $d^{\mathrm{safe}}$ to cover robot size and any estimated tracking error. 

\paragraph{Pressure: choose the buffers, then the gains.}
For a representative encounter $e=(i,j)$ between robots $i$ and $j$, the conflict pressure is
\begin{equation}
p^e=p_i^e=\sigma(d_{\mathrm{prox}}-d_{\mathrm{ca}}^e;k_{\mathrm d})
\prod_{k\in\{i,j\}}\sigma(t_{\mathrm{prox}}+t_{k^+}^e-t_{\mathrm{ca}}^e;k_{\mathrm t}).
\label{eq:practical_pressure}
\end{equation}
Here $\sigma(x;k)=(1+e^{-kx})^{-1}$. The controller uses current motion to estimate the closest-approach distance and time $(d_{\mathrm{ca}}^e,t_{\mathrm{ca}}^e)$ and each robot's entry and exit times $t_{k^-}^e,t_{k^+}^e$. Entry and exit refer to the contested segment, i.e., the part of the path around a crossing where simultaneous occupancy could violate $d^{\mathrm{safe}}$.

First choose a small distance buffer $\Delta d>0$ for sensing and motion uncertainty, and set $d_{\mathrm{prox}}=d^{\mathrm{safe}}+\Delta d=k_{\mathrm{prox}}d^{\mathrm{safe}}$. A larger $k_{\mathrm{prox}}$ raises pressure for encounters in which the robots are expected to pass farther apart. Next choose a time buffer $t_{\mathrm{prox}}>0$. For example, $t_{\mathrm{prox}}=0.5$~s keeps an encounter relevant when closest approach is expected up to roughly half a second after the earlier of the two exit times. Pressure decreases smoothly around this margin.

Then choose gains that make each factor in Equation~\eqref{eq:practical_pressure} nearly equal to 1 in a critical encounter so that closest approach is within $d^{\mathrm{safe}}$ and occurs before either robot exits. As a starting point, target at least $0.95$ per factor by setting $k_{\mathrm d}=\log(19)/\Delta d$ and $k_{\mathrm t}=\log(19)/t_{\mathrm{prox}}$, where $19=0.95/0.05$. The resulting pressure satisfies $p^e\geq p_{\mathrm c}=0.95^3\approx0.86$.

For weak encounters, $p^e\leq p_{\mathrm w}=0.05$ if $d_{\mathrm{ca}}^e\geq d_{\mathrm{prox}}+\Delta d$ or closest approach occurs more than $2t_{\mathrm{prox}}$ after either exit. All encounters satisfy $p^e\leq\bar p=\sigma(d_{\mathrm{prox}};k_{\mathrm d})<1$. Smaller gains make the transitions gentler; larger gains make them sharper and more sensitive to uncertainty near the boundaries. To achieve other values of $p_{\mathrm c}$, $p_{\mathrm w}$, and $\bar p$ iterate on the above procedure. The method is robust, so for sparser congestion a critical pressure $p_c=0.5$ could suffice.  

\paragraph{Weights: choose temperature first, then pressure tilt.}
Choose $0<T_{\mathrm{sm}}<1$; $T_{\mathrm{sm}}=0.1$ is an illustrative starting value. Lower temperature $T_{\mathrm{sm}}$ concentrates weight on the most urgent interactions; higher temperature spreads weight and can dilute an urgent yielding cue in congestion. Choose a reference time gap $\Delta t$ between the estimated times at which robots $i$ and $j$ enter their respective contested segments. Make this gap larger than the uncertainty in the estimated time difference. Set $g_{\mathrm c}\in(0,1)$, for example $0.75$, as the desired magnitude $|g_i^e|$ of the arrival-order signal when the robots' estimated entry times differ by $\Delta t$. For $M_{\max}>1$, calculate
\[
\begin{aligned}
k_{\mathrm{tilt}}^{\min}
&=\frac{1-p_{\mathrm c}}{p_{\mathrm c}g_{\mathrm c}}
\left(\bar p-p_{\mathrm c}+T_{\mathrm{sm}}
\log\!\left(\frac{(M_{\max}-1)\bar p}{p_{\mathrm c}g_{\mathrm c}}\right)\right)_+,\\
k_{\mathrm{tilt}}^{\max}
&=\frac{(p_{\mathrm c}-p_{\mathrm w})(1-p_{\mathrm w})}{p_{\mathrm w}},
\qquad [x]_+=\max\{x,0\}.
\end{aligned}
\]
Choose $k_{\mathrm{tilt}}$ above the lower bound and below the upper bound, starting near the lower end. The lower bound makes a critical yielding cue dominate competing proceed cues; the upper bound keeps weak encounters from receiving excessive priority. If the interval is empty, reduce $T_{\mathrm{sm}}$ or increase pressure separation. With only one neighbor, its weight is one and tilt is unnecessary.

\paragraph{Gate: resolve meaningful timing differences.}
Set $k_{\mathrm g}=\operatorname{arctanh}(g_{\mathrm c})/\Delta t$ so that $|g_i^e|=g_{\mathrm c}$ when the estimated entry times differ by $\Delta t$. Larger $k_{\mathrm g}$ responds to smaller differences, including noise; smaller values give gentler speed differentiation. For each neighbor, compute
\[
\begin{aligned}
g_i^e&=\tanh\!\left(k_{\mathrm g}(t_{j^-}^e-t_{i^-}^e)\right),
&\widehat p_i^e&=p_i^e\left(1-\frac{k_{\mathrm{tilt}}g_i^e}{1-p_i^e}\right),\\
w_i^e&=\frac{\exp(\widehat p_i^e/T_{\mathrm{sm}})}{\sum_{e'}\exp(\widehat p_i^{e'}/T_{\mathrm{sm}})},
&h_i&=\sum_e w_i^e p_i^e g_i^e.
\end{aligned}
\]
The sums run over sensed neighbors. Positive $g_i^e$ favors proceeding; negative $g_i^e$ favors yielding, and the tilt increases its weight.

\paragraph{Decision and speed response.}
Implement
\[
\begin{aligned}
\tau_z\dot z_i&=-\rho_z z_i+\tanh(u_i h_i),
&\tau_u\dot u_i&=-\rho_u u_i+u_0+\kappa_u z_i^2,\\
v_i^{\mathrm{cmd}}&=v_i^{\mathrm{nom}}(1+\kappa_V\tanh(\kappa_z z_i)).
\end{aligned}
\]
Unity is a convenient starting choice for $\rho_z,\rho_u$. For fixed speed authority $\kappa_V$, increasing $\kappa_z$ produces larger speed changes from smaller decisions; the fractional-speed slope is $\kappa_V\kappa_z$. Choose it to reach the needed speed range and stopping trigger, and enforce hardware limits in the command interface. Exact zero-speed commands require $\kappa_V\tanh(\kappa_z/\rho_z)>1$ from zero initialization; a platform with a low-speed stopping trigger can use that threshold instead. Start with $\tau_z/\rho_z$ comparable to or slightly below $\tau_v$, and $\tau_u/\rho_u$ on the same scale. Smaller $\tau_z$ gives faster decisions; larger values smooth and delay the response. Actuator lag still limits how quickly speed can change.

\paragraph{Basal attention.}
A direct starting scale is $u_{\mathrm b}=1/(p_{\mathrm c}g_{\mathrm c})$, with $u_0=\rho_u u_{\mathrm b}$. This gives an order-one input to $\tanh$ for one dominant critical interaction. Increase $u_{\mathrm b}$ if competing cues leave the response too weak; decrease it if weak encounters cause unnecessary speed changes. Start with modest reinforcement $\kappa_u\geq0$: increasing it sustains decisions more strongly but can delay recovery.

\paragraph{Sensing range from platform response and stopping distance.}
Use several response time constants as an initial coordination allowance, for example
\[
T_{\mathrm{allow}}=5\max\{\tau_z/\rho_z,\tau_u/\rho_u,\tau_v\},
\qquad
L_i=D_i^{\mathrm{stop}}+\overline v_i(T_{\mathrm{delay}}+T_{\mathrm{allow}}).
\]
Choose the sensing range so both robots detect one another with at least $L_i$ remaining before contested-segment entry, plus an uncertainty margin. For identical robots with equal entry distances at crossing angle $0<\alpha<\pi$, this gives
\[
R_{\mathrm{sens}}\geq2(L_i+r_{\mathrm c})\sin(\alpha/2),
\qquad r_{\mathrm c}=d^{\mathrm{safe}}/|\sin\alpha|.
\]
Use the projected crossing geometry for unequal speeds or arrival times. Longer stopping distances and slower responses require more sensing range. Near ties or weak decision feedback can require more than the initial allowance; the factor five is a starting margin, not a universal negotiation-time bound. Increasing sensing range also increases the possible neighbor count, so use the corresponding $M_{\max}$ in the tilt calculation.

\paragraph{Pseudocode: NDD on robot $i$.}
\begin{enumerate}
\setlength{\itemsep}{0pt}
\item Initialize $z_i=0$, $u_i=u_{\mathrm b}$; obtain the path, nominal steering, and nominal speed from the existing navigation stack.
\item At each update, measure own motion and neighbors' relative positions and velocities; compute closest approach and contested-segment entry/exit times, with an explicit near-zero-speed convention. Request re-routing for head-on encounters.
\item If no neighbors are sensed, set $h_i=0$. Otherwise compute $p_i^e,g_i^e,\widehat p_i^e,w_i^e$ and sum $h_i$. 
\item Integrate the decision and attention equations with a step well below their response times. Send the admissible speed command together with nominal steering to the tracking controller.
\item Repeat, retaining the decision--attention--speed states as neighbors appear or disappear. As evidence vanishes, decisions relax and nominal speed returns. 
\end{enumerate}


\subsection{Assumptions and Outline of Remaining Supplementary Sections}

The following sections build up to prove finite-time safe-and-live deconfliction in Theorem~\ref{thm:finite_time_safe_deconfliction} of the main text. Under the conditions stated in each result, local NDD feedback establishes a traversal order, produces the temporal separation needed to realize that order safely, and adjusts decisions and speeds continuously as new physical conflicts arise. 
We start with the following assumptions:

\begin{assumption}[Finite conflict activations and no Zeno behavior]
\label{ass:episode}
A graph event is a conflict activation or physical clearance, as defined in Material and Methods Section. Every bounded time interval contains only finitely many graph events (no Zeno behavior). Each conflict episode contains finitely many conflict activations, including those involving agents entering the population and those that merge groups.
\end{assumption}

Assumption~\ref{ass:episode} holds naturally for a finite population of agents traversing straight paths with nondegenerate crossings. Each pair has at most one crossing encounter, and retaining an activated conflict until physical clearance prevents repeated switching during that encounter. Consequently, the number of graph events is finite. The same reasoning extends to piecewise-straight paths with finitely many segments and to curved paths with finitely many crossing encounters. Episodes involving agent arrivals are also covered when finitely many agents participate in each episode.

The term ``final'' is relative to the current episode; an agent's final conflict is the last of its physical conflicts to clear in that episode.

\begin{assumption}[Reciprocal sensing]
\label{ass:reciprocal_sensing}
For each conflicting pair considered in the analysis, both agents can sense one another.
\end{assumption}

\paragraph{Remark (Heterogeneous sensing radii).}
The reciprocal-sensing assumption is conservative: with heterogeneous sensing radii, we analyze each pair from mutual detection onward. An agent with a larger radius may detect the conflict and begin responding earlier.

\begin{assumption}[Stable speed tracking]
\label{ass:tracking}
For each admissible constant speed command, the realized-speed dynamics have a locally unique, exponentially stable equilibrium at the commanded speed, with a known bounded transient response within the agent's actuation limits.
\end{assumption}
At an equilibrium with a constant speed command, Assumption~\ref{ass:tracking} and Equations~\eqref{eq:cmd_speed}--\eqref{eq:speed_response} give $v_i=v_i^{\mathrm{nom}}+V(z_i)$.

The proof proceeds as follows:
\begin{itemize}
\item \textbf{Section~\ref{sec:supp_pressure_regimes}: Separation of interaction regimes and time-scale by design.} The pressure gains separate weak from safety-critical interactions and keep pressure nearly constant during negotiation. The response-time and sensing conditions give agents enough time to negotiate a traversal order before they enter their contested segments.

\item \textbf{Section~\ref{sec:supp_spectral_tuning}: Local control design ensures the emergence of a traversal order across an engaged group.} Theorem~\ref{thm:autonomous_spectral_tuning} shows how local feedback autonomously destabilizes collective indecision and selects a differentiated equilibrium of the decision--attention--speed subsystem. The subsystem state tracks this equilibrium as the conflict geometry slowly evolves, and the resulting fractional-speed differences define a strict acyclic traversal order.

\item \textbf{Section~\ref{sec:parameter_design}: Safety during negotiation and traversal.} The yielding and stopping bounds preserve safety during negotiation. Proposition~\ref{prop:fixed_component_edge_safety} gives a sensing range that allows the selected speed differences to produce positive temporal clearance. Theorem~\ref{thm:fixed_component_safe_order} combines these conditions to guarantee safe traversal and a leading mover that clears a physical conflict in finite time.

\item \textbf{Section~\ref{sec:decision_to_global_deconfliction}: Safe-and-live deconfliction across a changing conflict graph.} Theorem~\ref{thm:graph_reconfiguration_induction} preserves safety when conflicts clear, groups split or merge, and agents entering the population introduce new conflicts. After a conflict activation, the affected group establishes a new traversal order whenever its updated physical-conflict set remains fixed for $T^{\mathrm{neg}}$. After the final conflict activation in an episode, all remaining physical conflicts in that episode clear in finite time. Section~\ref{sec:supp_proof_global} combines these results to prove the main theorem.

\item \textbf{Section~\ref{sec:supp_local_bifurcation}: Local bifurcation analysis at collective indecision.} We perform a Lyapunov-Schmidt reduction to derive the scalar critical equation that determines the local branches of differentiated fast subsystem equilibria used in Section~\ref{sec:supp_spectral_tuning}. Its coefficients describe how controller parameters and projected entry-time offsets affect these branches and their stability.

\item \textbf{Section~\ref{sec:numerical_bifurcation_illustration}: Numerical illustration of bifurcation-driven deconfliction.} A three-agent example compares equilibria before and after a physical conflict clears. It illustrates how equilibrium decisions can change sign while the traversal orders on the remaining conflicts persist.
\end{itemize}

\subsection{Separation of interaction regimes and time-scale by design}
\label{sec:supp_pressure_regimes}

We consider the admissible conflict episodes specified in Materials and Methods (\textit{Admissible conflict episodes and scope}) and the collective interaction dynamics in Equation~\eqref{eq:stacked_ndd}. For agent $k$, let $\underline v_k$ and $\overline v_k$ denote its lower and upper admissible traversal speeds. The upper bound applies to realized motion, $v_k(t)\leq\overline v_k$, including tracking and actuation transients; hence the agent travels at most $\overline v_kT$ over an interval of duration $T$.
We choose the controller and sensing parameters to give agents enough time to negotiate a traversal order before they enter their contested segments. We first tune the spatial and temporal pressure gains $k_{\mathrm d},k_{\mathrm t}$ in Equation~\eqref{eq:conflict_pressure} to maintain high, slowly varying pressure for safety-critical interactions and low pressure for weak interactions. We then choose the negotiation duration $T^{\mathrm{neg}}$ to accommodate the coupled decision, attention, and speed response within actuation limits. The sensing radius $R_{\mathrm{sens}}$ must provide mutual detection at $t_0$ with $\overline v_kT^{\mathrm{neg}}\ll d_{k^-}^e(t_0)$, so each agent travels only a small fraction of its remaining distance to its contested segment during negotiation, as required by Equation~\eqref{eq:time_scale_separation}.

\paragraph{Enforcing separation of interaction regimes by design.}
Conflict pressure $p^e$ measures how strongly a projected encounter calls for a response. When the projected motion already provides sufficient spatial or temporal separation, the interaction is weak. Its pressure is at most $p_{\mathrm w}$, the \emph{weak-pressure level}. When the projected encounter brings the agents within the safety distance before either exits its contested segment, the interaction is critical. The pressure gains are chosen so that every such interaction has pressure at least $p_{\mathrm c}$, the \emph{critical-pressure level}.
After one agent passes the exit of its contested segment ($i^+$ or $j^+$ in Fig.~\ref{fig:contested_regions}, A), the pair's pressure can decrease through either the spatial or temporal factor in Equation~\eqref{eq:conflict_pressure}. An increase in the predicted closest-approach distance $d_{\mathrm{ca}}^e$ reduces the spatial factor $P_{\mathrm d}$. An earlier projected exit from either contested segment relative to closest approach reduces the temporal factor $P_{\mathrm t}$. Both mechanisms contribute to release of the critical interaction
and gives $p^e\leq p_{\mathrm r}$, the \emph{release-pressure level}.

The sigmoid gains $k_{\mathrm d},k_{\mathrm t}$ in Equation~\eqref{eq:conflict_pressure} enforce a gap $p_{\mathrm w}<p_{\mathrm c}$ between weak and critical interactions and a uniform upper bound $\bar p<1$. Pressure varies continuously between these regimes.

\begin{lemma}[Separation of interaction regimes]
\label{prop:uniform_pressure_bands}
Consider the pressure in~\eqref{eq:conflict_pressure}, with $d_{\mathrm{prox}}>d^{\mathrm{safe}}$ and $t_{\mathrm{prox}}>0$. Fix $\delta_{\mathrm d},\delta_{\mathrm t}>0$ and $\delta_{\mathrm r}\geq\delta_{\mathrm t}$. Finite positive gains $k_{\mathrm d},k_{\mathrm t}$ can be chosen so that
\begin{equation}
0<p_{\mathrm r}\leq p_{\mathrm w}<p_{\mathrm c}\leq\bar p<1.
\label{eq:ordered_pressure_bounds}
\end{equation}
Every interaction satisfies $0<p^e\leq\bar p$, and:
\begin{enumerate}
\item $p^e\geq p_{\mathrm c}$ if $d_{\mathrm{ca}}^e\leq d^{\mathrm{safe}}$ and $t^e_{i^+},t^e_{j^+}\geq t_{\mathrm{ca}}^e$;
\item $p^e\leq p_{\mathrm w}$ if $d_{\mathrm{ca}}^e\geq d_{\mathrm{prox}}+\delta_{\mathrm d}$, or $t^e_{k^+}-t_{\mathrm{ca}}^e\leq-t_{\mathrm{prox}}-\delta_{\mathrm t}$ for some $k\in\{i,j\}$;
\item after leader $\ell$ clears its contested segment, $p^e\leq p_{\mathrm r}$ whenever $t^e_{\ell^+}-t_{\mathrm{ca}}^e\leq-t_{\mathrm{prox}}-\delta_{\mathrm r}$.
\end{enumerate}
\end{lemma}

\begin{proof}
Set
\begin{equation*}
\begin{aligned}
\bar p&:=\sigma(d_{\mathrm{prox}};k_{\mathrm d}),\\
p_{\mathrm c}&:=\sigma(d_{\mathrm{prox}}-d^{\mathrm{safe}};k_{\mathrm d})
\sigma^2(t_{\mathrm{prox}};k_{\mathrm t}),\\
p_{\mathrm w}&:=\max\left\{
\sigma(-\delta_{\mathrm d};k_{\mathrm d}),
\bar p\,\sigma(-\delta_{\mathrm t};k_{\mathrm t})
\right\},\\
p_{\mathrm r}&:=\bar p\,\sigma(-\delta_{\mathrm r};k_{\mathrm t}).
\end{aligned}
\end{equation*}
Since $d_{\mathrm{ca}}^e\geq0$, the spatial factor is at most $\bar p$. In the critical regime it is at least $\sigma(d_{\mathrm{prox}}-d^{\mathrm{safe}};k_{\mathrm d})$, and both temporal factors are at least $\sigma(t_{\mathrm{prox}};k_{\mathrm t})$. In the weak regime, either the spatial factor is at most $\sigma(-\delta_{\mathrm d};k_{\mathrm d})$ or a temporal factor is at most $\sigma(-\delta_{\mathrm t};k_{\mathrm t})$. The release condition similarly bounds the leader's temporal factor by $\sigma(-\delta_{\mathrm r};k_{\mathrm t})$. As the gains increase, $p_{\mathrm c}\to1$ and $p_{\mathrm w}\to0$, so finite gains ensure $p_{\mathrm w}<p_{\mathrm c}$. The remaining inequalities follow from $\delta_{\mathrm r}\geq\delta_{\mathrm t}$ and the sigmoid bounds.
\end{proof}

\paragraph{Enforcing time-scale separation by design.}
As established in the Results (\textit{Collective interaction dynamics}), a safe traversal order must emerge within a finite negotiation duration, $T^{\mathrm{neg}}>0$. Let $\mathcal I^{\mathrm{neg}}:=[t_0,t_0+T^{\mathrm{neg}}]$ denote this negotiation interval, triggered at the moment of mutual detection $t_0$. For NDD to successfully adjust speeds based on changing arrival-order cues (carried by $G$), the conflict pressure must remain high and nearly constant throughout this phase. While Lemma~\ref{prop:uniform_pressure_bands} guarantees a lower bound of $p^e\geq p_{\mathrm c}$ during safety-critical interactions, we must explicitly tune the pressure gains to bound the maximum pressure variation over $\mathcal I^{\mathrm{neg}}$ by a small constant $\varepsilon_p>0$. We also select the sensing range $R_{\mathrm{sens}}$ to ensure that this negotiation naturally concludes before any agent reaches a contested segment.



\begin{assumption}[Regular projected geometry]
\label{ass:regular_projected_geometry}
On each interval with fixed sensed interactions, the predicted closest-approach distance and exit-time offsets satisfy $|\dot d_{\mathrm{ca}}^e|\leq L_{\mathrm d}$ and $|\dot t^e_{k^+}-\dot t_{\mathrm{ca}}^e|\leq L_{\mathrm t}$ for $e=(i,j)$ and $k\in\{i,j\}$, and finite constants $L_{\mathrm d},L_{\mathrm t}\geq0$.
\end{assumption}
Here $t^e_{k^+}$ is the predicted time at which agent $k$ reaches the exit $k^+$ of its contested segment ($i^+$ or $j^+$ in Fig.~\ref{fig:contested_regions}, A). 
The distance $d_{\mathrm{ca}}^e$ and the offsets $t^e_{k^+}-t_{\mathrm{ca}}^e$ determine the spatial and temporal factors of Equation~\eqref{eq:conflict_pressure}, respectively. Known speed and acceleration bounds provide $L_{\mathrm d},L_{\mathrm t}$ on intervals where the kinematic projections have bounded sensitivity to the observed motion. 

\begin{lemma}[Pressure variation during negotiation]
\label{prop:pressure_response_tuning}
Under Assumption~\ref{ass:regular_projected_geometry}, suppose the sensed interactions are fixed and $d_{\mathrm{ca}}^e\leq d^{\mathrm{safe}}$, $t^e_{i^+},t^e_{j^+}\geq t_{\mathrm{ca}}^e$ throughout $\mathcal I^{\mathrm{neg}}$. The pressure rate admits a uniform bound $|\dot p^e|\leq L_p(k_{\mathrm d},k_{\mathrm t})$. For any $\varepsilon_p>0$, finite positive gains $k_{\mathrm d},k_{\mathrm t}$ can be chosen to satisfy both Equation~\eqref{eq:ordered_pressure_bounds} and
\begin{equation}
T^{\mathrm{neg}}L_p(k_{\mathrm d},k_{\mathrm t})\leq\varepsilon_p.
\label{eq:supp_pressure_variation_design}
\end{equation}
Every such choice gives $p^e(t)\geq p_{\mathrm c}$ and $|p^e(t)-p^e(s)|\leq\varepsilon_p$ for $s,t\in\mathcal I^{\mathrm{neg}}$.
\end{lemma}

\begin{proof}
The lower pressure bound follows from Lemma~\ref{prop:uniform_pressure_bands}. Differentiating Equation~\eqref{eq:conflict_pressure} with $\partial_x\sigma(x;k)=k\sigma(x;k)[1-\sigma(x;k)]$ gives
\begin{equation*}
\begin{aligned}
\dot p^e
={}&-k_{\mathrm d}p^e\bigl[1-\sigma(d_{\mathrm{prox}}-d_{\mathrm{ca}}^e;k_{\mathrm d})\bigr]\dot d_{\mathrm{ca}}^e\\
&+k_{\mathrm t}p^e\sum_{k\in\{i,j\}}
\bigl[1-\sigma(t_{\mathrm{prox}}+t^e_{k^+}-t_{\mathrm{ca}}^e;k_{\mathrm t})\bigr]
(\dot t^e_{k^+}-\dot t_{\mathrm{ca}}^e).
\end{aligned}
\end{equation*}
Using $0<p^e\leq1$ and $1-\sigma(x;k)\leq e^{-kx}$, the critical-regime conditions and Assumption~\ref{ass:regular_projected_geometry} yield
\begin{equation*}
\begin{aligned}
|\dot p^e|
&\leq k_{\mathrm d}e^{-k_{\mathrm d}(d_{\mathrm{prox}}-d_{\mathrm{ca}}^e)}|\dot d_{\mathrm{ca}}^e|
+k_{\mathrm t}\sum_{k\in\{i,j\}}e^{-k_{\mathrm t}(t_{\mathrm{prox}}+t^e_{k^+}-t_{\mathrm{ca}}^e)}
|\dot t^e_{k^+}-\dot t_{\mathrm{ca}}^e|\\
&\leq L_{\mathrm d}k_{\mathrm d}e^{-k_{\mathrm d}(d_{\mathrm{prox}}-d^{\mathrm{safe}})}
+2L_{\mathrm t}k_{\mathrm t}e^{-k_{\mathrm t}t_{\mathrm{prox}}}
=:L_p(k_{\mathrm d},k_{\mathrm t}).
\end{aligned}
\end{equation*}
Since $d_{\mathrm{prox}}-d^{\mathrm{safe}}>0$ and $t_{\mathrm{prox}}>0$, each term in $L_p$ tends to zero as its corresponding gain tends to infinity. Thus finite gains satisfy both Equations~\eqref{eq:ordered_pressure_bounds} and~\eqref{eq:supp_pressure_variation_design}. Integration gives $|p^e(t)-p^e(s)|\leq L_p|t-s|\leq L_pT^{\mathrm{neg}}\leq\varepsilon_p$.
Note that, sufficiently large $k_{\mathrm d},k_{\mathrm t}$ place the critical-regime sigmoids on nearly flat portions near one. 
\end{proof}

Because conflict pressure depends on closest-approach quantities, it remains largely invariant during the negotiation phase, even as speed adjustments alter the projected encounter. For instance, two agents traveling at nearly constant speeds along straight segments prior to their closest approach will maintain a constant predicted closest-approach distance. In this scenario, their projected exit and closest-approach times shift in tandem, leaving the difference $t^e_{k^+}-t_{\mathrm{ca}}^e$ unchanged. Hence, the conflict pressure remains constant despite the agents' ongoing motion.

The speed tracker and actuator limits constrain how quickly agents can change their realized speeds during negotiation.
Given the speed-response time $\tau_v$ and actuation limits, choose a feasible $T^{\mathrm{neg}}$ and tune $\tau_z,\tau_u$ so that $T_{\mathrm{fast}}:=\max\{\tau_z/\rho_z,\tau_u/\rho_u,\tau_v\}\lesssim T^{\mathrm{neg}}$, with the pressure gains satisfying Equation~\eqref{eq:supp_pressure_variation_design} for that duration.
To ensure that negotiation concludes well before the agents reach their contested segments, we limit the distance each agent travels during negotiation to a prescribed small fraction $0<\varepsilon_{\mathrm{trav}}\ll1$ of its remaining distance to its contested segment, measured at mutual detection. We enforce this limit by selecting a sensing radius that provides mutual detection with
\begin{equation}
d_{k^-}^e(t_0)\geq\frac{\overline v_kT^{\mathrm{neg}}}{\varepsilon_{\mathrm{trav}}},
\qquad k\in\{i,j\}.
\label{eq:supp_negotiation_sensing_budget}
\end{equation}
With $T_{\mathrm{trav}}^e:=\min_{k\in\{i,j\}}d_{k^-}^e(t_0)/\overline v_k$, the response and detection choices give
\begin{equation}
T_{\mathrm{fast}}\lesssim T^{\mathrm{neg}}
\leq\varepsilon_{\mathrm{trav}}T_{\mathrm{trav}}^e
\ll T_{\mathrm{trav}}^e.
\label{eq:supp_joint_timescale_budget}
\end{equation}
The chosen travel fraction $\varepsilon_{\mathrm{trav}}$ therefore sets the time-scale separation, and the sensing condition enforces it. Lemmas~\ref{prop:uniform_pressure_bands} and~\ref{prop:pressure_response_tuning} keep pressure high and slowly varying as long as the interaction remains safety-critical. Pressure continues to be updated from observed motion.

\paragraph{Remark (Practical choice of response times and sensing range).}
In practice, we start from the speed-response time $\tau_v$, actuation limits, and perception update rate. We choose the decision and attention response times $\tau_z/\rho_z,\tau_u/\rho_u$ and negotiation duration $T^{\mathrm{neg}}$ so that the actuators can produce the commanded speed changes and perception updates can resolve the resulting changes in relative motion. Once $\tau_z/\rho_z$ is shorter than the speed-response time or the time needed for motion to produce observable changes in arrival-order evidence, further reducing $\tau_z$ may bring little improvement in the overall negotiation time. 

Slower speed tracking, limited acceleration, or a lower perception update rate may require a longer negotiation interval. We accommodate that interval by increasing $R_{\mathrm{sens}}$ to detect conflicts earlier and retain $T^{\mathrm{neg}}\leq\varepsilon_{\mathrm{trav}}T_{\mathrm{trav}}^e$ in Equation~\eqref{eq:supp_joint_timescale_budget}.
A larger sensing range may also increase $|\mathcal E_i(t)|$ and its design bound $M_{\max}$. We therefore check that the processor can handle these additional interactions at the selected perception rate. Evaluating the evidence in Equation~\eqref{eq:ndd:evidence} takes $\mathcal O(|\mathcal E_i(t)|)$ work per update, so this cost grows linearly with the number of sensed interactions.


\paragraph{Remark (admissible pressure functions)}
The pressure map in~Equation~\eqref{eq:conflict_pressure} is one admissible realization. An alternative map to $[0,1)$ may be used if it is 1) pair-invariant, so $P(\mathcal X_i^e)=P(\mathcal X_j^{\bar e})$; 2) sufficiently smooth on each fixed-interaction domain; 3) increasing with projected conflict severity and decaying after clearance; and 4) able to realize the ordered critical, weak, and release levels of Lemma~\ref{prop:uniform_pressure_bands} while keeping safety-critical pressure slowly varying during the coupled decision, attention, and speed response described in Lemma~\ref{prop:pressure_response_tuning}. Pair invariance makes pressure a shared severity measure, while $G$ supplies the currently favored pairwise traversal-order sign.

\subsection{Local control design ensures the emergence of a traversal order across an engaged group}
\label{sec:supp_spectral_tuning}

We analyze the steady states of the fast decision--attention--speed subsystem to show how a collective traversal order emerges from local NDD controllers and how local parameter choices tune this emergent response.
The time-scale separation established in Section~\ref{sec:supp_pressure_regimes} allows us to hold the variables evolving on slower time-scales (remaining traversal distances and sensed interaction network) fixed, during the negotiation, while retaining the feedback between decisions, attention, realized speeds, and speed-dependent conflict evidence.

The entries of the interaction-sensitivity matrix $\boldsymbol A$ in Equation~\eqref{eq:evidence_jacobian} quantify how decision-induced speed changes affect each agent's conflict evidence. They depend on the locally tunable pressure and weighting parameters and the decision-to-speed map.
As agents sense and respond to one another's motion, their local NDD feedback autonomously shapes the entries of $\boldsymbol A$. Its dominant eigenvalue $\lambda_{\max}(\boldsymbol A)$ captures how these local interactions collectively govern the stability of indecision.
We identify a threshold $\lambda^\ast$ such that collective indecision loses stability when $\lambda_{\max}(\boldsymbol A)$ crosses it from below. We derive local controller conditions that ensure pressure-dependent feedback drives this eigenvalue above $\lambda^\ast$ when any interaction becomes safety-critical and returns it below $\lambda^\ast$ when all interactions are weak.
Under the local bifurcation conditions in Section~\ref{sec:supp_local_bifurcation}, this transition gives rise to attracting equilibria with differentiated decisions and speeds.
We show how the resulting speed differences between conflicting agents, measured relative to their nominal speeds, define a consistent traversal order.
We then establish conditions under which the decision--attention--speed state tracks the changing equilibrium closely enough to preserve this order as the geometry evolves slowly.
As conflicts clear, declining pressure brings the system back below threshold, allowing decisions and attention to relax and agents to resume nominal motion.


\subsubsection{The interaction-sensitivity matrix governs the onset of collective differentiation}

At a group tie, agents moving at their nominal speeds have equal projected entry times on each physical conflict. The signed arrival-order evidence therefore vanishes, and the decision--attention--speed subsystem has a collective-indecision equilibrium. Equation~\eqref{eq:reduced_steady_state} describes its equilibria using only the decision variables. The following lemma establishes that this reduction preserves the equilibria of the coupled subsystem. It also identifies when the full subsystem Jacobian has a zero eigenvalue and relates the associated eigenvector to the interaction-sensitivity matrix $\boldsymbol A$. 

\begin{lemma}[Spectral condition for steady-state bifurcation]
\label{lem:jacobian_steady_state}
Under Assumption~\ref{ass:tracking} and the time-scale separation established in Section~\ref{sec:supp_pressure_regimes}, the fast subsystem has the following properties.
\begin{enumerate}
\item The zeros of Equation~\eqref{eq:reduced_steady_state} correspond locally one-to-one with equilibria of the fast decision--attention--speed subsystem in Equation~\eqref{eq:stacked_ndd}.
\item At a group tie, let $\boldsymbol J_\ast$ be the Jacobian of this subsystem evaluated at collective indecision, let $\boldsymbol K:=D_{\boldsymbol z}\boldsymbol V(\boldsymbol0)$, and evaluate the interaction-sensitivity matrix $\boldsymbol A$ from Equation~\eqref{eq:evidence_jacobian} at the tie. The Jacobian $\boldsymbol J_\ast$ has a zero eigenvalue exactly when $\boldsymbol A$ has the eigenvalue $\lambda^\ast=\rho_z\rho_u/[S'(0)u_0]$ defined in Equation~\eqref{eq:formation_spectral_condition}. The associated eigenvectors of $\boldsymbol J_\ast$  are exactly $\operatorname{col}(\boldsymbol\phi,\boldsymbol0,\boldsymbol K\boldsymbol\phi)$, where $\boldsymbol A\boldsymbol\phi=\lambda^\ast\boldsymbol\phi$ and $\boldsymbol\phi\neq\boldsymbol0$.
\end{enumerate}
\end{lemma}

\begin{proof}
For the first claim, the attention equation and Assumption~\ref{ass:tracking} give the locally unique equilibrium states
\begin{equation}
\label{eq:supp_fast_equilibrium_states}
\boldsymbol u_{\mathrm{ss}}(\boldsymbol z)
=\frac{u_0}{\rho_u}\boldsymbol 1
+\frac{\kappa_u}{\rho_u}(\boldsymbol z\odot\boldsymbol z),
\qquad
\boldsymbol v_{\mathrm{ss}}(\boldsymbol z)
=\boldsymbol v^{\mathrm{nom}}+\boldsymbol V(\boldsymbol z).
\end{equation}
Substitution into the decision equation yields Equation~\eqref{eq:reduced_steady_state}. Conversely, each zero of $\boldsymbol F$ reconstructs a unique local equilibrium through these expressions.

For the second claim, we invoke the time-scale separation in Section~\ref{sec:supp_pressure_regimes} to treat the slowly evolving traversal distances as fixed during the fast response. 
Define the derivatives at collective indecision by
\begin{equation*}
\boldsymbol H_v:=D_{\boldsymbol v}\boldsymbol h\vert_\ast,
\qquad
\boldsymbol Q:=D_{\boldsymbol v}\boldsymbol f_v\vert_\ast,
\qquad
\boldsymbol C:=D_{\boldsymbol v^{\mathrm{cmd}}}\boldsymbol f_v\vert_\ast.
\end{equation*}
Here $\boldsymbol H_v$ is the evidence Jacobian with respect to realized speeds; $\boldsymbol Q$ and $\boldsymbol C$ are the Jacobians of the speed-tracking dynamics with respect to realized and commanded speeds, respectively. Exponential stability of speed tracking makes $\boldsymbol Q$ nonsingular. Assumption~\ref{ass:tracking} also states that equilibrium speed equals commanded speed. Changing the realized and commanded speeds by the same infinitesimal amount therefore leaves the speed tracker at equilibrium, giving $\boldsymbol Q+\boldsymbol C=\boldsymbol0$ and hence $-\boldsymbol Q^{-1}\boldsymbol C=\boldsymbol I$.
Let $\delta\boldsymbol x=\operatorname{col}(\delta\boldsymbol z,\delta\boldsymbol u,\delta\boldsymbol v)$ denote small deviations from collective indecision $(\boldsymbol0,(u_0/\rho_u)\boldsymbol1,\boldsymbol v^{\mathrm{nom}})$. Their linearized dynamics are $\delta\dot{\boldsymbol x}=\boldsymbol J_\ast\delta\boldsymbol x$, with
\begin{equation*}
\begin{gathered}
\boldsymbol J_\ast
=\operatorname{diag}\!\left(\tau_z^{-1}\boldsymbol I,
\tau_u^{-1}\boldsymbol I,\tau_v^{-1}\boldsymbol I\right)
\boldsymbol J_{\mathrm{eq}},\\[1ex]
\boldsymbol J_{\mathrm{eq}}
=\begin{pmatrix}
-\rho_z\boldsymbol I&\boldsymbol0&S'(0)\dfrac{u_0}{\rho_u}\boldsymbol H_v\\
\boldsymbol0&-\rho_u\boldsymbol I&\boldsymbol0\\
\boldsymbol C\boldsymbol K&\boldsymbol0&\boldsymbol Q
\end{pmatrix}.
\end{gathered}
\end{equation*}
The decision--attention block vanishes because $\boldsymbol h=\boldsymbol0$ at the tie; the attention--decision block vanishes because the decision-dependent attention input is quadratic. The block $\boldsymbol C\boldsymbol K$ accounts for how decisions change the speed command.
Eliminating the nonsingular attention and speed blocks by a Schur complement gives
\begin{equation*}
\det\boldsymbol J_{\mathrm{eq}}
=\det(-\rho_u\boldsymbol I)\det(\boldsymbol Q)
\det\!\left(-\rho_z\boldsymbol I
+S'(0)\frac{u_0}{\rho_u}\boldsymbol A\right),
\end{equation*}
where $-\boldsymbol Q^{-1}\boldsymbol C=\boldsymbol I$ and the interaction-sensitivity matrix is $\boldsymbol A=\boldsymbol H_v\boldsymbol K$. The first two determinant factors are nonzero. The positive time constants also rescale the determinant by a nonzero factor. Thus $\boldsymbol J_\ast$ has a zero eigenvalue exactly when $\boldsymbol A$ has the eigenvalue $\lambda^\ast=\rho_z\rho_u/[S'(0)u_0]$.

The block equations $\boldsymbol J_\ast\delta\boldsymbol x=\boldsymbol0$ give $\delta\boldsymbol u=\boldsymbol0$, $\delta\boldsymbol v=\boldsymbol K\delta\boldsymbol z$, and $\boldsymbol A\delta\boldsymbol z=\lambda^\ast\delta\boldsymbol z$. Hence the eigenvectors of $\boldsymbol J_\ast$ associated with zero eigenvalues are precisely
\begin{equation*}
\delta\boldsymbol x=
\begin{pmatrix}
\boldsymbol\phi\\
\boldsymbol0\\
\boldsymbol K\boldsymbol\phi
\end{pmatrix},
\qquad
\boldsymbol A\boldsymbol\phi=\lambda^\ast\boldsymbol\phi,
\quad \boldsymbol\phi\neq\boldsymbol0.
\end{equation*}
\end{proof}

Equation~\eqref{eq:ndd:evidence} includes both physical-conflict and nonconflicting sensed interactions. Recall that $\mathcal E_i^{\mathrm{conf}}(t)\subseteq\mathcal E_i(t)$ is the set of physical-conflict interactions. We decompose the evidence as
	\[
	h(\mathcal X_i)=h_i^{\mathrm{conf}}+h_i^{\mathrm{nc}},
	\]
	where
	\[
	h_i^{\mathrm{conf}}
	=\sum_{e\in\mathcal E_i^{\mathrm{conf}}(t)}
	P(\mathcal X_i^e)W(\mathcal X_i^e;\mathcal X_i)G(\mathcal X_i^e)
	\quad \text{and} \quad 
	h_i^{\mathrm{nc}}=h(\mathcal X_i)-h_i^{\mathrm{conf}}.
	\]
	
	Consider a locally engaged group $\mathcal G^c$, a strongly connected component of the physical-conflict graph $\mathcal G$. Its agent set is $\mathcal V^c$, and its physical-conflict pairs form $\mathcal P^c$. At the reference group tie, the corresponding within-component block of the interaction-sensitivity matrix $\boldsymbol A$ decomposes as
	\[
	\boldsymbol A_{cc}
	=\boldsymbol A_{cc}^{0}+\Delta\boldsymbol A_{cc},
	\]
	where, for $i,j\in\mathcal V^c$,
	\[
	(\boldsymbol A_{cc}^{0})_{ij}
	=\left.\frac{\partial h_i^{\mathrm{conf}}}{\partial z_j}\right|_*
	\quad \text{and} \quad 
	(\Delta\boldsymbol A_{cc})_{ij}
	=\left.\frac{\partial h_i^{\mathrm{nc}}}{\partial z_j}\right|_*.
	\]
	
	We next characterize the spectrum of the physical-conflict interaction-sensitivity matrix $\boldsymbol A_{cc}^{0}$.

We next characterize the spectrum of the interaction-sensitivity matrix $\boldsymbol A$ on a locally engaged group $\mathcal G^c$, a strongly connected component of the physical-conflict graph $\mathcal G$. Its agent set is $\mathcal V^c$, and its physical-conflict pairs form $\mathcal P^c$. 
\begin{proposition}[Reciprocal interactions yield differentiated critical eigenmodes]
\label{prop:interaction_sensitivity_spectrum}
Consider a connected conflict component $\mathcal G^c$ with $|\mathcal V^c|\geq2$, reciprocal sensing as specified in Assumption~\ref{ass:reciprocal_sensing}, and the design in Section~\ref{sec:supp_pressure_regimes}. At a group tie during negotiation, the {within component conflict-sensitivity $\boldsymbol A_{cc}^{0}$} is diagonally symmetrizable with a real, nonnegative, semisimple spectrum. Its zero eigenvalue is simple and every eigenvector associated with a positive eigenvalue contains both positive and negative components.

In a nondegenerate parameter family of reciprocal interactions, the dominant eigenvalue $\lambda_{\max}>0$ is generically simple, with eigenvector $\boldsymbol\phi_{\max}$. At a first crossing of $\lambda^\ast$ from below by a simple dominant eigenvalue, only this eigenmode of {$\boldsymbol A_{cc}^{0}$} can become critical near the crossing.
\end{proposition}

\begin{proof}
For each physical-conflict pair $\{i,j\}\in\mathcal P^c$, let $e=(i,j)$ denote the ordered interaction and $\bar e=(j,i)$ its reciprocal. At a group tie, their common projected entry time is $t^{e,\ast}:=d^e_{i^-}/v_i^{\mathrm{nom}}=d^e_{j^-}/v_j^{\mathrm{nom}}$.
Let $\boldsymbol B\in\mathbb R^{|\mathcal P^c|\times|\mathcal V^c|}$ be the incidence matrix, with one row per physical conflict and one column per agent. For each pair $\{i,j\}\in\mathcal P^c$, choose either ordering $(i,j)$ and place $+1$ in column $i$, $-1$ in column $j$, and zeros elsewhere in its row. This arbitrary sign convention does not specify which agent passes first. Define
\begin{equation}
\label{eq:supp_spectral_factor_definitions}
\begin{gathered}
s_i^\ast:=\sum_{e'\in\mathcal E_i}\exp(p^{e',\ast}/T_{\mathrm{sm}}),
\qquad
c^{e,\ast}:=k_{\mathrm g}t^{e,\ast}p^{e,\ast}\exp(p^{e,\ast}/T_{\mathrm{sm}}),\\
\boldsymbol D_s:=\operatorname{diag}(s_i^\ast),\qquad
\boldsymbol D_c:=\operatorname{diag}(c^{e,\ast}),\qquad
\boldsymbol D_\nu:=\operatorname{diag}(\nu_i),\qquad
\nu_i:=\frac{V_i'(0)}{v_i^{\mathrm{nom}}}.
\end{gathered}
\end{equation}
Here $s_i^\ast$ is agent $i$'s softmax normalization, and $c^{e,\ast}$ is the shared timing and pressure coefficient of conflict $e$.
Agent $i$'s fractional speed change is $(v_i-v_i^{\mathrm{nom}})/v_i^{\mathrm{nom}}$. Linearizing the steady-state speed relation in Equation~\eqref{eq:supp_fast_equilibrium_states} at collective indecision gives $\delta v_i/v_i^{\mathrm{nom}}=\nu_i\delta z_i$.

Lemma~\ref{prop:uniform_pressure_bands} gives $p^{e,\ast}>0$, and the softmax in Equation~\eqref{eq:pressure_softmax} gives $s_i^\ast>0$ and $w_i^{e,\ast}>0$. Thus every conflict has positive interaction weight $p^{e,\ast}w_i^{e,\ast}$. The sensing condition in Equation~\eqref{eq:supp_negotiation_sensing_budget} keeps each agent outside its contested segment throughout negotiation, so $t^{e,\ast}>0$. Together with $k_{\mathrm g}>0$, these facts give $c^{e,\ast}>0$ in Equation~\eqref{eq:supp_spectral_factor_definitions}. The speed-command law in Equation~\eqref{eq:tanh_speed_map} also gives $\nu_i=\kappa_V\kappa_z>0$.

Under the time-scale separation established in Section~\ref{sec:supp_pressure_regimes}, we treat the remaining traversal distances as fixed during negotiation, while the projected entry times $t^e_{k^-}=d^e_{k^-}/v_k$ retain their dependence on realized speeds. Linearizing the arrival-order map in Equation~\eqref{eq:arrival_order_function} at nominal speeds gives
\begin{equation*}
\delta G_i^e=k_{\mathrm g}t^{e,\ast}
\left(\frac{\delta v_i}{v_i^{\mathrm{nom}}}-\frac{\delta v_j}{v_j^{\mathrm{nom}}}\right).
\end{equation*}
Since $G_i^{e,\ast}=0$, the first-order contribution of conflict $e$ to agent $i$'s evidence is $p^{e,\ast}w_i^{e,\ast}\delta G_i^e$; pressure and weight derivatives do not enter this expression. At the tie, $w_i^{e,\ast}=\exp(p^{e,\ast}/T_{\mathrm{sm}})/s_i^\ast$. Substituting the tracking-equilibrium relation $\delta v_i=V_i'(0)\delta z_i$ and stacking the reciprocal interaction contributions gives
\begin{equation}
\boldsymbol A_{cc}^{0}=\boldsymbol D_s^{-1}\boldsymbol L_{\mathrm g}\boldsymbol D_\nu,
\qquad
\boldsymbol L_{\mathrm g}:=\boldsymbol B^{\mathsf T}\boldsymbol D_c\boldsymbol B.
\label{eq:symmetrizable_A_factorization}
\end{equation}
The matrix $\boldsymbol  L_{\mathrm g}$ is a symmetric positive-semidefinite weighted graph Laplacian. With $\boldsymbol H_c:=\boldsymbol D_s\boldsymbol D_\nu$,
\begin{equation*}
\boldsymbol H_c^{1/2}\boldsymbol A_{cc}^{0}A\boldsymbol H_c^{-1/2}
=\boldsymbol H_c^{-1/2}\boldsymbol D_\nu\boldsymbol L_{\mathrm g}
\boldsymbol D_\nu\boldsymbol H_c^{-1/2}
\end{equation*}
is symmetric positive semidefinite. This similarity therefore makes $\boldsymbol A_{cc}^{0}$ diagonally symmetrizable with a real, nonnegative, semisimple spectrum. Connectedness gives $\ker\boldsymbol A_{cc}^{0}=\operatorname{span}\{\boldsymbol D_\nu^{-1}\boldsymbol 1\}$. Since $|\mathcal V^c|\geq2$ and zero is simple, all remaining eigenvalues are strictly positive; hence $\lambda_{\max}>0$. The zero eigenmode of $\boldsymbol A_{cc}^{0}$ gives equal fractional-speed perturbations for all agents and preserves pairwise entry-time ties to first order. The positive left zero vector $\boldsymbol\psi_0=\boldsymbol D_s\boldsymbol 1$ satisfies $\boldsymbol\psi_0^{\mathsf T}\boldsymbol\phi=0$ for every positive-eigenvalue eigenvector $\boldsymbol\phi$, which therefore must contain {components with positive and negative} signs.

The matrix $\boldsymbol A_{cc}^{0}$ is similar to a real symmetric matrix. Repeated eigenvalues separate under generic perturbations available in a nondegenerate structured parameter family~\cite{kato1966perturbation}. For a continuous family of such group-tie reference matrices $\boldsymbol A_{cc}^{0}(t)$ on the fixed interaction graph, let $t_{\mathrm c}$ denote the first threshold-crossing time from below, $\lambda_{\max}(\boldsymbol A_{cc}^{0} (t_{\mathrm c}))=\lambda^\ast$, at which the dominant eigenvalue is simple. The next-largest eigenvalue $\lambda_{N-1}$, where $N:=|\mathcal V^c|$, then satisfies $\lambda_{N-1}(\boldsymbol A(t_{\mathrm c}))<\lambda^\ast$. By continuity, there are an open time interval $I_{\mathrm c}$ containing $t_{\mathrm c}$ and a constant $\delta_{\mathrm{sp}}>0$ such that
\begin{equation*}
\lambda_{\max}(\boldsymbol A_{cc}^{0}(t))-\lambda_{N-1}(\boldsymbol A_{cc}^{0}(t))\geq\delta_{\mathrm{sp}},
\qquad
\lambda_{N-1}(\boldsymbol A_{cc}^{0}(t))\leq\lambda^\ast-\delta_{\mathrm{sp}},
\qquad t\in I_{\mathrm c}.
\end{equation*}
Only the dominant eigenmode can therefore become critical on $I_{\mathrm c}$. In the absence of nonconflicting evidence, Lemma~\ref{lem:jacobian_steady_state} gives the corresponding eigenvector in the full decision--attention--speed subsystem. If exact structural symmetry preserves a repeated dominant eigenvalue, its eigenspace remains semisimple and requires a multidimensional critical-space analysis.
\end{proof}

\paragraph{Remark (Boundedness of weak nonconflicting interactions).}
\label{rem:weak_nonconflicting}
Let $\mathcal{G}^1,\ldots,\mathcal{G}^m$ denote the physical-conflict components, including singletons. For each component $c$, let $\mathcal{U}_c$ be a neighborhood of its nominal group-tie configuration, including the states of sensed agents outside the component. Because the pressure gains and softmax tilt separate weak from safety-critical regimes (Lemmas~\ref{prop:uniform_pressure_bands} and~\ref{lem:tilt_upper_bound}) and local congestion is uniformly bounded by $M_{\max}$ (Assumption~\ref{ass:uniform_spectral_observations}), sensed nonconflicting interactions carry pressure at most $p_w \ll p_c$ and receive exponentially suppressed softmax weights whenever a safety-critical conflict is present. Consequently, the aggregate nonconflicting evidence $h_i^{\mathrm{nc}}$ (computed using the softmax normalization over all sensed interactions) satisfies $\sup_{\boldsymbol{x}\in\mathcal{U}_c} |h_i^{\mathrm{nc}}(\boldsymbol{x})| \leq \varepsilon_h$ and $\sup_{\boldsymbol{x}\in\mathcal{U}_c} \|\frac{\partial h_i^{\mathrm{nc}}}{\partial\boldsymbol{z}}(\boldsymbol{x})\|_2 \leq \varepsilon_h'$ for all $i\in\mathcal{V}^c$, where $\boldsymbol{z}$ collects the decision states of component $c$ and its sensed neighbors. 

Correspondingly, let $\boldsymbol{A} = \boldsymbol{A}_0 + \boldsymbol{E}$ be the interaction-sensitivity matrix of the full sensed-interaction system, where $\boldsymbol{A}_0 = \operatorname{blkdiag}(\boldsymbol{A}_{11}^{0},\ldots,\boldsymbol{A}_{mm}^{0})$ collects the reciprocal physical-conflict contributions $\boldsymbol{A}_{cc}^{0}$ to each component $c$ (evaluated under the actual softmax normalization) and $\boldsymbol{E}$ collects the remaining sensitivity contributions, including cross-component coupling. With $\boldsymbol{H} = \operatorname{blkdiag}(\boldsymbol{H}_1,\ldots,\boldsymbol{H}_m)$, where $\boldsymbol{H}_c = \boldsymbol{D}_s\boldsymbol{D}_\nu$ is the positive diagonal symmetrizer from Proposition~\ref{prop:interaction_sensitivity_spectrum} for nonsingleton components and $\boldsymbol{H}_c = 1$ for singletons, the weak-pressure design ensures $\|\boldsymbol{H}^{1/2}\boldsymbol{E}\boldsymbol{H}^{-1/2}\|_2 \leq \varepsilon_A$. The constants $\varepsilon_h, \varepsilon_h', \varepsilon_A > 0$ are uniform over the admissible operating domain and can be made arbitrarily small by tuning $p_w/p_c$ and $T_{\mathrm{sm}}$.

\begin{corollary}[Spectral robustness under weak nonconflicting interactions]
\label{cor:weak_spectral_robustness}
Under the weak-interaction bounds of the preceding Remark, every eigenvalue of $\boldsymbol{A}$ lies within $\varepsilon_A$ of the spectrum of $\boldsymbol{A}_0$. Furthermore, for each physical-conflict component $c$, suppose $\boldsymbol{A}_{cc}^{0}$ has a simple dominant eigenvalue $\lambda_{c,1}^{0}$ separated from its remaining eigenvalues by $\delta_{\mathrm{sp},c} := \lambda_{c,1}^{0} - \lambda_{c,2}^{0} > 0$. If $\varepsilon_A < \delta_{\mathrm{sp},c}/2$, then the perturbed block $\boldsymbol{A}_{cc} := \boldsymbol{A}_{cc}^{0} + \boldsymbol{E}_{cc}$ has a simple real dominant eigenvalue $\widehat{\lambda}_{c,1}$ satisfying $|\widehat{\lambda}_{c,1} - \lambda_{c,1}^{0}| \leq \varepsilon_A$, and every other eigenvalue $\lambda$ of $\boldsymbol{A}_{cc}$ satisfies $\operatorname{Re}\lambda \leq \lambda_{c,2}^{0} + \varepsilon_A$.
\end{corollary}

\begin{proof}
By Proposition~\ref{prop:interaction_sensitivity_spectrum}, each nonsingleton block $\boldsymbol{A}_{cc}^{0}$ is diagonally symmetrizable, so $\boldsymbol{B}_0 := \boldsymbol{H}^{1/2}\boldsymbol{A}_0\boldsymbol{H}^{-1/2}$ is symmetric. Let $\boldsymbol{F} := \boldsymbol{H}^{1/2}\boldsymbol{E}\boldsymbol{H}^{-1/2}$, which satisfies $\|\boldsymbol{F}\|_2 \leq \varepsilon_A$. For any eigenvalue $z$ of $\boldsymbol{B}_0 + \boldsymbol{F}$ with unit eigenvector $\boldsymbol{v}$, we have $(z\boldsymbol{I} - \boldsymbol{B}_0)\boldsymbol{v} = \boldsymbol{F}\boldsymbol{v}$. Symmetry of $\boldsymbol{B}_0$ yields $\min_{\lambda_0\in\sigma(\boldsymbol{B}_0)} |z - \lambda_0| \leq \|(z\boldsymbol{I} - \boldsymbol{B}_0)\boldsymbol{v}\|_2 = \|\boldsymbol{F}\boldsymbol{v}\|_2 \leq \varepsilon_A$. Since $\boldsymbol{A}$ and $\boldsymbol{A}_0$ are similar to $\boldsymbol{B}_0 + \boldsymbol{F}$ and $\boldsymbol{B}_0$, respectively, every eigenvalue of $\boldsymbol{A}$ lies within $\varepsilon_A$ of an eigenvalue of $\boldsymbol{A}_0$.

For an individual component $c$, the corresponding principal submatrix $\boldsymbol{F}_{cc}$ has norm at most $\varepsilon_A$. Applying the same perturbation bound to $\boldsymbol{A}_{cc}^{0}$ and $\boldsymbol{E}_{cc}$ shows that when $\delta_{\mathrm{sp},c} > 2\varepsilon_A$, the $\varepsilon_A$-disk around $\lambda_{c,1}^{0}$ is disjoint from those around the remaining eigenvalues. As the perturbation $s\boldsymbol{E}_{cc}$ varies continuously from $s=0$ to $s=1$, exactly one eigenvalue (counting algebraic multiplicity) remains inside this isolated disk. Because $\boldsymbol{A}_{cc}$ is real and nonreal eigenvalues occur in conjugate pairs, this isolated eigenvalue $\widehat{\lambda}_{c,1}$ must be real and simple, satisfying $|\widehat{\lambda}_{c,1} - \lambda_{c,1}^{0}| \leq \varepsilon_A$, while all other eigenvalues satisfy $\operatorname{Re}\lambda \leq \lambda_{c,2}^{0} + \varepsilon_A$.
\end{proof}

In the absence of nonconflicting evidence, Lemma~\ref{lem:jacobian_steady_state} and Proposition~\ref{prop:interaction_sensitivity_spectrum} locate the bifurcation point of an isolated physical-conflict component at $\lambda_{\max}(\boldsymbol{A}_{cc}^{0}) = \lambda^\ast$, where positive motion feedback balances decision damping, $S'(0)(u_0/\rho_u)\lambda^\ast = \rho_z$, and the dominant decision mode becomes critical. Crossing this threshold generates differentiated equilibrium branches along the critical decision eigenvector $\boldsymbol{\phi}_{\max}$ and speed pattern $\boldsymbol{K}\boldsymbol{\phi}_{\max}$ (Section~\ref{sec:supp_local_bifurcation}); the speed-tracking dynamics and subsystem time constants govern transient growth and decay rates without shifting this steady-state threshold. Under weak nonconflicting interactions, Corollary~\ref{cor:weak_spectral_robustness} guarantees that the interaction-sensitivity spectrum is perturbed by at most $\varepsilon_A$ and preserves each component's isolated simple dominant mode whenever $\varepsilon_A < \delta_{\mathrm{sp},c}/2$. Although weak nonconflicting evidence slightly unfolds the collective-indecision equilibrium, sufficiently small $\varepsilon_h, \varepsilon_h', \varepsilon_A$ preserve the spectral separation between weak and safety-critical regimes.


\subsubsection{Local tuning makes a single safety-critical conflict destabilize collective indecision}

An agent's basal-attention input $u_0$ lets us place the bifurcation point at a desired interaction strength.
At fixed $\rho_u$, increasing $u_0$ lowers $\lambda^\ast$, shifting the bifurcation toward weaker interaction feedback.
We now choose fixed local NDD parameters so that pressure-dependent feedback keeps $\lambda_{\max}(\boldsymbol A)<\lambda^\ast$ when all sensed interactions are weak and autonomously drives it above $\lambda^\ast$ when even one interaction becomes safety-critical.


\begin{assumption}[Bounded local congestion]
\label{ass:uniform_spectral_observations}
Each agent has at most $M_{\max}$ sensed interactions. 
\end{assumption}
The bound $M_{\max}$ is chosen conservatively for the number of simultaneous interactions an agent may sense.

\begin{lemma}[Uniform spectral bands for weak and safety-critical interactions]
\label{lem:uniform_spectral_bands}
Under the hypotheses of Proposition~\ref{prop:interaction_sensitivity_spectrum} and Assumption~\ref{ass:uniform_spectral_observations}, let $\bar t$ be a uniform upper bound on the magnitude of the projected time until each agent reaches its contested segment and $\underline t_{\mathrm c}>0$ a uniform lower bound on that time for safety-critical interactions. Let $\underline\nu>0$ and $\bar\nu$ be uniform lower and upper bounds on the local fractional-speed gains $\nu_i$. With the weak-pressure level $p_{\mathrm w}$, critical-pressure level $p_{\mathrm c}$, and pressure upper bound $\bar p$ from Lemma~\ref{prop:uniform_pressure_bands},
\begin{equation}
\lambda_{\max}(\boldsymbol A)\leq
\overline\Lambda_{\mathrm w}:=2k_{\mathrm g}\bar t\bar\nu p_{\mathrm w}
\label{eq:supp_weak_spectral_bound}
\end{equation}
when all sensed interactions are weak, and
\begin{equation*}
\lambda_{\max}(\boldsymbol A)\geq
\underline\Lambda_{\mathrm c}:=
\frac{k_{\mathrm g}\underline t_{\mathrm c}\underline\nu p_{\mathrm c}}
{1+(M_{\max}-1)\exp\!\left((\bar p-p_{\mathrm c})/T_{\mathrm{sm}}\right)}
\end{equation*}
when at least one interaction is safety-critical. The pressure gains and softmax temperature $T_{\mathrm{sm}}$ can be chosen so that $\overline\Lambda_{\mathrm w}<\underline\Lambda_{\mathrm c}$.
\end{lemma}

\begin{proof}
At a group tie, the softmax scores equal the pressures. For a safety-critical interaction $e=(i,j)$, Lemma~\ref{prop:uniform_pressure_bands} bounds its score below by $p_{\mathrm c}$ and every competing score above by $\bar p$. With at most $M_{\max}$ interactions per agent, $w_i^{e,\ast},w_j^{\bar e,\ast}\geq\underline w_{\mathrm c}$, where
\begin{equation}
\underline w_{\mathrm c}:=
\left[1+(M_{\max}-1)\exp\!\left(\frac{\bar p-p_{\mathrm c}}{T_{\mathrm{sm}}}\right)\right]^{-1}.
\label{eq:supp_critical_weight_bound}
\end{equation}
Equation~\eqref{eq:symmetrizable_A_factorization} gives
\begin{equation*}
(\boldsymbol A\boldsymbol\xi)_i
=\sum_{e=(i,j)\in\mathcal E_i}
 k_{\mathrm g}t^{e,\ast}p^{e,\ast}w_i^{e,\ast}(\nu_i\xi_i-\nu_j\xi_j).
\end{equation*}
For weak interactions, $p^{e,\ast}\leq p_{\mathrm w}$ and $\sum_e w_i^{e,\ast}=1$, hence
\begin{equation*}
\sum_j|A_{ij}|
\leq\sum_{e=(i,j)\in\mathcal E_i}
k_{\mathrm g}|t^{e,\ast}|p^{e,\ast}w_i^{e,\ast}(\nu_i+\nu_j)
\leq 2k_{\mathrm g}\bar t\bar\nu p_{\mathrm w}.
\end{equation*}
Because the normalized weights sum to one, adding more weak interactions does not increase this row-sum bound. Bounding the largest eigenvalue by the maximum absolute row sum proves Equation~\eqref{eq:supp_weak_spectral_bound}.

For a safety-critical interaction $e=(i,j)$, Equation~\eqref{eq:supp_critical_weight_bound} gives $w_i^{e,\ast},w_j^{\bar e,\ast}\geq\underline w_{\mathrm c}$. By Proposition~\ref{prop:interaction_sensitivity_spectrum}, the symmetric matrix similar to $\boldsymbol  A$ has the same diagonal entries. Its $i$th diagonal contains the nonnegative contribution $k_{\mathrm g}t^{e,\ast}p^{e,\ast}w_i^{e,\ast}\nu_i$.
The Rayleigh--Ritz characterization therefore gives
$
\lambda_{\max}(\boldsymbol  A)
\geq
k_{\mathrm g}\underline t_{\mathrm c}\underline\nu
p_{\mathrm c}\underline w_{\mathrm c}.
$
Lemma~\ref{prop:uniform_pressure_bands} permits $p_{\mathrm w}/p_{\mathrm c}$ to be made sufficiently small, while $T_{\mathrm{sm}}$ controls $\underline w_{\mathrm c}$.
Choosing $k_{\mathrm d},k_{\mathrm t}$ and $T_{\mathrm{sm}}$ jointly makes the eigenvalue upper bound for a component with only weak sensed interactions smaller than the lower bound supplied by a single safety-critical conflict.
\end{proof}

{
By Corollary~\ref{cor:weak_spectral_robustness}, the preceding spectral bounds extend to the perturbed within-component sensitivity $\boldsymbol{A}_{cc}$ whenever $\varepsilon_A < \min\{\lambda^\ast - \overline{\Lambda}_{\mathrm{w}}, \underline{\Lambda}_{\mathrm{c}} - \lambda^\ast\}$. Consequently, all eigenvalues of $\boldsymbol{A}_{cc}$ remain below $\lambda^\ast$ in the weak-interaction regime, whereas a safety-critical interaction drives the dominant eigenvalue above $\lambda^\ast$. Under the spectral-gap condition $\varepsilon_A < \delta_{\mathrm{sp},c}/2$ of Corollary~\ref{cor:weak_spectral_robustness}, this dominant eigenvalue remains simple and real.
}

Choose $\lambda^\ast\in(\overline\Lambda_{\mathrm w},\underline\Lambda_{\mathrm c})$ and set the local basal-attention input to $u_0=\rho_z\rho_u/[S'(0)\lambda^\ast]$. The lemma then gives $\lambda_{\max}(\boldsymbol A)<\lambda^\ast$ when all interactions are weak and $\lambda_{\max}(\boldsymbol A)>\lambda^\ast$ when at least one is safety-critical.
Pressure and softmax weighting set the eigenvalue bounds of the interaction-sensitivity matrix $\boldsymbol A$, while the basal-attention input $u_0$ fixes the bifurcation threshold $\lambda^\ast$ between them. As relative motion changes the projected conflict severity, locally updated pressures and softmax weights regulate the strength of NDD feedback across the component. When a group with only weak interactions develops a safety-critical conflict, this pressure regulation drives $\lambda_{\max}(\boldsymbol A)$ across the fixed threshold, autonomously triggering collective differentiation without agents computing eigenvalues or knowing the component size.
%


\subsubsection{Collective differentiation establishes an acyclic traversal order}

Once collective indecision loses stability, the nonlinear decision--attention--speed dynamics determine the differentiated equilibrium toward which the agents move. Under the nondegeneracy and local stability conditions in Section~\ref{sec:supp_local_bifurcation}, an attracting differentiated branch emerges just above the bifurcation threshold. At zero projected entry-time bias, its decision state satisfies $\boldsymbol z^\dagger(a)=a\boldsymbol\phi_{\max}+\mathcal O(a^2)$. Thus the dominant eigenvector $\boldsymbol\phi_{\max}$ of the interaction-sensitivity matrix $\boldsymbol A$ determines the leading collective decision pattern, while the scalar $a$ sets its amplitude and sign. Projected entry-time bias and inherited decisions influence which equilibrium is selected.

The speed-command law in Equation~\eqref{eq:tanh_speed_map} converts this decision pattern into equilibrium fractional speeds
\begin{equation*}
\vartheta_i^\dagger(a):=\frac{v_i^\dagger(a)}{v_i^{\mathrm{nom}}}
=1+a\nu_i\phi_{\max,i}+\mathcal O(a^2).
\end{equation*}
Each fractional speed uses that agent's own nominal speed; its fractional speed change from nominal is $\vartheta_i^\dagger-1$.
For a physical conflict $\{i,j\}\in\mathcal P^c$ with directed interaction $e=(i,j)$, let $t^{e,\ast}$ be the shared projected entry time at the reference group tie. Under time-scale separation, the projected entry times at the differentiated equilibrium are $t_{i^-}^{e,\dagger}=t^{e,\ast}/\vartheta_i^\dagger(a)$ and $t_{j^-}^{e,\dagger}=t^{e,\ast}/\vartheta_j^\dagger(a)$.
The agent with the larger fractional speed is therefore projected to enter its contested segment first.

A strict traversal order at every conflict requires different fractional speeds for each conflicting pair. Mixed signs of $\boldsymbol\phi_{\max}$ guarantee that some agents speed up and others slow down, but do not by themselves exclude a tied pair. The required separation is expressed by a positive minimum gap:
\begin{equation}
\delta_\vartheta(a)
:=\min_{\{i,j\}\in\mathcal P^c}
\left|\vartheta_i^\dagger(a)-\vartheta_j^\dagger(a)\right|>0.
\label{eq:edge_fractional_speed_margin}
\end{equation}
When $\delta_\vartheta(a)>0$, we select $i\prec_e j$ as the traversal order for a conflict $e=(i,j)$ whenever $\vartheta_i^\dagger(a)>\vartheta_j^\dagger(a)$.
We represent these selected orders by a traversal-order graph with the agents in $\mathcal V^c$ as vertices and a directed edge $i\to j$ for each selected order $i\prec_e j$. A source has no incoming edge, so no conflicting agent is selected to precede it.
Fractional speeds decrease strictly along its edges, excluding a directed cycle such as $i\to j\to k\to i$.
The following proposition establishes $\delta_\vartheta(a)>0$ for generic measurements and characterizes the resulting traversal order.

\begin{proposition}[A traversal order emerges across an engaged group]
\label{prop:strict_acyclic_traversal_order}
Under the time-scale separation established in Section~\ref{sec:supp_pressure_regimes}, let $\boldsymbol z^\dagger(a)=a\boldsymbol\phi_{\max}+\mathcal O(a^2)$ be a local attracting branch emerging from a group tie on a fixed component. For generic reciprocal-interaction measurements and small nonzero $a$, $\delta_\vartheta(a)>0$. The traversal-order graph induced by the selected orders is acyclic. An agent of maximal fractional speed is a source and has equilibrium speed at least its nominal speed. Every nonsource can be reached from a source along directed traversal-order edges, and no two sources share a physical conflict.
\end{proposition}

\begin{proof}
Set $\boldsymbol\eta:=\boldsymbol D_\nu\boldsymbol\phi_{\max}$, so $\eta_i=\nu_i\phi_{\max,i}$ is agent $i$'s first-order fractional-speed coefficient. The factorization in Equation~\eqref{eq:symmetrizable_A_factorization} gives the node equation
\begin{equation*}
\frac{\nu_i}{s_i^\ast}\sum_{k\in\mathcal N_i^c}c^{ik,\ast}(\eta_i-\eta_k)
=\lambda_{\max}\eta_i.
\end{equation*}
If conflicting agents have $\eta_i=\eta_j=r$, their common edge contributes zero to both node equations. Their interactions with the rest of the component must then satisfy the balance
\begin{equation*}
\frac{\nu_i}{s_i^\ast}\sum_{k\in\mathcal N_i^c\setminus\{j\}}c^{ik,\ast}(r-\eta_k)
=\frac{\nu_j}{s_j^\ast}\sum_{k\in\mathcal N_j^c\setminus\{i\}}c^{jk,\ast}(r-\eta_k).
\end{equation*}
This equality constrains the two agents' interaction neighborhoods, projected entry times, pressures, normalizations, and speed gains. Unless structural symmetry enforces the balance, its zero set has empty interior. Since the component has finitely many physical conflicts, generic measurements therefore give $\delta_\eta:=\min_{\{i,j\}\in\mathcal P^c}|\eta_i-\eta_j|>0$.
The equilibrium speed expansion above gives
\begin{equation}
\vartheta_i^\dagger(a)-\vartheta_j^\dagger(a)
=a(\eta_i-\eta_j)+\mathcal O(a^2),
\qquad \{i,j\}\in\mathcal P^c.
\label{eq:branch_edge_fractional_speed}
\end{equation}
For small nonzero $a$, the nonzero first-order difference on each edge dominates its higher-order correction, giving $\delta_\vartheta(a)>0$. Along each traversal-order edge, the equilibrium fractional speed decreases strictly. A directed cycle would require it to decrease and return to its starting value, which is impossible.
By Proposition~\ref{prop:interaction_sensitivity_spectrum}, $\boldsymbol\phi_{\max}$ has both positive and negative entries. Positivity of the $\nu_i$ gives the same property for $\boldsymbol\eta$, so for either sign of small nonzero $a$, some agent has a fractional speed greater than one. A maximizer has no incoming edge and is a source. In a finite acyclic graph, tracing incoming edges from a nonsource ends at a source. Two sources cannot share an edge, since one endpoint would then have an incoming edge.
\end{proof}


\subsubsection{Persistence of the traversal order as agents move}

After the order forms, motion-dependent pressure feedback continues to regulate differentiation. As projected conflict severity decreases, lower pressure can permit speeds to relax toward nominal values; if this relaxation increases projected conflict severity, rising pressure reinforces decision differentiation. On intervals with a fixed conflict graph, smooth equilibrium continuation and the time-scale separation in Section~\ref{sec:supp_pressure_regimes} allow the decision--attention--speed subsystem to track the evolving differentiated equilibrium while preserving the strict traversal order locally.
Write $\boldsymbol x=(\boldsymbol z,\boldsymbol u,\boldsymbol v)$ for the decision--attention--speed state. Let $\boldsymbol\theta$ collect the group-tie projected times and pressures, fractional-speed gains, and controller parameters, and let $a^\dagger(\boldsymbol\theta)$ be the selected equilibrium amplitude.

\begin{lemma}[Smooth continuation preserves the selected traversal order]
\label{prop:smooth_dominant_eigenvector_branch}
Fix the conflict graph. In the nondegenerate local regime of Section~\ref{sec:supp_local_bifurcation}, suppose the interaction-sensitivity matrix $\boldsymbol A$ has a simple dominant eigenvalue and the selected differentiated equilibrium $\boldsymbol x^\dagger(\boldsymbol\theta_0)$ is hyperbolic and attracting. Then the dominant eigenpair varies smoothly near $\boldsymbol\theta_0$, and the selected equilibrium has a unique smooth continuation as attracting equilibria $\boldsymbol x^\dagger(\boldsymbol\theta)$. If $\delta_\vartheta(a^\dagger(\boldsymbol\theta_0))>0$, this continuation preserves the pairwise traversal order in a sufficiently small neighborhood of $\boldsymbol\theta_0$.
\end{lemma}

\begin{proof}
Standard perturbation theory gives smooth continuation of the simple dominant eigenpair~\cite{kato1966perturbation}. When the diagonal symmetrization remains well conditioned, the sensitivity bound for $\boldsymbol\phi_{\max}$ to changes in the interaction-sensitivity matrix $\boldsymbol A$ scales inversely with the spectral gap. Hyperbolicity permits the implicit function theorem to continue the equilibrium, with attraction persisting locally. Continuity then preserves the sign of each nonzero pairwise fractional-speed difference and hence the traversal order.
\end{proof}

Let $\boldsymbol\zeta(t)$ collect the slowly evolving conflict geometry and conflict variables as agents move, and let $\boldsymbol x^\dagger(\boldsymbol\zeta(t))$ be the corresponding differentiated equilibrium of the decision--attention--speed subsystem. On a compact portion of this equilibrium branch with uniform exponential attraction and bounded parameter sensitivities, standard time-scale separation results bound the error between the subsystem state and this evolving equilibrium after the state enters the branch's attraction neighborhood and the initial transient decays (see~\cite{khalil2008nonlinear}, chap.~11):
\begin{equation}
\label{eq:supp_branch_tracking_error}
\|\boldsymbol x(t)-\boldsymbol x^\dagger(\boldsymbol\zeta(t))\|
=\mathcal O(\varepsilon),
\qquad
\varepsilon:=\max_e\frac{T_{\mathrm{fast}}}{T_{\mathrm{trav}}^e}\ll1.
\end{equation}
Here $T_{\mathrm{fast}}$ is the time scale of the coupled decision--attention--speed subsystem, and $T_{\mathrm{trav}}^e$ is a lower bound on the time from detection until either agent reaches its contested segment; their separation follows from the response and sensing design in Section~\ref{sec:supp_pressure_regimes}.
Thus, if the subsystem state remains sufficiently close to the evolving differentiated equilibrium and the equilibrium fractional-speed differences remain positive in the selected order, the agents preserve that order as their speeds change.

{Generic symmetry breaking ensures departure from collective indecision but does not, by itself, provide a uniform negotiation-time bound. We therefore assume that, over the domain, the unstable decision mode receives uniformly nonvanishing excitation and the coupled decision--attention--speed dynamics reach the selected attracting branch's neighborhood within a common finite time $T^{\mathrm{neg}}$.}The pressure and spectral conditions, branch-selection analysis, and equilibrium-tracking estimate give the following result for a fixed group.

\begin{theorem}[Autonomous traversal-order emergence in a fixed group]
\label{thm:autonomous_spectral_tuning}
Suppose the pressure and response-time conditions of Section~\ref{sec:supp_pressure_regimes}, the separated spectral bands of Lemma~\ref{lem:uniform_spectral_bands}, and the generic attracting-branch conditions of Section~\ref{sec:supp_local_bifurcation} hold. Fixed local controller parameters can then be chosen independently of population and realized group sizes so that each locally engaged group has the following response while its sensed-interaction sets remain fixed:
\begin{enumerate}
\item Weak interactions leave collective indecision stable. A safety-critical interaction makes collective indecision unstable before either conflicting agent enters its contested segment.
\item A generic projected entry-time bias or inherited amplitude of the critical decision mode selects an attracting differentiated branch of Equation~\eqref{eq:main_differentiated_branch}. Under the prescribed response-time conditions, the decision--attention--speed state reaches this branch's attraction neighborhood within $T^{\mathrm{neg}}$.
\item As the conflict geometry varies slowly, the decision--attention--speed state tracks the corresponding differentiated equilibrium, with the error given by Equation~\eqref{eq:supp_branch_tracking_error}. Under the strict fractional-speed gap condition~\eqref{eq:edge_fractional_speed_margin}, the selected traversal-order graph remains locally unchanged and acyclic, and has a source with positive equilibrium speed.
\end{enumerate}
After the group's all physical conflict in the current episode clears, $\boldsymbol z\to\boldsymbol0$, $\boldsymbol u\to(u_0/\rho_u)\boldsymbol1$, and $\boldsymbol v\to\boldsymbol v^{\mathrm{nom}}$ as residual conflict evidence vanishes, provided no new conflicts involve these agents. Each agent uses local observations and fixed parameters, without online eigenvalue or group-size estimates.
\end{theorem}

\begin{proof}
For conclusion 1, choose $\lambda^\ast\in(\overline\Lambda_{\mathrm w},\underline\Lambda_{\mathrm c})$. Lemma~\ref{prop:uniform_pressure_bands} separates the pressure regimes, and Lemma~\ref{lem:uniform_spectral_bands} gives $\lambda_{\max}\leq\overline\Lambda_{\mathrm w}<\lambda^\ast$ for weak interactions and $\lambda_{\max}\geq\underline\Lambda_{\mathrm c}>\lambda^\ast$ with a critical interaction.
Setting $u_0=\rho_z\rho_u/[S'(0)\lambda^\ast]$ realizes the threshold by Equation~\eqref{eq:formation_spectral_condition}. The sensing and response-time conditions of Section~\ref{sec:supp_pressure_regimes} ensure that this loss of stability occurs before any agent enters its contested segment. On a fixed sensed-interaction interval, pressures and softmax weights vary continuously, {so the dominant eigenvalue of the within component conflict sensitivity crosses $\lambda^\ast$ between these regimes; weak nonconflicting evidence may perturb the corresponding stability transition.}

For conclusion 2, Proposition~\ref{prop:interaction_sensitivity_spectrum}
identifies the simple dominant eigenmode of {$\boldsymbol A_{cc}^{0}$}, and
Lemma~\ref{lem:jacobian_steady_state} gives the corresponding zero eigenmode of the full subsystem Jacobian {restricted to conflicting edges.}
The local critical equation~\eqref{eq:supp_critical_equation} establishes the exchange of stability and an attracting differentiated branch {under the stated nondegeneracy
conditions. Under weak nonconflicting interactions, the corresponding hyperbolic differentiated equilibria persist locally.}
{A projected entry-time bias or inherited decision amplitude excites the unstable mode.
	The assumed uniform symmetry-breaking condition, together with the positive instability margin and uniform branch-attraction and speed-tracking bounds, ensures that the decision--attention--speed state reaches the selected branch's attraction neighborhood
	within $T^{\mathrm{neg}}$. Under the assumed uniform symmetry-breaking condition,} a projected entry-time bias or a nonzero inherited amplitude excites the unstable critical decision mode. The critical amplitude initially grows exponentially after which the decision--attention--speed state approaches
	the hyperbolic differentiated branch exponentially. 	Under the prescribed response-time conditions, the state 	reaches the branch's attraction neighborhood within
	$T^{\mathrm{neg}}$. Lemma~\ref{prop:pressure_response_tuning} supplies sustained, slowly varying critical pressure on this interval.

For conclusion 3, Proposition~\ref{prop:strict_acyclic_traversal_order} converts the strict fractional-speed gaps into an acyclic order and identifies a source with positive equilibrium speed. Lemma~\ref{prop:smooth_dominant_eigenvector_branch} continues the differentiated equilibrium under small geometry variations, and Equation~\eqref{eq:supp_branch_tracking_error} bounds the distance from the decision--attention--speed state to this evolving equilibrium. The selected traversal order is preserved when these variations and the equilibrium-tracking error are sufficiently small.

After physical conflict clearance, Lemma~\ref{prop:uniform_pressure_bands} bounds the interaction pressures in the weak and release regimes when the corresponding geometric conditions hold. Falling pressure weakens feedback and returns the local response toward collective indecision. As residual evidence tends to zero, bounded decision and attention states make the decision input vanish, and the stable decision, attention, and speed dynamics give the stated limits. Once all residual evidence from the episode vanishes, $\boldsymbol h=\boldsymbol0$ and $\tau_z\dot{\boldsymbol z}=-\rho_z\boldsymbol z$, so decisions decay exponentially while no new conflict evidence arises.

The tuning depends only on the prescribed local interaction, geometry, and agentwise bounds. Each agent evaluates its own evidence and maintains its own basal attention, so the same local parameters apply independently of the realized group size.
\end{proof}

\subsection{Safety during negotiation and traversal}
\label{sec:parameter_design}

The emerging traversal order must give each leader enough time to leave its contested segment before its follower enters its own. Proposition~\ref{prop:fixed_component_edge_safety} combines speed authority, local yielding, and detection lead into a sufficient sensing-range bound. Theorem~\ref{thm:fixed_component_safe_order} then shows how these same local responses maintain safety and produce finite-time progress. 

\subsubsection{Speed authority and early enough detection enable safe yielding}

To realize the emergent traversal order, the decision-to-speed map must make every speed in each agent's admissible hardware speed interval $[\underline v_i,\overline v_i]$ available as a command. The bifurcation analysis in Section~\ref{sec:supp_local_bifurcation} describes decisions near collective indecision, so we seek this full speed range within a local neighborhood $|z_i|\leq z_{\mathrm{loc}}$. In Equation~\eqref{eq:tanh_speed_map}, $\kappa_V$ sets the maximum fractional-speed departure from nominal, while increasing $\kappa_z$ reduces the decision magnitude needed for any attainable departure. The mapping gains and neighborhood radius $z_{\mathrm{loc}}>0$ are chosen jointly so that the local analysis applies and decision trajectories remain within that neighborhood.

\begin{lemma}[Full admissible speed range near collective indecision]
\label{lem:decision_range_and_speed_extrema}
Consider the speed-command law~\eqref{eq:tanh_speed_map}, with positive nominal speeds $v_i^{\mathrm{nom}}$, finite admissible speed intervals $[\underline v_i,\overline v_i]$, and a bounded, nonzero decision sigmoid $S$. For any prescribed decision radius $z_{\mathrm{loc}}>0$, positive gains $\kappa_V,\kappa_z$ and decision damping $\rho_z$ can be chosen so that every admissible speed can be commanded from a forward-invariant decision interval contained in $|z_i|\leq z_{\mathrm{loc}}$. Decisions initialized in this interval remain there independently of the conflict evidence, including across graph events.
\end{lemma}

\begin{proof}
Set $s_{\max}:=\sup_{x\in\mathbb R}|S(x)|>0$ and choose $\rho_z\geq s_{\max}/z_{\mathrm{loc}}$, so that $z_{\max}:=s_{\max}/\rho_z\leq z_{\mathrm{loc}}$. At $z_i=z_{\max}$, boundedness of $S$ gives $\tau_z\dot z_i\leq-\rho_z z_{\max}+s_{\max}=0$; at $z_i=-z_{\max}$, it gives $\tau_z\dot z_i\geq\rho_z z_{\max}-s_{\max}=0$. Thus $[-z_{\max},z_{\max}]$ is forward invariant independently of the evidence. Continuity of the decision states preserves this bound across graph events.

To cover the full admissible speed range within this interval, define the maximum required fractional-speed departure by
\begin{equation*}
\delta_V:=\max_i
\frac{\max\{v_i^{\mathrm{nom}}-\underline v_i,
\overline v_i-v_i^{\mathrm{nom}}\}}{v_i^{\mathrm{nom}}}.
\end{equation*}
Choose $\kappa_V>\delta_V$ and $\kappa_z\geq z_{\max}^{-1}\operatorname{arctanh}(\delta_V/\kappa_V)$. Both gains are finite, and $z_{\max}\leq z_{\mathrm{loc}}$ implies the local command-coverage condition
\begin{equation}
\kappa_z\geq\frac{1}{z_{\mathrm{loc}}}
\operatorname{arctanh}\!\left(\frac{\delta_V}{\kappa_V}\right).
\label{eq:local_decision_range_condition}
\end{equation}
Since $\tanh$ is odd, continuous, and increasing, the commands on the invariant interval span
\begin{equation*}
\left[
v_i^{\mathrm{nom}}\bigl(1-\kappa_V\tanh(\kappa_z z_{\max})\bigr),
v_i^{\mathrm{nom}}\bigl(1+\kappa_V\tanh(\kappa_z z_{\max})\bigr)
\right].
\end{equation*}
The chosen gains ensure, for every agent,
\begin{equation}
v_i^{\mathrm{nom}}\kappa_V\tanh(\kappa_z z_{\max})
\geq\max\{v_i^{\mathrm{nom}}-\underline v_i,
\overline v_i-v_i^{\mathrm{nom}}\},
\label{eq:full_speed_range_condition}
\end{equation}
so this command range contains the entire admissible interval: $[\underline v_i,\overline v_i]\subseteq\{v_i^{\mathrm{cmd}}(z_i):|z_i|\leq z_{\max}\}$.
\end{proof}

Negotiation must leave the stopping option available to either agent. For agent $k$, let $v_k^{\mathrm{stop}}$ be the largest commanded speed that invokes its certified stopping response, and let $D_k^{\mathrm{stop}}$ be the associated stopping distance. This distance is measured from the instant the stopping command becomes available and includes the distance traveled while the realized speed responds to that command, including actuation transients. Lemma~\ref{lem:decision_range_and_speed_extrema} makes the stopping range commandable whenever $\underline v_k\leq v_k^{\mathrm{stop}}$.

\begin{lemma}[Early detection keeps either traversal order feasible]
\label{lem:stopping_distance_detection_margin}
Consider a conflicting pair with bounded traversal speeds and attainable certified stopping responses of finite distance. For any finite negotiation duration $T^{\mathrm{neg}}$, detection can occur with enough distance remaining for either agent to stop before it enters its contested segment, provided its stopping command becomes available at negotiation end. Both traversal orders therefore remain physically feasible.
\end{lemma}

\begin{proof}
At negotiation end, the stopping option remains available to either agent if $d_{i^-}^e\geq D_i^{\mathrm{stop}}$ and $d_{j^-}^e\geq D_j^{\mathrm{stop}}$.
For $i\prec_e j$, the second bound permits agent $j$ to stop outside its contested segment while agent $i$ traverses its own; the first permits the reverse order. Since agent $k$ travels at most $\overline v_kT^{\mathrm{neg}}$ during negotiation, $d_{k^-}^e(t_0+T^{\mathrm{neg}})\geq d_{k^-}^e(t_0)-\overline v_kT^{\mathrm{neg}}$.
Thus it is sufficient to ensure mutual detection at $t_0$ with
\begin{equation}
d_{k^-}^e(t_0)\geq\overline v_kT^{\mathrm{neg}}+D_k^{\mathrm{stop}},
\qquad k\in\{i,j\}.
\label{eq:total_detection_distance_bound}
\end{equation}
These required distances are finite by hypothesis. At a symmetric nominal pairwise tie, let $t_{\mathrm{det}}^e$ be the time from detection until each agent would reach its contested segment at its nominal speed. Substituting $d_{k^-}^e(t_0)=v_k^{\mathrm{nom}}t_{\mathrm{det}}^e$ gives the sufficient detection lead $t_{\mathrm{det}}^e\geq t_{\mathrm{req,stop}}^e$, where
\begin{equation*}
t_{\mathrm{req,stop}}^e
:=\max_{k\in\{i,j\}}
\frac{\overline v_kT^{\mathrm{neg}}+D_k^{\mathrm{stop}}}{v_k^{\mathrm{nom}}}.
\end{equation*}
\end{proof}

To choose a sensing radius that leaves either agent enough distance to stop, we express the required distances to their contested segments in terms of their separation. At a symmetric nominal pairwise tie, the crossing geometry provides this relation and hence the sensing radius needed for a prescribed detection lead time.

\begin{lemma}[Sensing range provides the required detection lead]
\label{prop:sensing_range_detection_margin}
Consider a symmetric nominal pairwise tie with positive nominal speeds, crossing angle $\alpha^e>0$, and finite contested-segment half-length $r_{\mathrm c}^e$. For any finite $t_{\mathrm{req}}^e>0$, a sufficiently large finite common sensing radius ensures mutual detection at least $t_{\mathrm{req}}^e$ before either agent would reach its contested segment at nominal speed.
\end{lemma}

\begin{proof}
Let $t\geq0$ be the time remaining until each agent would reach its contested segment at its nominal speed. Agent $k$ is then at projected distance $x_k^e(t)=v_k^{\mathrm{nom}}t+r_{\mathrm c}^e$ from the crossing point. The law of cosines gives the inter-agent separation
\begin{equation}
R^e(t):=
\left[
\left(v_i^{\mathrm{nom}}t+r_{\mathrm c}^e\right)^2
+\left(v_j^{\mathrm{nom}}t+r_{\mathrm c}^e\right)^2
-2\left(v_i^{\mathrm{nom}}t+r_{\mathrm c}^e\right)
\left(v_j^{\mathrm{nom}}t+r_{\mathrm c}^e\right)\cos\alpha^e
\right]^{1/2}.
\label{eq:pairwise_range_at_remaining_time}
\end{equation}
For $\alpha^e>0$ and positive nominal speeds, $R^e(t)$ increases with $t\geq0$. A common sensing radius $R_{\mathrm{sens}}\geq R^e(t_{\mathrm{req}}^e)$ therefore ensures mutual detection at least $t_{\mathrm{req}}^e$ before either agent would reach its contested segment at nominal speed.

To provide both the stopping reserve and time-scale separation, choose $t_{\mathrm{req}}^e\geq t_{\mathrm{req,stop}}^e$ and $t_{\mathrm{req}}^e\geq(T^{\mathrm{neg}}/\varepsilon_{\mathrm{trav}})\max_{k\in\{i,j\}}(\overline v_k/v_k^{\mathrm{nom}})$. The first inequality gives Equation~\eqref{eq:total_detection_distance_bound}; the second gives Equation~\eqref{eq:supp_negotiation_sensing_budget}, using $d_{k^-}^e(t_0)=v_k^{\mathrm{nom}}t_{\mathrm{det}}^e$. A single finite sensing radius works across the admissible interaction family whenever the pairwise requirements $R^e(t_{\mathrm{req}}^e)$ have a finite common upper bound.
\end{proof}

	For a conflicting pair $e=(i,j)$, let
	$t_k^e:=d_{k^-}^e/v_k^{\mathrm{nom}}$ denote agent $k$'s
	nominal entry time. Let $t_{\mathrm{req}}^e>0$ denote a prescribed detection-lead time. Consider the configuration at which $\min\{t_i^e,t_j^e\}=t_{\mathrm{req}}^e$, and suppose the nominal entry-time difference satisfies $|t_i^e-t_j^e|\leq\bar\delta_e$ at this configuration.
	Relative to the nominal entry-time tie in
	Equation~\eqref{eq:pairwise_range_at_remaining_time},
	one agent is displaced along its projected path by at most
	$\max\{v_i^{\mathrm{nom}},v_j^{\mathrm{nom}}\}\bar\delta_e$.
	The triangle inequality therefore gives the sufficient
	sensing condition
	\begin{equation}
		\min\{R_i,R_j\}\geq R_{\mathrm{req}}^e
		:=R^e(t_{\mathrm{req}}^e)
		+\max\{v_i^{\mathrm{nom}},v_j^{\mathrm{nom}}\}
		\bar\delta_e.
		\label{eq:asymmetric_sensing_radius}
	\end{equation}
	This bound guarantees mutual detection by the required
	nominal entry time and permits heterogeneous sensing radii.

\subsubsection{Local pressure tilt ensures yielding to leaders despite competing cues}

During collective differentiation, an agent can simultaneously be a follower in some conflicts and a leader in others, so its local evidence combines yielding and proceed cues. The pressure tilt in Equation~\eqref{eq:tilted_pressure} of Materials and Methods raises the softmax scores of interactions with the agent's leaders and lowers those of interactions with its followers (Equation~\eqref{eq:pressure_softmax}). We choose $k_{\mathrm{tilt}}$ and $T_{\mathrm{sm}}$ to keep the tilted scores of weak interactions below $p_{\mathrm c}$ and to make an agent's net evidence negative whenever a safety-critical interaction with one of its leaders has the prescribed projected entry-time lead, despite the combined proceed cues from its own followers. Additional leaders reinforce yielding.
Under the conditions below, this dominance makes the agent yield to each such leader, enabling safe realization of the emerging traversal order. Because these leader--follower relations are updated from current local motion measurements, the controller does not reconstruct or store the full traversal order.

Select an entry-time margin $\delta_{\mathrm g}>0$ above the certified projected entry-time uncertainty and a target arrival-order magnitude $g_{\mathrm c}\in(0,1)$. Choosing $k_{\mathrm g}\geq\operatorname{arctanh}(g_{\mathrm c})/\delta_{\mathrm g}$ ensures that a follower $f$ projected to enter its contested segment at least $\delta_{\mathrm g}$ after its leader enters its own has arrival-order evidence $g_f^e\leq-g_{\mathrm c}$ by Equation~\eqref{eq:arrival_order_function}. The next lemma bounds the tilt from above to preserve weak-interaction scores and from below to ensure yielding evidence.

		Define  the upper and lower	tilt thresholds by
		\[
		k_{\mathrm{tilt}}^{\max}
		:=\frac{(p_{\mathrm c}-p_{\mathrm w})(1-p_{\mathrm w})}
		{p_{\mathrm w}},
		\]
		and
		\[
		k_{\mathrm{tilt}}^{\min}(T_{\mathrm{sm}})
		:=
		\begin{cases}
			\displaystyle
			\frac{1-p_{\mathrm c}}{p_{\mathrm c}g_{\mathrm c}}
			\left[
			\bar p-p_{\mathrm c}
			+T_{\mathrm{sm}}\log
			\left(\frac{(M_{\max}-1)\bar p}
			{p_{\mathrm c}g_{\mathrm c}}\right)
			\right]_+,
			& M_{\max}>1,\\[6pt]
			0, & M_{\max}=1,
		\end{cases}
		\]
		where $[x]_+:=\max\{x,0\}$.

\begin{lemma}[Tilt bounds preserve weak interactions and enforce yielding evidence]
\label{lem:tilt_upper_bound}
\label{lem:tilt_lower_bound}
Under the pressure bounds in Equation~\eqref{eq:ordered_pressure_bounds}, suppose each agent senses at most $M_{\max}\geq1$ interactions. For fixed $g_{\mathrm c}\in(0,1)$ and $T_{\mathrm{sm}}>0$, there are finite thresholds $k_{\mathrm{tilt}}^{\max}>0$ and $k_{\mathrm{tilt}}^{\min}(T_{\mathrm{sm}})\geq0$ with the following properties:
\begin{enumerate}
\item If $0\leq k_{\mathrm{tilt}}<k_{\mathrm{tilt}}^{\max}$, every weak interaction retains a tilted score below $p_{\mathrm c}$, regardless of its arrival-order evidence.
\item For each fixed $k_{\mathrm{tilt}}>k_{\mathrm{tilt}}^{\min}(T_{\mathrm{sm}})$, a follower $f$ with a safety-critical interaction $e=(f,\ell)$ with its leader $\ell$, satisfying $p^e\geq p_{\mathrm c}$ and $g_f^e\leq-g_{\mathrm c}$, has net evidence $h_f\leq-h_{\mathrm c}$ for some $h_{\mathrm c}>0$ uniform over admissible observations. This bound holds despite competing proceed cues from the follower's other interactions.
\end{enumerate}
\end{lemma}

\begin{proof}
For the upper bound, set $k_{\mathrm{tilt}}^{\max}:=(p_{\mathrm c}-p_{\mathrm w})(1-p_{\mathrm w})/p_{\mathrm w}>0$ and take $0\leq k_{\mathrm{tilt}}<k_{\mathrm{tilt}}^{\max}$.
For $p^{e'}\leq p_{\mathrm w}$, the tilted score is largest at $g_i^{e'}=-1$ and satisfies
\begin{equation*}
\hat{p}^{e'}\leq
p_{\mathrm w}\left(1+\frac{k_{\mathrm{tilt}}}{1-p_{\mathrm w}}\right)
<p_{\mathrm c}.
\end{equation*}

For the lower bound, write $[x]_+:=\max\{x,0\}$ and choose $k_{\mathrm{tilt}}>k_{\mathrm{tilt}}^{\min}(T_{\mathrm{sm}})$, where
\begin{equation}
\label{eq:yielding_tilt_threshold}
k_{\mathrm{tilt}}^{\min}(T_{\mathrm{sm}}):=
\begin{cases}
\displaystyle
\frac{1-p_{\mathrm c}}{p_{\mathrm c}g_{\mathrm c}}
\left[
\bar p-p_{\mathrm c}
+T_{\mathrm{sm}}\log\left(
\frac{(M_{\max}-1)\bar p}{p_{\mathrm c}g_{\mathrm c}}
\right)\right]_+,
&M_{\max}>1,\\[3ex]
0,&M_{\max}=1.
\end{cases}
\end{equation}
The leader interaction satisfies
\begin{equation*}
\hat{p}^e\geq
p_{\mathrm c}+\frac{k_{\mathrm{tilt}}p_{\mathrm c}g_{\mathrm c}}{1-p_{\mathrm c}}
=:L.
\end{equation*}
Let $\mathcal P_f$ contain the other sensed interactions with $g_f^{e'}\geq0$. For these interactions $\hat{p}^{e'}\leq p^{e'}\leq\bar p$. Define $\Delta_{\ell f}:=L-\bar p$.
Cancellation of the softmax normalization gives
\begin{equation*}
\frac{w_f^{e'}}{w_f^e}
=\exp\left(\frac{\hat{p}^{e'}-\hat{p}^e}{T_{\mathrm{sm}}}\right)
\leq e^{-\Delta_{\ell f}/T_{\mathrm{sm}}},
\qquad e'\in\mathcal P_f.
\end{equation*}
Dropping the nonpositive contributions of other yielding interactions yields
\begin{equation}
\label{eq:yielding_evidence_comparison}
\begin{aligned}
h_f
&\leq-p_{\mathrm c}g_{\mathrm c}w_f^e
+\bar p\sum_{e'\in\mathcal P_f}w_f^{e'}\\
&\leq-w_f^e\left[
p_{\mathrm c}g_{\mathrm c}
-(M_{\max}-1)\bar p\,e^{-\Delta_{\ell f}/T_{\mathrm{sm}}}
\right]=:-w_f^eB.
\end{aligned}
\end{equation}
The required margin is
\begin{equation*}
B=p_{\mathrm c}g_{\mathrm c}
-(M_{\max}-1)\bar p\,e^{-\Delta_{\ell f}/T_{\mathrm{sm}}}>0.
\end{equation*}
For $M_{\max}>1$, this is equivalent to
\begin{equation*}
\Delta_{\ell f}>
T_{\mathrm{sm}}\log\left(\frac{(M_{\max}-1)\bar p}{p_{\mathrm c}g_{\mathrm c}}\right),
\end{equation*}
which follows from~\eqref{eq:yielding_tilt_threshold}. For $M_{\max}=1$, $B=p_{\mathrm c}g_{\mathrm c}>0$. Every sensed interaction also satisfies
\begin{equation*}
\hat{p}^{e'}\leq
U:=\bar p\left(1+\frac{k_{\mathrm{tilt}}}{1-\bar p}\right),
\end{equation*}
so
\begin{equation*}
w_f^e\geq
\frac{1}{1+(M_{\max}-1)e^{(U-L)/T_{\mathrm{sm}}}}.
\end{equation*}
Combining this bound with~\eqref{eq:yielding_evidence_comparison} proves
\begin{equation*}
h_f\leq-h_{\mathrm c},
\qquad
h_{\mathrm c}:=
\frac{B}{1+(M_{\max}-1)e^{(U-L)/T_{\mathrm{sm}}}}>0.
\end{equation*}
\end{proof}

Choosing $p_{\mathrm w}/p_{\mathrm c}$ and then $T_{\mathrm{sm}}$ sufficiently small gives a nonempty interval $k_{\mathrm{tilt}}^{\min}(T_{\mathrm{sm}})<k_{\mathrm{tilt}}<k_{\mathrm{tilt}}^{\max}$ in which weak interactions retain scores below $p_{\mathrm c}$ and a safety-critical leader interaction with $g_f^e\leq-g_{\mathrm c}$ gives uniformly negative net evidence.

Under the preceding conditions, pressure tilt makes evidence for yielding to a safety-critical leader dominate competing proceed cues. The next result bounds the time needed for local NDD feedback to generate a stopping command while this negative evidence persists. The follower's remaining traversal distance to its earliest contested segment must cover both the distance it travels while the stopping command is generated and the certified stopping distance in Lemma~\ref{lem:stopping_distance_detection_margin}.

\begin{proposition}[Persistent yielding evidence generates a stopping command in finite time]
\label{prop:following_decision_invariance}
At $t_{\mathrm s}$, let $e=(f,\ell)$ satisfy part~2 of Lemma~\ref{lem:tilt_lower_bound}, with $u_f(t_{\mathrm s})>0$ and $|z_f(t_{\mathrm s})|\leq z_{\max}$. 
{Let $z_f^{\mathrm{stop}}<0$ be an attainable decision threshold for $v_f^{\mathrm{cmd}}\leq v_f^{\mathrm{stop}}$. Define $\underline u_f:=\min\{u_f(t_{\mathrm s}),u_0/\rho_u\}$ and assume the yielding margin $\varepsilon_f:=S(\underline u_f h_{\mathrm c})-\rho_z|z_f^{\mathrm{stop}}|>0$.} While these conditions persist, the follower's decision reaches $z_f^{\mathrm{stop}}$ in bounded time, generating the stopping command, unless the leader is first certified to leave its contested segment before the follower enters its own. The sensing range can be chosen such that there is enough remaining traversal distance for the follower to respond and stop before it enters its contested segment.
\end{proposition}

\begin{proof}
The attention dynamics give $u_f(t)\geq\underline u_f>0$, and the assumed yielding margin is $\varepsilon_f>0$, with
\begin{equation*}
\underline u_f:=\min\left\{u_f(t_{\mathrm s}),\frac{u_0}{\rho_u}\right\},
\qquad
\varepsilon_f:=S(\underline u_fh_{\mathrm c})-\rho_z|z_f^{\mathrm{stop}}|.
\end{equation*}
The evidence bound $h_f\leq-h_{\mathrm c}$ gives, for $z_f\geq z_f^{\mathrm{stop}}$,
\begin{equation}
\tau_z\dot z_f
\leq\rho_z|z_f^{\mathrm{stop}}|-S(\underline u_fh_{\mathrm c})
=-\varepsilon_f<0.
\label{eq:strict_follower_yielding}
\end{equation}
Since the speed-command map increases with $z_f$, the decreasing decision lowers the follower's commanded speed. Unless positive temporal clearance is first certified under the prescribed uncertainty bounds, integrating Equation~\eqref{eq:strict_follower_yielding} bounds the time needed for the decision to reach $z_f^{\mathrm{stop}}$. Adding the certified stopping distance to the maximum travel during this interval gives a sufficient condition on the follower's remaining distance to its contested segment:
\begin{equation}
\varepsilon_f>0,
\qquad
T_f^{\mathrm{yld}}:=
\frac{\tau_z(z_{\max}-z_f^{\mathrm{stop}})}{\varepsilon_f},
\qquad
d_{f^-}^e(t_{\mathrm s})
\geq\overline v_fT_f^{\mathrm{yld}}+D_f^{\mathrm{stop}}.
\label{eq:follower_safety_conditions}
\end{equation}
At $z_f=z_f^{\mathrm{stop}}$, Equation~\eqref{eq:strict_follower_yielding} prevents the decision from increasing above this threshold, so $v_f^{\mathrm{cmd}}\leq v_f^{\mathrm{stop}}$ while the yielding conditions persist. When the command becomes available, this condition ensures that the follower still has at least distance $D_f^{\mathrm{stop}}$ remaining to its contested segment, so the stopping-distance argument of Lemma~\ref{lem:stopping_distance_detection_margin} applies. If $z_f(t_{\mathrm s})\leq z_f^{\mathrm{stop}}$, the command is already available.
\end{proof}

The follower's required distance to its contested segment accounts for motion before the stopping command is generated as well as motion during the certified stop. Positive temporal clearance may be certified before the stopping command is generated, so the follower need not come to a full stop. If the projected entry order changes, the same estimate applies to the new follower whenever its evidence and distance conditions hold.

\subsubsection{Sensing range for safe traversal}

The sensing range must allow agents to detect conflicts early enough to adjust their speeds so that each leader leaves its contested segment before its follower enters its own. Enough distance must also remain for either agent to stop before it enters its contested segment if needed (Proposition~\ref{prop:following_decision_invariance}). The following result gives a finite sensing range that meets these requirements and accounts for the time spent negotiating and yielding.

\begin{proposition}[Sufficient sensing range makes the emerging traversal order safe]
\label{prop:fixed_component_edge_safety}
Assume the conditions of Lemmas~\ref{lem:decision_range_and_speed_extrema} and~\ref{prop:sensing_range_detection_margin}, Propositions~\ref{prop:strict_acyclic_traversal_order} and~\ref{prop:following_decision_invariance}, and the tracking estimate~\eqref{eq:supp_branch_tracking_error}, with positive selected equilibrium speeds. Let the time buffer $\Delta_t^e>0$ account for sensing uncertainty, measurement noise, and the prescribed geometry and motion errors. 
Then there exists a finite sensing range
\begin{equation}
{R_{\mathrm{sens}}\geq\max_e R_{\mathrm{req}}^e,}
\label{eq:complete_sensing_range_bound}
\end{equation}
{Here $R_{\mathrm{req}}^e$ is the pairwise sensing-radius bound defined in Equation~\eqref{eq:asymmetric_sensing_radius}, evaluated at the required detection lead $t_{\mathrm{req}}^e$ constructed in Equation~\eqref{eq:complete_required_detection_time}. } This sensing range preserves safety during negotiation, allows the group to reach $\mathcal Z_{\mathrm{safe}}^c(\boldsymbol q(t))$ before any agent enters its contested segment, and maintains the selected order until physical clearance.
\end{proposition}

\begin{proof}
Fix a conflict $e=(i,j)$ with selected order $\ell\prec_e f$ at branch amplitude $a\ne0$, and write $\vartheta_k^\dagger=\vartheta_k^\dagger(a)$. Proposition~\ref{prop:strict_acyclic_traversal_order} and the positive-speed hypothesis give $\vartheta_\ell^\dagger>\vartheta_f^\dagger>0$. At order formation, set $t_{\mathrm{rem}}^e:=d_{f^-}^e/v_f^{\mathrm{nom}}$ and  {Use a constant-speed reference in which the follower's nominal remaining entry time is $t_{\mathrm{rem}}^e$ and the leader's is $t_{\mathrm{rem}}^e-\delta_e$, and each agent moves at speed $v_k^{\mathrm{nom}}\vartheta_k^\dagger$.} Since the leader's segment has length $2r_{\mathrm c}^e$, the time between the leader leaving its contested segment and the follower entering its own under these equilibrium speeds is
\begin{equation*}
{\pi_{\ell\prec f}^{e,\mathrm{ref}}=t_{\mathrm{rem}}^e\left(\frac{1}{\vartheta_f^\dagger}-\frac{1}{\vartheta_\ell^\dagger}\right)+\frac{\delta_e}{\vartheta_\ell^\dagger}-\frac{2r_{\mathrm c}^e}{v_\ell^{\mathrm{nom}}\vartheta_\ell^\dagger}.}
\end{equation*}
The coefficient of $t_{\mathrm{rem}}^e$ is positive. {Accounting for the nominal entry-time offset $\delta_e$, suppose $|\delta_e|\leq\bar\delta_e$ and the selected fractional speeds have a uniform positive lower bound $\underline\vartheta$. A sufficient condition for temporal clearance is then $t_{\mathrm{rem}}^e\geq t_{\mathrm{safe}}^e(a)$, where the dependence on $\bar\delta_e$ is implicit:}
\begin{equation}
{t_{\mathrm{safe}}^e(a):=
\frac{\Delta_t^e+2r_{\mathrm c}^e/(v_\ell^{\mathrm{nom}}\vartheta_\ell^\dagger)+\bar\delta_e/\underline\vartheta}
{1/\vartheta_f^\dagger-1/\vartheta_\ell^\dagger}<\infty.}
\label{eq:required_clearance_time}
\end{equation}

The assumed error bound includes geometry and motion deviations for both commanded- and actual-speed predictions relative to this fixed motion. Hence $t_{\mathrm{rem}}^e\geq t_{\mathrm{safe}}^e(a)$ ensures positive temporal separation for both commanded- and actual-speed predictions until physical clearance. Applied to every conflict, these inequalities give $\boldsymbol z(t)\in\mathcal Z_{\mathrm{safe}}^c(\boldsymbol q(t))$ and $\boldsymbol v(t)\in\mathcal V_{\mathrm{safe}}^c(\boldsymbol q(t))$ by Equations~\eqref{eq:safe_speed_set}--\eqref{eq:safe_decision_set}. These positive temporal separations preserve the selected traversal orders and prevent simultaneous occupancy of each pair's contested segments, which ensures spatial safety by the contested-segment construction.

Combining the stopping and yielding requirements of Lemma~\ref{lem:stopping_distance_detection_margin} and Proposition~\ref{prop:following_decision_invariance} with Equation~\eqref{eq:required_clearance_time} gives
\begin{equation}
d_{k^-}^e(t_0)\geq
\overline v_k(T^{\mathrm{neg}}+T_k^{\mathrm{yld}})
+\max\{D_k^{\mathrm{stop}},v_k^{\mathrm{nom}}t_{\mathrm{safe}}^{e,k}\},
\qquad k\in\{i,j\},
\label{eq:complete_safe_detection_distance}
\end{equation}
where $t_{\mathrm{safe}}^{e,k}$ is $t_{\mathrm{safe}}^e(a)$ for the equilibrium branch on which agent $k$ is the follower in conflict $e$, and $T_k^{\mathrm{yld}}$ is its yielding-response bound from Equation~\eqref{eq:follower_safety_conditions}.  {Define $t_{\mathrm{det}}^e:=\min_{k\in\{i,j\}}d_{k^-}^e(t_0)/v_k^{\mathrm{nom}}$, so that $d_{k^-}^e(t_0)\geq v_k^{\mathrm{nom}}t_{\mathrm{det}}^e$ for both agents.} Including the time-scale requirement~\eqref{eq:supp_negotiation_sensing_budget} therefore gives the lower bound on detection lead time
\begin{equation}
t_{\mathrm{det}}^e\geq t_{\mathrm{req}}^e:=
\max_{k\in\{i,j\}}\max\left\{
\frac{\overline v_k(T^{\mathrm{neg}}+T_k^{\mathrm{yld}})
+\max\{D_k^{\mathrm{stop}},v_k^{\mathrm{nom}}t_{\mathrm{safe}}^{e,k}\}}
{v_k^{\mathrm{nom}}},
\frac{\overline v_kT^{\mathrm{neg}}}{\varepsilon_{\mathrm{trav}}v_k^{\mathrm{nom}}}
\right\}<\infty.
\label{eq:complete_required_detection_time}
\end{equation}
Lemma~\ref{prop:sensing_range_detection_margin} gives a finite sensing radius that provides this detection lead time. {Taking the supremum of the pairwise sensing requirements
	over the domain, with entry-time offsets accounted for by
	Equation~\eqref{eq:asymmetric_sensing_radius}, yields a finite
	common sensing bound provided these requirements are uniformly bounded.}
\end{proof}

Earlier detection allows smaller fractional-speed differences to produce safe temporal separation. For fixed geometry and zero timing bias, Equations~\eqref{eq:branch_edge_fractional_speed} and~\eqref{eq:required_clearance_time} give, as $a\to0$ with its sign fixed,
\begin{equation*}
t_{\mathrm{safe}}^e(a)\sim
\frac{\Delta_t^e+2r_{\mathrm c}^e/v_\ell^{\mathrm{nom}}}
{|a|\,|\nu_\ell\phi_{\max,\ell}-\nu_f\phi_{\max,f}|}.
\end{equation*}
The numerator is the leader's nominal traversal time through its contested segment plus the robustness buffer $\Delta_t^e$. The denominator is the leading-order difference between the leader's and follower's fractional speeds. Thus, near collective indecision and with geometry and buffer fixed, halving this speed difference approximately doubles the follower's required distance to its contested segment at order formation under this bound. Longer contested segments or a larger robustness buffer also require earlier detection. The full detection lead in Equation~\eqref{eq:complete_required_detection_time} additionally accounts for negotiation, yielding, stopping, and time-scale separation.

Local NDD feedback selects the critical decision amplitude and collective decision pattern, producing both the traversal order and the speed differences that realize it safely. Sufficient sensing range lets this emergent response establish the required temporal separation without prescribed pairwise speed targets or traversal schedules. {With heterogeneous sensing radii, each conflicting pair instead requires $\min \{R_i,R_j\}\geq R_{\mathrm{req}}^e$ in Equation~\eqref{eq:asymmetric_sensing_radius}.}

\subsubsection{Joint parameter design for safe-and-live deconfliction}
\label{sec:supp_joint_design}

The \emph{operating design domain} (ODD) consists of the admissible conflict episodes specified in Materials and Methods (\textit{Admissible conflict episodes and scope}), subject to Assumptions~\ref{ass:episode}--\ref{ass:uniform_spectral_observations} and the prescribed bounds on geometry, speeds, stopping distances, sensing errors, and actuation and speed-tracking performance. One set of controller parameters and one sensing range serve this domain when the following five conditions hold.
\begin{enumerate}
\item \emph{Pressure and spectral separation.} The pressure gains satisfy the ordered pressure bounds in Equation~\eqref{eq:ordered_pressure_bounds}. Together with the softmax temperature, they also give separated weak and critical spectral bounds, $\overline\Lambda_{\mathrm w}<\underline\Lambda_{\mathrm c}$, as required by Lemma~\ref{lem:uniform_spectral_bands}. {For nonconflicting sensed interactions, choose the pressure and softmax parameters to satisfy the $C^1$ and symmetrized spectral perturbation margins of Section~\ref{sec:supp_spectral_tuning}.}

\item \emph{Autonomous differentiation.} Choose the spectral threshold $\lambda^\ast$ between $\overline\Lambda_{\mathrm w}$ and $\underline\Lambda_{\mathrm c}$, and set the basal attention input to $u_0=\rho_z\rho_u/[S'(0)\lambda^\ast]$. {Under the assumed uniform symmetry-breaking condition, a projected entry-time bias or a nonzero inherited decision amplitude excites the unstable critical mode and selects an attracting differentiated branch.} As required by Theorem~\ref{thm:autonomous_spectral_tuning}, the fast subsystem state reaches the selected equilibrium's attraction neighborhood within $T^{\mathrm{neg}}$ and tracks that equilibrium as the agents move. The equilibrium fractional-speed differences on the selected branch are strictly nonzero on every physical conflict. {During the required traversal-order interval, the selected equilibrium remains on a compact, uniformly hyperbolic and attracting branch with a uniform strict fractional-speed gap; no further loss of hyperbolicity occurs before clearance or renewed negotiation after a graph event.}

\item \emph{Local decisions with full speed authority.} Decision damping and speed-map gains satisfy $z_{\max}=s_{\max}/\rho_z\leq z_{\mathrm{loc}}$ and Equations~\eqref{eq:full_speed_range_condition} and~\eqref{eq:local_decision_range_condition}. For $|z_i(0)|\leq z_{\max}$, the decision remains in $[-z_{\max},z_{\max}]$. Every admissible speed command, including a certified stopping command, is attainable within this interval.

\item \emph{Response time and sensing.}
{Under uniform symmetry breaking, the negotiation
	interval $T^{\mathrm{neg}}$ accounts for any delay before
	strict instability, growth of the unstable decision mode,
	attraction to the differentiated branch, and the associated
	tracking and actuation response.}
The decision, attention, and speed dynamics thereby establish
the differentiated response and positive temporal clearance
within $T^{\mathrm{neg}}$, while
Equation~\eqref{eq:supp_pressure_variation_design}
limits pressure variation during negotiation.
The required detection lead accounts for negotiation,
yielding, stopping, temporal clearance, time-scale separation,
and sensing and processing delays.
In particular, a uniformly bounded numerator and a uniformly
positive denominator in Equation~\eqref{eq:required_clearance_time}
bound the required nominal lead $t_{\mathrm{safe}}^e(a)$.
{Nominal entry-time offsets satisfy
	$|\delta_e|\leq\bar\delta_e$ over the prescribed domain
	and are accounted for by the pairwise sensing
	bound~\eqref{eq:asymmetric_sensing_radius}, separately
	from the buffer for sensing, tracking, actuation,
	and geometry errors.}
The sensing range satisfies
Equation~\eqref{eq:complete_sensing_range_bound}.
A single sensing range serves the ODD provided these
pairwise requirements are uniformly bounded.

\item \emph{Yielding and clearance persistence.}
The entry-time margin $\delta_{\mathrm g}$ exceeds the
projected entry-time uncertainty. Choosing
$k_{\mathrm g}\geq\operatorname{arctanh}(g_{\mathrm c})/
\delta_{\mathrm g}$ gives arrival-order evidence at most
$-g_{\mathrm c}$ whenever the follower is projected to
enter its contested segment at least $\delta_{\mathrm g}$
after its leader enters its own.
The pressure-tilt gain satisfies
$k_{\mathrm{tilt}}^{\min}(T_{\mathrm{sm}})
<k_{\mathrm{tilt}}<k_{\mathrm{tilt}}^{\max}$,
and each possible follower satisfies
Equation~\eqref{eq:follower_safety_conditions}.
As required by
Proposition~\ref{prop:fixed_component_edge_safety},
the buffer $\Delta_t^e$ strictly exceeds the certified
uniform bound on the loss of temporal separation due to
geometry variations and sensing, actuation, tracking,
and implementation errors from order formation until
physical clearance.
{During admissible conflict activations and
	clearances, group splits, and mergers, each continuing
	conflict either retains its certified positive
	temporal-clearance buffer or leaves the affected follower
	sufficient time and distance to yield and stop under
	Proposition~\ref{prop:following_decision_invariance}
	until positive clearance is restored.}

\end{enumerate}

These conditions couple sensing to feedback control tuning. A larger sensing radius can increase local congestion $M_{\max}$ and the projected-time bound $\bar t$, changing both the spectral and yielding estimates. The same softmax temperature and arrival-order gain enter both estimates. Basal attention also links the bifurcation threshold to the yielding margin through $u_0/\rho_u=\rho_z/[S'(0)\lambda^\ast]$, so threshold placement must retain $\varepsilon_f>0$. The bounds $\underline t_{\mathrm c}$, $\bar t$, $\underline\nu$, and $\bar\nu$ therefore correspond to the chosen sensing, geometry, and speed-map design. Certification uses these prescribed bounds, without online estimates of eigenvalues or group size.

Theorem~\ref{thm:autonomous_spectral_tuning} establishes a traversal order through decision and speed differentiation. The joint conditions above preserve safety while that order forms and allow a leading mover to traverse its contested segment and clear a physical conflict.

\begin{theorem}[Safe traversal and finite-time progress in a fixed group]
\label{thm:fixed_component_safe_order}
Consider a locally engaged group $c$ in an episode within the ODD of Section~\ref{sec:supp_joint_design}, whose physical-conflict set remains fixed for at least $T^{\mathrm{neg}}$. Suppose the hypotheses of Theorem~\ref{thm:autonomous_spectral_tuning}, the joint parameter conditions of Section~\ref{sec:supp_joint_design}, and the physical-safety conditions of Proposition~\ref{prop:fixed_component_edge_safety} hold. Then:
\begin{enumerate}
\item every conflicting pair remains at least $d^{\mathrm{safe}}$ apart during negotiation;
\item within $T^{\mathrm{neg}}$, the group's decision state belongs to $\mathcal Z_{\mathrm{safe}}^c(\boldsymbol q(t))$ and defines a strict acyclic traversal order. Each conflicting pair has positive temporal clearance under the prescribed uncertainties: $\pi_{i\prec j}^e(\boldsymbol q(t),\boldsymbol v(t))>0$ or $\pi_{j\prec i}^e(\boldsymbol q(t),\boldsymbol v(t))>0$;
\item each selected pairwise order persists until physical clearance in the absence of new conflict activations; and
\item the set of leading movers $\mathcal M^c(t)$ becomes nonempty, and at least one physical conflict clears in finite time.
\end{enumerate}
\end{theorem}

\begin{proof}
For conclusion 1, consider a conflicting pair from detection until temporal clearance is established. Equation~\eqref{eq:complete_safe_detection_distance} leaves each agent enough distance to its contested segment to complete negotiation and, if selected as the follower, yield and stop if necessary before it enters that segment. Under the follower-evidence conditions, Proposition~\ref{prop:following_decision_invariance} generates the stopping command within the allowed time unless temporal clearance occurs first. The follower therefore remains outside its contested segment while the projected occupancy intervals can overlap. By construction of the contested segments, this excludes pairwise separation below $d^{\mathrm{safe}}$ during negotiation.

For conclusion 2, Theorem~\ref{thm:autonomous_spectral_tuning} brings the decision--attention--speed state into the attraction neighborhood of a differentiated equilibrium whose fractional-speed differences define a strict acyclic traversal order. Proposition~\ref{prop:fixed_component_edge_safety} supplies the sensing lead needed for those speed differences to produce positive temporal clearance and for the decision state to belong to $\mathcal Z_{\mathrm{safe}}^c(\boldsymbol q(t))$. The response-time conditions in Section~\ref{sec:supp_joint_design} ensure that the group establishes this safe order within $T^{\mathrm{neg}}$.

For conclusion 3, smooth continuation of the selected equilibrium and the prescribed bound on the subsystem's tracking error preserve each pairwise order and its temporal clearance until the conflict clears physically or a new physical conflict activates.

For conclusion 4, Proposition~\ref{prop:strict_acyclic_traversal_order} identifies an agent that is a source of the traversal-order graph and whose equilibrium speed exceeds its nominal speed on the selected local branch. The prescribed bound on the difference between realized and equilibrium speeds ensures that this agent has realized speed at least $\underline v>0$, so it belongs to $\mathcal M^c(t)$. It traverses its finite-length contested segments and clears a physical conflict in finite time. Section~\ref{sec:decision_to_global_deconfliction} extends this argument to the graph changes caused by physical clearances and conflict activations.
\end{proof}

\subsection{Safe-and-live deconfliction across a changing conflict graph}
\label{sec:decision_to_global_deconfliction}

Theorems~\ref{thm:autonomous_spectral_tuning} and~\ref{thm:fixed_component_safe_order} establish safe traversal within a fixed group. We now consider physical-conflict activations and clearances. Conflict activations may occur within a group, involve agents entering the population, or merge groups; clearances may split a group. The following theorem extends the safe-and-live deconfliction guarantees to the updated conflict set: traversal orders adapt as agents' states evolve continuously while preserving safety, and the remaining conflicts clear in finite time after the final conflict activation in the episode.

\begin{theorem}[Safe adaptation under conflict-graph reconfiguration]
\label{thm:graph_reconfiguration_induction}
Under Assumption~\ref{ass:episode} and the joint design of Section~\ref{sec:supp_joint_design}, consider a conflict episode within the ODD. Suppose every newly activated physical conflict satisfies Equation~\eqref{eq:complete_safe_detection_distance}, the speed gains satisfy Equations~\eqref{eq:full_speed_range_condition} and~\eqref{eq:local_decision_range_condition}, and each agent entering the population satisfies the ODD's physical conditions and starts at the nominal fast subsystem equilibrium. Then:
\begin{enumerate}
\item physical clearance removes a pair from $\mathcal P(t)$ and preserves safety and any established orders on the remaining pairs, including when a component $\mathcal G^c(t)$ splits;
\item a conflict activation adds a pair to $\mathcal P(t)$. The emergent traversal order in the affected component $\mathcal G^c(t)$ adapts as existing agents' states evolve continuously while preserving safety. This includes conflict activations within a component, those involving an agent entering the population, and those merging previously separate components. Whenever the updated set of physical conflicts $\mathcal P^c(t)$ remains fixed for $T^{\mathrm{neg}}$, a strict acyclic order is established within that interval and may differ from the previous order on any conflict; and
\item after the final conflict activation in the episode, all remaining physical conflicts clear in finite time: $\mathcal P(t_{\mathrm f})=\varnothing$ for some finite $t_{\mathrm f}$, which ends that episode.
\end{enumerate}
\end{theorem}

\begin{proof}
For conclusion 1, when a physical conflict clears, its reciprocal conflict edges are removed while the agents' states $(\boldsymbol q,\boldsymbol z,\boldsymbol u,\boldsymbol v)$ remain continuous. If the pair remains sensed, it continues to contribute conflict evidence as its pressure decreases through the weak and release bands. For an established traversal order, the prescribed continuation and tracking bounds keep the subsystem state close to the evolving selected equilibrium, as in Lemma~\ref{prop:smooth_dominant_eigenvector_branch} and Equation~\eqref{eq:supp_branch_tracking_error}. The strict fractional-speed differences on unresolved conflicts retain their signs, and Proposition~\ref{prop:fixed_component_edge_safety} preserves positive temporal clearance. When a weak interaction subsequently leaves sensing range, its evidence contribution is removed and the affected agents recompute their normalized weights. If negotiation is still in progress, the yielding and stopping conditions of Proposition~\ref{prop:following_decision_invariance} preserve physical safety. Thus a group split preserves safety and any established orders on the remaining conflicts.

{During changes in the sensed-interaction sets, 	including physical conflict clearances, each continuing conflict either maintains positive temporal clearanc eafter accounting for the buffer $\Delta_t^e$ or allows	sufficient time and distance for the follower to yield	and stop safely.}

For conclusion 2, at a conflict activation, existing agents retain their continuous states, including agents in groups that merge. Agents entering the population start at the nominal fast subsystem equilibrium. For each new conflict, Equation~\eqref{eq:complete_safe_detection_distance} leaves enough distance for negotiation and for either candidate follower to yield and, unless temporal clearance is established first, stop before it enters its contested segment. Proposition~\ref{prop:following_decision_invariance} applies to the instantaneous follower on existing conflicts while the group adapts to the new conflict. Lemma~\ref{lem:decision_range_and_speed_extrema} keeps decisions in the prescribed local neighborhood and makes every required speed command, including a certified stopping command, attainable.

{Activation of a new conflict may alter the speed of an agent involved in an unresolved conflict. The reconfiguration conditions of Section~\ref{sec:supp_joint_design} ensure that each continuing conflict either maintains positive temporal clearance after accounting for $\Delta_t^e$ or retains sufficient negative yielding evidence and stopping distance for the follower to stop safely until positive clearance is restored. Equation~\eqref{eq:complete_safe_detection_distance} provides the corresponding detection and stopping reserve for the newly activated pair.}

The inherited decision state supplies the initial amplitude along the dominant eigenmode of the updated interaction-sensitivity matrix $\boldsymbol A$, while new projected entry-time offsets contribute to the unfolding term $\eta_\tau$ in Equation~\eqref{eq:supp_critical_equation}. If the group's updated physical-conflict set remains fixed for $T^{\mathrm{neg}}$, the branch-selection and response-time conditions of Theorem~\ref{thm:autonomous_spectral_tuning} ensure that these cues select an attracting differentiated equilibrium and that the fast subsystem state reaches its attraction neighborhood within that interval. Theorem~\ref{thm:fixed_component_safe_order} then establishes a safe acyclic traversal order. Proposition~\ref{prop:strict_acyclic_traversal_order} identifies a source of the updated traversal-order graph with positive equilibrium speed; this agent need not have been a source before the conflict activation.

A further conflict activation repeats this argument on the updated group. Physical clearances during negotiation preserve safety and remove traversal-order constraints. Each remaining group establishes a safe order if its physical-conflict set remains fixed for $T^{\mathrm{neg}}$. Previously selected orders may change as the agents adapt to new conflicts.

For conclusion 3, the finite-activation condition in Assumption~\ref{ass:episode} gives a finite time after which no further physical conflict activates in the current episode. From any later time within that episode with unresolved conflicts, either a physical conflict clears within $T^{\mathrm{neg}}$, or the conflict set remains fixed throughout that interval. In the latter case, Theorem~\ref{thm:fixed_component_safe_order} supplies a leading mover that clears a physical conflict in finite time. Each clearance reduces the finite number of unresolved physical conflicts, including when it splits a group. Induction on that number proves that every remaining conflict clears in finite time, ending the episode.
\end{proof}

\subsubsection{Proof of Theorem~\ref{thm:finite_time_safe_deconfliction}}
\label{sec:supp_proof_global}

Applying the group results concurrently and inducting over graph events establishes the population-level guarantee stated in the main text.

\begin{proof}
The joint conditions of Section~\ref{sec:supp_joint_design} give controller parameters, sensing radii, and robustness margins that apply throughout the ODD. Their construction uses the prescribed bounds on local congestion, geometry, sensing errors, actuation, speed tracking, and stopping distances. It also requires uniform bounds on the time to reach the selected equilibrium's attraction neighborhood and on the sensing lead needed for positive temporal clearance. These design bounds are independent of the realized group and population sizes, establishing the existence of one common design over the ODD.

For conclusion 1, each summand of Equation~\eqref{eq:ndd:evidence} uses the agent's own state and one sensed pairwise observation. Each agent performs one motion projection per sensed interaction, aggregates scalar evidence, and integrates its local decision, attention, and speed dynamics. This computation is linear in the number of sensed interactions, which is bounded by $M_{\max}$. Equation~\eqref{eq:speed_response} accommodates heterogeneous tracking and actuation within the stated bounds. No calculation requires group-state reconstruction, communication, centralized scheduling, or combinatorial order assignment.

For conclusion 2, the components partition the physical conflicts at every time. Theorem~\ref{thm:fixed_component_safe_order} gives safety on fixed conflict sets. Theorem~\ref{thm:graph_reconfiguration_induction} preserves it when conflicts clear, groups split, or new conflicts activate within a group, involve agents entering the population, or merge groups. The design buffers $\Delta_t^e$ cover the prescribed geometry variations and sensing, tracking, actuation, and implementation errors, leaving positive temporal clearance. Assumption~\ref{ass:episode} ensures finitely many graph events on each bounded time interval, so induction over those events gives $\|\boldsymbol q_i(t)-\boldsymbol q_j(t)\|\geq d^{\mathrm{safe}}$ for every $\{i,j\}\in\mathcal P(t)$.

For conclusion 3, Theorem~\ref{thm:fixed_component_safe_order} establishes a safe acyclic traversal order whenever a group's physical-conflict set remains fixed for $T^{\mathrm{neg}}$. Theorem~\ref{thm:graph_reconfiguration_induction} preserves established orders when physical conflicts clear and groups split. Traversal orders adapt safely to new conflicts as existing decisions and speeds evolve continuously, including when agents entering the population join a group or previously separate groups merge. The affected group establishes a new safe order if its updated conflict set remains fixed for $T^{\mathrm{neg}}$; the leader and follower may change on any conflict.

For conclusion 4, the finite-activation condition in Assumption~\ref{ass:episode} ensures that conflict activations cease after a finite time within the current episode, including those involving agents entering the population or merging groups. Theorem~\ref{thm:graph_reconfiguration_induction} then gives a finite $t_{\mathrm f}$ with $\mathcal P(t_{\mathrm f})=\varnothing$, ending that episode. Different components share no physical conflict and can progress concurrently. Proposition~\ref{prop:strict_acyclic_traversal_order} likewise excludes a physical conflict between two sources of the traversal-order graph, permitting concurrent progress by nonconflicting leading movers.

For conclusion 5, consider agent $i$ after its final physical conflict in the current episode clears, and suppose no new conflict involving $i$ activates. As its remaining sensed interactions enter the release regime, their contributions to $h_i$ decay. As $h_i\to0$, the stable decision dynamics in Equation~\eqref{eq:ndd:decision} give $z_i\to0$; once $h_i=0$, the decay is exponential because $\tau_z\dot z_i=-\rho_z z_i$. Equation~\eqref{eq:ndd:attention} then gives $u_i\to u_0/\rho_u$. The speed command tends to $v_i^{\mathrm{nom}}$, and stable speed tracking gives $v_i\to v_i^{\mathrm{nom}}$.
\end{proof}

\subsection{Local bifurcation analysis at collective indecision}
\label{sec:supp_local_bifurcation}

At the spectral threshold, $\boldsymbol\phi_{\max}$ identifies the critical direction in decision space. Keeping the traversal distances fixed as in Section~\ref{sec:supp_spectral_tuning}, we expand the conflict-evidence map $\boldsymbol h$ and the reduced equilibrium map $\boldsymbol F$ about collective indecision. We solve the noncritical projection of $\boldsymbol F=\boldsymbol 0$ and obtain a scalar equation for the equilibrium amplitude along $\boldsymbol\phi_{\max}$.

The maps $\boldsymbol h$ and $\boldsymbol F$ are smooth in the neighborhood used for the expansion, where the traversal distances and sensed interaction set are fixed. A smooth function of squared closest-approach distance is one suitable realization of the spatial pressure $P_{\mathrm d}$. The local controller design uses the simple isolated dominant mode of Proposition~\ref{prop:interaction_sensitivity_spectrum}, with the noncritical fast modes stable.

Use the unfolding parameter $\mu=\lambda_{\max}(\boldsymbol A)-\lambda^\ast$ and let $\bar{\boldsymbol\tau}$ collect the projected entry-time offsets at collective indecision. Set $u_\ast:=u_0/\rho_u$ and $s_1:=S'(0)$, and denote the reduced Jacobian in Equation~\eqref{eq:steady_state_jacobian} by $\boldsymbol J(\mu)$. At $\mu=0$, its kernel is spanned by the right eigenvector $\boldsymbol\phi_{\max}$. Choose the corresponding left eigenvector $\boldsymbol\psi_{\max}$ so that $\boldsymbol\psi_{\max}^{\mathsf T}\boldsymbol\phi_{\max}=1$, and continue both eigenvectors smoothly in $\mu$.

\subsubsection{Conflict-evidence expansion at collective indecision}
\label{sec:supp_evidence_expansion}

For a directed interaction $e=(i,j)$, define $\tau_i^e:=t^e_{j^-}-t^e_{i^-}$ and $G_i^e:=\tanh(k_{\mathrm g}\tau_i^e)$. The reciprocal directed interaction $\bar e=(j,i)$ satisfies $\tau_j^{\bar e}=-\tau_i^e$ and $G_j^{\bar e}=-G_i^e$.

At a speed-tracking equilibrium, the realized speed equals the commanded speed, $v_k=v_k^{\mathrm{nom}}+V_k(z_k)$. Holding $d^e_{k^-}$ fixed therefore gives
\begin{equation*}
t^e_{k^-}(z_k)
=\frac{d^e_{k^-}}{v_k^{\mathrm{nom}}+V_k(z_k)}
=t_k^{e,-,\ast}\left(1+\frac{V_k(z_k)}{v_k^{\mathrm{nom}}}\right)^{-1},
\qquad
t_k^{e,-,\ast}:=\frac{d^e_{k^-}}{v_k^{\mathrm{nom}}}.
\end{equation*}
Thus $t_k^{e,-,\ast}$ is the time until agent $k$, moving at its nominal speed, reaches its contested segment; it factors out of the Taylor expansion in $z_k$. For the shared decision-to-speed gains in Equation~\eqref{eq:tanh_speed_map}, define $\nu:=\kappa_V\kappa_z$ and $\chi:=\kappa_V\kappa_z^3(1/3-\kappa_V^2)$. The same calculation applies componentwise with heterogeneous gains. Expanding the reciprocal gives
\begin{equation*}
t^e_{k^-}(z_k)
=
t_k^{e,-,\ast}
\left(
1-\nu z_k+\nu^2z_k^2+\chi z_k^3
\right)
+
\mathcal O(|z_k|^4).
\end{equation*}
At the reference group tie, agents $i$ and $j$ have the same projected time $t^{e,\ast}$ until they enter their respective contested segments. Subtracting their expansions and allowing a small projected entry-time bias gives
\begin{equation}
\tau_i^e(\boldsymbol  z)
=
\bar\tau_i^e
+
t^{e,\ast}
\left[
\nu(z_i-z_j)
-
\nu^2(z_i^2-z_j^2)
-
\chi(z_i^3-z_j^3)
\right]
+
\mathcal O\left(
\|\boldsymbol  z\|^4
+
|\bar\tau_i^e|\|\boldsymbol  z\|
\right),
\label{eq:supp_tau_expansion}
\end{equation}
where $\bar\tau_i^e:=\tau_i^e(\boldsymbol  0)$ is a small projected entry-time bias. In the expansion about the group tie, $\bar\tau_i^e=0$.

Write $\tau_i^e=\bar\tau_i^e+\tau_{i,1}^e+\tau_{i,2}^e+\tau_{i,3}^e+\cdots$, where $\tau_{i,q}^e$ is homogeneous of degree $q$ in $\boldsymbol  z$. Equation~\eqref{eq:supp_tau_expansion} gives
\begin{equation*}
\tau_{i,1}^e
=
t^{e,\ast}\nu(z_i-z_j),
\qquad
\tau_{i,2}^e
=
-t^{e,\ast}\nu^2(z_i^2-z_j^2),
\qquad
\tau_{i,3}^e
=
-t^{e,\ast}\chi(z_i^3-z_j^3).
\end{equation*}
At the group tie,
\begin{equation}
\label{eq:supp_arrival_evidence_expansion}
G_i^e
=
k_{\mathrm g}\tau_{i,1}^e
+
k_{\mathrm g}\tau_{i,2}^e
+
\left[
k_{\mathrm g}\tau_{i,3}^e
-
\frac{k_{\mathrm g}^3}{3}(\tau_{i,1}^e)^3
\right]
+
\mathcal O(\|\boldsymbol  z\|^4).
\end{equation}
The first cubic term comes from the entry-time response, while the second comes from saturation of the arrival-order map.

Let $\omega_i^e(\boldsymbol z):=P_i^e(\boldsymbol z)W_i^e(\boldsymbol z)$. With the traversal distances and sensed interaction set fixed, expand $\omega_i^e$ about $\boldsymbol z=\boldsymbol 0$, retaining the dependence of pressure and weights on decision-dependent speeds. Here $D\omega_i^e(\boldsymbol 0)$ is the first derivative with respect to $\boldsymbol z$, a linear map, and $D^2\omega_i^e(\boldsymbol 0)$ is the second derivative, a bilinear map:
\begin{equation}
\label{eq:supp_pressure_weight_expansion}
\omega_i^e(\boldsymbol  z)
=
\omega_i^{e,\ast}
+
D\omega_i^e(\boldsymbol 0)\,\boldsymbol  z
+
\frac{1}{2}D^2\omega_i^e(\boldsymbol 0)(\boldsymbol  z,\boldsymbol  z)
+
\mathcal O(\|\boldsymbol  z\|^3).
\end{equation}
Since $h_i(\boldsymbol  z)=\sum_{e\in\mathcal E_i}\omega_i^e(\boldsymbol  z)G_i^e(\boldsymbol  z)$, multiplying Equations~\eqref{eq:supp_arrival_evidence_expansion} and~\eqref{eq:supp_pressure_weight_expansion} gives the contribution of edge $e$ at each order:
\begin{equation*}
\begin{aligned}
\text{linear:}\quad
&
\omega_i^{e,\ast}k_{\mathrm g}\tau_{i,1}^e,
\\[2pt]
\text{quadratic:}\quad
&
\omega_i^{e,\ast}k_{\mathrm g}\tau_{i,2}^e
+
D\omega_i^e(\boldsymbol 0)\,\boldsymbol  z\,k_{\mathrm g}\tau_{i,1}^e,
\\[2pt]
\text{cubic:}\quad
&
\omega_i^{e,\ast}
\left[
k_{\mathrm g}\tau_{i,3}^e
-
\frac{k_{\mathrm g}^3}{3}(\tau_{i,1}^e)^3
\right]
+
D\omega_i^e(\boldsymbol 0)\,\boldsymbol  z\,k_{\mathrm g}\tau_{i,2}^e
+
\frac{1}{2}D^2\omega_i^e(\boldsymbol 0)(\boldsymbol  z,\boldsymbol  z)\,k_{\mathrm g}\tau_{i,1}^e.
\end{aligned}
\end{equation*}
To collect the quadratic and cubic evidence terms, define the symmetric multilinear maps $\mathcal H_{\mathrm{quad}}:=D_{\boldsymbol z}^2\boldsymbol h(\boldsymbol 0;0)$ and $\mathcal H_{\mathrm{cub}}:=D_{\boldsymbol z}^3\boldsymbol h(\boldsymbol 0;0)$ at the critical group tie. These are vector-valued functions of two and three vector arguments, respectively, linear in each argument separately. Evaluating them with every argument equal to $\boldsymbol z$ gives terms quadratic and cubic in $\boldsymbol z$; the factors $1/2$ and $1/6$ are the Taylor coefficients. Summing over the sensed directed interactions gives
\begin{equation}
\boldsymbol  h(\boldsymbol  z;\mu)
=
\boldsymbol  A(\mu)\boldsymbol  z
+
\frac{1}{2}\mathcal H_{\mathrm{quad}}(\boldsymbol  z,\boldsymbol  z)
+
\frac{1}{6}\mathcal H_{\mathrm{cub}}(\boldsymbol  z,\boldsymbol  z,\boldsymbol  z)
+
\mathcal O\left(
\|\boldsymbol  z\|^4+|\mu|\|\boldsymbol  z\|^2
\right).
\label{eq:main_evidence_expansion}
\end{equation}
Because $G_i^e(\boldsymbol 0)=0$ at the group tie, derivatives of pressure and weights do not enter the linear coefficient $\boldsymbol A$; only their values at the tie enter. Their first and second derivatives contribute to $\mathcal H_{\mathrm{quad}}$ and $\mathcal H_{\mathrm{cub}}$, respectively.

Similarly, let $\mathcal B_{\mathrm{quad}}:=D_{\boldsymbol z}^2\boldsymbol F(\boldsymbol 0;0)$ and $\mathcal C_{\mathrm{cub}}:=D_{\boldsymbol z}^3\boldsymbol F(\boldsymbol 0;0)$ be the bilinear and trilinear derivative maps of the reduced equilibrium map $\boldsymbol F$. Substituting Equation~\eqref{eq:main_evidence_expansion} into Equation~\eqref{eq:reduced_steady_state} yields
\begin{equation}
\boldsymbol  F(\boldsymbol  z;\mu)
=
\boldsymbol  J(\mu)\boldsymbol  z
+
\frac{1}{2}\mathcal B_{\mathrm{quad}}(\boldsymbol  z,\boldsymbol  z)
+
\frac{1}{6}\mathcal C_{\mathrm{cub}}(\boldsymbol  z,\boldsymbol  z,\boldsymbol  z)
+
\mathcal O\left(
\|\boldsymbol  z\|^4+|\mu|\|\boldsymbol  z\|^2
\right),
\label{eq:main_reduced_map_expansion}
\end{equation}
The definitions above give $\mathcal B_{\mathrm{quad}}=s_1u_\ast\mathcal H_{\mathrm{quad}}$. The map $\mathcal C_{\mathrm{cub}}$ combines the cubic evidence response, saturation of $S$, and decision-dependent attention.

Let $s_3:=S'''(0)$. Oddness of $S$ gives $S''(0)=0$. Evaluating the trilinear map $\mathcal C_{\mathrm{cub}}$ along a direction $\boldsymbol x$ gives
\begin{equation}
\label{eq:supp_cubic_tensor}
\begin{aligned}
\mathcal C_{\mathrm{cub}}(\boldsymbol  x,\boldsymbol  x,\boldsymbol  x)
&=
s_1u_\ast
\mathcal H_{\mathrm{cub}}(\boldsymbol  x,\boldsymbol  x,\boldsymbol  x)
+
s_3u_\ast^3(\boldsymbol  A\boldsymbol  x)^{\odot3}
+
\frac{6s_1\kappa_u}{\rho_u}
(\boldsymbol  x\odot\boldsymbol  x)\odot\boldsymbol  A\boldsymbol  x.
\end{aligned}
\end{equation}
These terms arise from the cubic evidence response, saturation of $S$, and decision-dependent attention, respectively.

Expanding the evidence at $\boldsymbol z=\boldsymbol 0$ with nonzero projected entry-time offsets gives the bias term in the scalar critical equation. At $\boldsymbol  z=\boldsymbol  0$, $G_i^e=k_{\mathrm g}\bar\tau_i^e+\mathcal O((\bar\tau_i^e)^3)$, so
\begin{equation}
\label{eq:supp_unfolding_bias}
(h_\tau)_i
=
\sum_{e\in\mathcal E_i}
\omega_i^{e,\ast}k_{\mathrm g}\bar\tau_i^e
+
\mathcal O(\|\bar{\boldsymbol \tau}\|^2),
\qquad
\eta_\tau
=
s_1u_\ast\boldsymbol \psi_{\max}^{\mathsf T}\boldsymbol  h_\tau.
\end{equation}
Thus $\eta_\tau=0$ at the group tie, while a small projected entry-time mismatch near the group tie supplies the leading constant term in the scalar critical equation.

The multilinear maps in Equation~\eqref{eq:main_reduced_map_expansion} and the unfolding term $\eta_\tau$ in Equation~\eqref{eq:supp_unfolding_bias} are the inputs to the Lyapunov--Schmidt reduction.

\subsubsection{Lyapunov--Schmidt split into range and critical equations}
\label{sec:supp_ls_split}

Apply the Lyapunov--Schmidt reduction~\cite{golubitsky1985singularities} first at zero entry-time bias. The matrices $\boldsymbol\Pi_c(\mu):=\boldsymbol\phi_{\max}(\mu)\boldsymbol\psi_{\max}(\mu)^{\mathsf T}$ and $\boldsymbol\Pi_s(\mu):=\boldsymbol I-\boldsymbol\Pi_c(\mu)$ project onto the critical eigendirection and its complementary noncritical subspace, respectively. We therefore write $\boldsymbol z=a\boldsymbol\phi_{\max}(\mu)+\boldsymbol y$, where $a=\boldsymbol\psi_{\max}(\mu)^{\mathsf T}\boldsymbol z$ and $\boldsymbol y=\boldsymbol\Pi_s(\mu)\boldsymbol z$. In particular, $\boldsymbol\psi_{\max}(\mu)^{\mathsf T}\boldsymbol y=0$.

Applying these projections to the equilibrium condition $\boldsymbol F(\boldsymbol z;\mu)=\boldsymbol 0$ gives two equations. The \emph{range equation} sets the noncritical projection of the equilibrium residual $\boldsymbol F$ to zero:
\begin{equation*}
\boldsymbol R(a,\boldsymbol y,\mu)
:=\boldsymbol\Pi_s(\mu)
\boldsymbol F(a\boldsymbol\phi_{\max}(\mu)+\boldsymbol y;\mu)
=\boldsymbol 0.
\end{equation*}
At $\mu=0$, this complementary subspace is the range of $\boldsymbol J(0)$. The \emph{critical projection} is $\boldsymbol\Pi_c(\mu)\boldsymbol F=\boldsymbol\phi_{\max}(\mu)(\boldsymbol\psi_{\max}(\mu)^{\mathsf T}\boldsymbol F)$. Setting it to zero is equivalent to the scalar condition
\begin{equation*}
\boldsymbol\psi_{\max}(\mu)^{\mathsf T}
\boldsymbol F(a\boldsymbol\phi_{\max}(\mu)+\boldsymbol y;\mu)=0.
\end{equation*}
We solve the range equation for $\boldsymbol y$ first, then substitute that solution into the scalar condition to determine $a$.

At the bifurcation point, $\boldsymbol R(0,\boldsymbol 0,0)=\boldsymbol 0$, and its derivative with respect to the noncritical variable is $D_{\boldsymbol y}\boldsymbol R(0,\boldsymbol 0,0)=\boldsymbol J_s$. Here $\boldsymbol J_s$ is the restriction of $\boldsymbol J(0)$ to $\ker(\boldsymbol\psi_{\max}^{\mathsf T})$. The spectral gap makes this restriction invertible, so the implicit function theorem gives a unique local smooth solution $\boldsymbol y=\boldsymbol y(a,\mu)$. Because $\boldsymbol\Pi_s(\mu)\boldsymbol J(\mu)\boldsymbol\phi_{\max}(\mu)=\boldsymbol 0$, this solution has no term linear in $a$ and begins at order $a^2$. Substituting $\boldsymbol y(a,\mu)$ into the range equation gives the leading balance
\begin{equation*}
\boldsymbol  0
=
\boldsymbol  J_s\boldsymbol  y
+
\frac{a^2}{2}
\boldsymbol \Pi_s
\mathcal B_{\mathrm{quad}}(\boldsymbol \phi_{\max},\boldsymbol \phi_{\max})
+
\mathcal O(|a|^3+|\mu|a^2).
\end{equation*}
Thus, the noncritical component of the equilibrium decision state is
\begin{equation}
\boldsymbol  y(a,\mu)
=
-\frac{a^2}{2}
\boldsymbol  J_s^{-1}\boldsymbol \Pi_s
\mathcal B_{\mathrm{quad}}(\boldsymbol \phi_{\max},\boldsymbol \phi_{\max})
+
\mathcal O(|a|^3+|\mu|a^2).
\label{eq:main_range_solution}
\end{equation}
In the leading coefficient of Equation~\eqref{eq:main_range_solution}, the eigenvector, projection, and derivative map are evaluated at $\mu=0$. For a small entry-time bias $\bar{\boldsymbol\tau}$, the same implicit-function argument gives a continuation $\boldsymbol y(a,\mu;\bar{\boldsymbol\tau})$, which reduces to $\boldsymbol y(a,\mu)$ at zero bias. Substituting this noncritical response into the scalar condition defines the \emph{critical equation}:
\begin{equation*}
g(a,\mu;\bar{\boldsymbol \tau})
:=
\boldsymbol \psi_{\max}(\mu)^{\mathsf T}
\boldsymbol  F\left(
a\boldsymbol \phi_{\max}(\mu)
+
\boldsymbol  y(a,\mu;\bar{\boldsymbol\tau});
\mu,\bar{\boldsymbol \tau}
\right)
=
0.
\end{equation*}

\subsubsection{Critical-equation coefficients and parameter dependence}
\label{sec:supp_critical_coefficients}

The linear term in the scalar critical equation is $\alpha\mu a$, where $\alpha=s_1u_\ast=S'(0)u_0/\rho_u>0$, because $\mu=\lambda_{\max}(\boldsymbol A)-\lambda^\ast$. The coefficient of $a^2$ is
\begin{equation}
\label{eq:supp_quadratic_coefficient}
\beta_2
=
\frac{1}{2}
\boldsymbol \psi_{\max}^{\mathsf T}
\mathcal B_{\mathrm{quad}}(\boldsymbol \phi_{\max},\boldsymbol \phi_{\max}).
\end{equation}
This coefficient depends on $u_\ast$ and $s_1$; on the group-tie pressures and weights through the dominant eigenvectors; on the first pressure and softmax variations through $D\omega_i^e(\boldsymbol 0)$; and on the arrival-order and decision-to-speed controller gains through $k_{\mathrm g}$, $\nu$, and $\nu^2$. Thus $T_{\mathrm{sm}}$, $k_{\mathrm{tilt}}$, and the pressure gains enter $\beta_2$ through derivatives of $\omega_i^e$. Define
\begin{equation*}
\boldsymbol  y_2
:=
-\frac{1}{2}
\boldsymbol  J_s^{-1}\boldsymbol \Pi_s
\mathcal B_{\mathrm{quad}}(\boldsymbol \phi_{\max},\boldsymbol \phi_{\max}),
\end{equation*}
so that $\boldsymbol  y(a,\mu)=a^2\boldsymbol  y_2+\mathcal O(|a|^3+|\mu|a^2)$. Substituting the noncritical coefficient $\boldsymbol y_2$ and the trilinear map in Equation~\eqref{eq:supp_cubic_tensor} into the critical equation gives the coefficient of $a^3$:
\begin{equation}
\label{eq:supp_cubic_coefficient}
\begin{aligned}
\beta_3
&=
\frac{1}{6}
\boldsymbol \psi_{\max}^{\mathsf T}
\mathcal C_{\mathrm{cub}}(\boldsymbol \phi_{\max},\boldsymbol \phi_{\max},\boldsymbol \phi_{\max})
+
\boldsymbol \psi_{\max}^{\mathsf T}
\mathcal B_{\mathrm{quad}}(\boldsymbol \phi_{\max},\boldsymbol  y_2)
\\
&=
\frac{1}{6}
\boldsymbol \psi_{\max}^{\mathsf T}
\mathcal C_{\mathrm{cub}}(\boldsymbol \phi_{\max},\boldsymbol \phi_{\max},\boldsymbol \phi_{\max})
\\
&\quad
-
\frac{1}{2}
\boldsymbol \psi_{\max}^{\mathsf T}
\mathcal B_{\mathrm{quad}}
\left(
\boldsymbol \phi_{\max},
\boldsymbol  J_s^{-1}\boldsymbol \Pi_s
\mathcal B_{\mathrm{quad}}(\boldsymbol \phi_{\max},\boldsymbol \phi_{\max})
\right).
\end{aligned}
\end{equation}
The first term projects the cubic term of $\boldsymbol F$ onto the critical direction. The second accounts for the contribution of the noncritical component $a^2\boldsymbol y_2$ through the quadratic term of $\boldsymbol F$. Accordingly, $\beta_3$ depends on the third derivative of the evidence, $S'''(0)$, $\kappa_u/\rho_u$, first and second pressure and softmax variations, $k_{\mathrm g}$ and $k_{\mathrm g}^3$, and the speed coefficients $\nu$, $\nu^2$, and $\chi$.

Continuing the expansion of the scalar critical equation gives higher-order coefficients, which determine whether additional folds occur within its neighborhood of validity. Decision-dependent attention can shift $\beta_3$ and, together with a stabilizing higher-order term, create a finite-amplitude attracting equilibrium branch and hysteresis. Such hysteresis can sustain the selected differentiated response as conflict pressures decrease. Equations~\eqref{eq:supp_quadratic_coefficient} and~\eqref{eq:supp_cubic_coefficient} give the quadratic and cubic contributions to the equilibrium branches for a specified controller. The fast subsystem time constants do not enter these equilibrium coefficients. They govern the rates of decay toward attracting equilibria and growth away from unstable equilibria, and hence whether the subsystem reaches the selected branch's attraction neighborhood before the conflict geometry or dominant interaction eigenvector changes appreciably.

\subsubsection{Critical branches and timing bias}

The equation $g(a,\mu;\bar{\boldsymbol\tau})=0$ determines the equilibrium critical amplitude $a$ for given $\mu$ and projected entry-time offsets $\bar{\boldsymbol\tau}$. Substituting Equation~\eqref{eq:main_range_solution} and the leading bias term $\eta_\tau$ gives
\begin{equation}
0
=
g(a,\mu;\bar{\boldsymbol \tau})
=
\eta_\tau
+\alpha\mu a
+\beta_2a^2
+\beta_3a^3
+\text{higher-order terms}.
\label{eq:supp_critical_equation}
\end{equation}

If $\beta_2\neq0$, the nontrivial branch at zero entry-time bias satisfies $a_*(\mu)=-\alpha\mu/\beta_2+\mathcal O(\mu^2)$. Differentiating the critical equation gives
\begin{equation*}
g_a(0,\mu;\boldsymbol 0)=\alpha\mu,
\qquad
g_a(a_*(\mu),\mu;\boldsymbol 0)=-\alpha\mu+\mathcal O(\mu^2).
\end{equation*}

The noncritical modes of the fast subsystem remain stable under the local controller design. Near the crossing, $g_a>0$ corresponds to growth of a perturbation along the critical mode, and $g_a<0$ corresponds to decay. Thus, for a small $\mu>0$ and zero projected entry-time bias, the collective-indecision equilibrium is unstable, whereas the full fast subsystem equilibrium reconstructed from $a_*(\mu)$ is locally asymptotically stable.

Each root $a^\dagger$ of the scalar critical equation, together with the range solution $\boldsymbol y(a^\dagger,\mu;\bar{\boldsymbol\tau})$, reconstructs an equilibrium decision state in Equation~\eqref{eq:main_differentiated_branch}. The higher-order terms determine how the equilibrium branches bend, saturate, and change stability. At zero projected entry-time bias, $\boldsymbol\phi_{\max}$ has components of both signs, so sufficiently small nonzero equilibria contain both positive and negative decision states. To leading order, the signs of $a^\dagger\phi_{\max,i}$ determine which agents speed up or slow down; the magnitudes of these products and the decision-to-speed gains determine their speed changes. A projected entry-time bias can also produce an equilibrium decision component of order $\|\bar{\boldsymbol\tau}\|$ in the noncritical subspace.

\paragraph{Unfolding by initial timing bias}
At an exact group tie, $\eta_\tau=0$ and collective indecision $a=0$ is an equilibrium. Infinitesimally small initial projected entry-time offsets contribute the constant unfolding term $\eta_\tau$ in Equation~\eqref{eq:supp_critical_equation}. When $\eta_\tau\neq0$, $a=0$ is no longer a root of the displayed leading-order equation, and the equilibrium branches shift and deform. The sign and magnitude of the bias, together with $\mu$ and the nonlinear coefficients, determine the resulting local branch structure. Such timing offsets arise naturally from differences in positions and speeds, asynchronous updates, and sensing and actuation variability. Near collective indecision, an arbitrarily small perturbation with a nonzero component along the unstable critical direction grows exponentially in that component.

Two illustrative scalar dynamics, $\dot a=g_+(a,\mu)$ and $\dot a=g_-(a,\mu)$, show how the sign of the cubic coefficient changes the equilibrium branches. With $\alpha=1$ and $\beta_2=-0.5$, define
\begin{equation*}
\begin{aligned}
g_+(a,\mu)&=\eta_\tau+\mu a-0.5a^2+a^3-0.3a^5,\\
g_-(a,\mu)&=\eta_\tau+\mu a-0.5a^2-a^3.
\end{aligned}
\end{equation*}
In the positive-cubic model, the negative quintic term stabilizes the outer branches; folds join stable and unstable segments and permit multiple stable responses. The negative-cubic model saturates at cubic order. At $\eta_\tau=0$, both exhibit a transcritical exchange of stability at $(\mu,a)=(0,0)$. A positive bias $\eta_\tau=0.05$ unfolds the crossing, removes the zero-amplitude equilibrium, and favors positive amplitude without necessarily eliminating competing stable equilibria. Stability in these scalar illustrations is determined by $g_a<0$. The coefficients of these scalar dynamics illustrate possible equilibrium branches and are not derived from a particular controller.

\subsection{Numerical illustration of bifurcation-driven deconfliction}
\label{sec:numerical_bifurcation_illustration}

Consider a group tie among three agents, labeled $1$, $2$, and $3$, with all three pairs in conflict. Using illustrative interaction matrices and a scalar amplitude model, we show how safety-critical interactions destabilize collective indecision and amplify a small decision asymmetry, selecting an attracting branch that determines which agent passes first in each conflict. We then follow the selected response as the $1$--$3$ conflict clears. The example connects branch selection to traversal order and shows how a cleared conflict changes the collective response.

\emph{From local interactions to a collective response.} Near the group tie, small decision changes alter speeds and relative arrival times, changing the evidence computed by each agent. With conflict-geometry and interaction topology frozen on the fast time scale, let $\delta\boldsymbol z$ denote a decision perturbation and $\delta\boldsymbol h$ its first-order effect on the evidence. Lemma~\ref{lem:jacobian_steady_state} gives $\delta\boldsymbol h=\boldsymbol A\,\delta\boldsymbol z$. The matrix $\boldsymbol A$ is $3\times3$, with rows and columns ordered by agent labels $1,2,3$. Its entry $A_{ij}$ measures the change in agent $i$'s evidence per unit change in agent $j$'s decision, including the associated speed change. For this illustration, take equal softmax denominators $s_1^\ast=s_2^\ast=s_3^\ast$ at the reference group tie, with the denominators defined in Equation~\eqref{eq:supp_spectral_factor_definitions}. Choose a common fractional-speed gain $\nu=V_i'(0)/v_i^{\mathrm{nom}}$, so decision and fractional-speed perturbations satisfy $\delta v_i/v_i^{\mathrm{nom}}=\nu\,\delta z_i$.

We collect the linearized interaction weights in a weighted adjacency matrix $\boldsymbol W$, with $W_{ii}=0$. In this specialization, the reciprocal directed interactions $e=(i,j)$ and $\bar e=(j,i)$ have equal weights, $W_{ij}=W_{ji}=k_{\mathrm g}t^{e,\ast}p^{e,\ast}w_i^{e,\ast}\nu$. This weight combines the conflict pressure $p^{e,\ast}$, the normalized softmax weight $w_i^{e,\ast}$, the shared projected entry time $t^{e,\ast}$, the arrival-order gain $k_{\mathrm g}$, and the fractional-speed gain $\nu$.

Summing the weighted decision differences gives
\begin{equation*}
\delta h_i
=\sum_{j\ne i}W_{ij}(\delta z_i-\delta z_j)
=\left(\sum_{j\ne i}W_{ij}\right)\delta z_i
-\sum_{j\ne i}W_{ij}\delta z_j.
\end{equation*}
The coefficient of a neighbor's decision change is therefore $A_{ij}=-W_{ij}$, while that of the agent's own decision change is $A_{ii}=\sum_{j\ne i}W_{ij}$. Hence $\boldsymbol A=\boldsymbol D-\boldsymbol W$, where $\boldsymbol D$ is the diagonal matrix of weighted degrees. This symmetric weighted Laplacian is a special case of the scaled weighted-Laplacian structure in Proposition~\ref{prop:interaction_sensitivity_spectrum}.

\emph{Before and after one conflict clears.} For schematic sensitivity snapshots, let superscripts $-$ and $+$ denote snapshots before and after clearance of the $1$--$3$ conflict. With all three pressures in the critical band at the reference tie, choose the initial weights $W_{12}^-=1$, $W_{13}^-=0.8$, and $W_{23}^-=0.6$, with equal weights in the reciprocal directions. We illustrate a release-stage snapshot in which $p^{12,+},p^{23,+}\geq p_{\mathrm c}$ on the unresolved conflicts, while $p^{13,+}\leq p_{\mathrm r}$ on the cleared pair. To represent its small residual influence, we reduce the $1$--$3$ weight to $W_{13}^+=W_{31}^+=0.05$ and leave the other weights unchanged. Forming $\boldsymbol A^\pm=\boldsymbol D^\pm-\boldsymbol W^\pm$ gives
\begin{equation*}
\boldsymbol A^-=
\begin{pmatrix}
1.8 & -1 & -0.8\\
-1 & 1.6 & -0.6\\
-0.8 & -0.6 & 1.4
\end{pmatrix},
\qquad
\boldsymbol A^+=
\begin{pmatrix}
1.05 & -1 & -0.05\\
-1 & 1.6 & -0.6\\
-0.05 & -0.6 & 0.65
\end{pmatrix}.
\end{equation*}
For agent~1, the first row of each matrix gives:
\begin{equation*}
\begin{aligned}
\delta h_1^-&=(\delta z_1-\delta z_2)+0.8(\delta z_1-\delta z_3),\\
\delta h_1^+&=(\delta z_1-\delta z_2) +0.05(\delta z_1-\delta z_3).
\end{aligned}
\end{equation*}
The reduced $1$--$3$ weight changes the diagonal entry from $1.8=1+0.8$ to $1.05=1+0.05$. Its nonzero residual value represents the cleared pair's continuing influence while it remains sensed.

\emph{The collective decision pattern.} The largest eigenvalue $\lambda_{\max}$ of $\boldsymbol A$ and its associated eigenvector $\boldsymbol\phi_{\max}$ connect the local interaction sensitivities to a collective decision pattern. When $\lambda_{\max}$ reaches the threshold $\lambda^\ast$, the full fast Jacobian has a zero eigenvalue, and the decision component of the corresponding critical eigenvector is proportional to $\boldsymbol\phi_{\max}$ (Lemma~\ref{lem:jacobian_steady_state}). The positive and negative components of $\boldsymbol\phi_{\max}$ (Proposition~\ref{prop:interaction_sensitivity_spectrum}) describe decision changes in opposite directions across the group.

For the two snapshots, the dominant eigenpairs are
\begin{equation}
\label{eq:supp_example_eigenpairs}
\begin{aligned}
\lambda_{\max}^-&\approx2.746,
&\boldsymbol\phi_{\max}^-&\approx(0.789,-0.577,-0.211)^{\mathsf T},
\\
\lambda_{\max}^+&\approx2.476,
&\boldsymbol\phi_{\max}^+&\approx(0.551,-0.797,0.247)^{\mathsf T}.
\end{aligned}
\end{equation}
Each eigenvector is normalized to unit norm, with its overall sign chosen so that its first component is positive. The three components correspond to agents $1$, $2$, and $3$, respectively. Since $\boldsymbol A^-$ and $\boldsymbol A^+$ are symmetric with zero row sums, their eigenvectors associated with positive eigenvalues lie in the differentiated plane $\mathcal D:=\{\boldsymbol z:z_1+z_2+z_3=0\}$. The change from $\boldsymbol\phi_{\max}^-$ to $\boldsymbol\phi_{\max}^+$ represents a rotation of the dominant direction within this plane.

To interpret this rotation, use the orthonormal basis
\begin{equation*}
\boldsymbol c=\frac{1}{\sqrt{3}}(1,1,1)^{\mathsf T},
\qquad
\boldsymbol {\xi_1}=\frac{1}{\sqrt{2}}(1,-1,0)^{\mathsf T},
\qquad
\boldsymbol {\xi_2}=\frac{1}{\sqrt{6}}(1,1,-2)^{\mathsf T}.
\end{equation*}
The common-mode vector is $\boldsymbol c$, and $\boldsymbol {\xi_1}$ and $\boldsymbol {\xi_2}$ span $\mathcal D=\boldsymbol c^\perp$. The modal coordinates are $\xi_1=(z_1-z_2)/\sqrt{2}$ and $\xi_2=(z_1+z_2-2z_3)/\sqrt{6}$. Along $\boldsymbol {\xi_1}$, agents~1 and~2 have equal and opposite decision components and agent~3 has zero component. Along $\boldsymbol {\xi_2}$, agents~1 and~2 have equal components, opposed by agent~3's component of twice their individual magnitude. In this basis, the dominant eigenvectors are
\begin{equation*}
\begin{aligned}
\boldsymbol\phi_{\max}^-&\approx0.966\boldsymbol {\xi_1}+0.259\boldsymbol {\xi_2},\\
\boldsymbol\phi_{\max}^+&\approx0.953\boldsymbol {\xi_1}-0.302\boldsymbol {\xi_2}.
\end{aligned}
\end{equation*}
Thus, weakening the $1$--$3$ interaction rotates the dominant eigenvector across the $\boldsymbol {\xi_1}$ axis and changes the sign of its $\boldsymbol {\xi_2}$ coordinate. For positive amplitude $a$, agent~3's component of $a\boldsymbol\phi_{\max}$ changes from negative to positive, while agent~2 retains the smallest decision component.

\emph{Branch selection.} Set the instability threshold to $\lambda^\ast=2.2$ and define $\mu=\lambda_{\max}-\lambda^\ast$. The two snapshots give $\mu^-\approx0.546$ and $\mu^+\approx0.276$, placing both above threshold. Let $a$ denote the amplitude of the decision pattern along $\boldsymbol\phi_{\max}$, with the sign convention in Equation~\eqref{eq:supp_example_eigenpairs}. We illustrate its evolution with the scalar model $\dot a=g(a,\mu)=a(\mu-0.65a-a^2)$. The illustrative coefficients $\alpha=1$, $\beta_2=-0.65$, and $\beta_3=-1$ give a cubic truncation of the critical equation~\eqref{eq:supp_critical_equation} with zero projected entry-time bias. The quadratic term gives the branches' unequal amplitude magnitudes, and the negative cubic term provides saturation. The controller-specific coefficients can be computed from the formulas in Supplementary Sec.~\ref{sec:supp_critical_coefficients}.

For $\mu>0$, the equilibrium $a=0$ is unstable, and the two attracting equilibria are
\begin{equation*}
a_{\mathrm{up}}(\mu)=\frac{-0.65+\sqrt{0.65^2+4\mu}}{2},
\qquad
a_{\mathrm{low}}(\mu)=\frac{-0.65-\sqrt{0.65^2+4\mu}}{2}.
\end{equation*}
In this scalar model, a small positive initial amplitude selects the upper branch, and a small negative initial amplitude selects the lower branch.

\emph{Upper branch and traversal order.} At the two snapshots, the upper-branch amplitudes are $a_{\mathrm{up}}^-\approx0.482$ and $a_{\mathrm{up}}^+\approx0.293$. The leading-order relation $\boldsymbol z^\dagger\approx a\boldsymbol\phi_{\max}$ from Equation~\eqref{eq:main_differentiated_branch} then gives
\begin{equation*}
\begin{aligned}
\boldsymbol z_{\mathrm{up}}^-&\approx(0.381,-0.279,-0.102)^{\mathsf T},
\\
\boldsymbol z_{\mathrm{up}}^+&\approx(0.161,-0.234,0.072)^{\mathsf T}.
\end{aligned}
\end{equation*}
The decision ranking is $z_1>z_3>z_2$ at both snapshots. For identical, strictly increasing fractional-speed response maps, the fractional speeds have the same ranking. With the temporal separation required by Proposition~\ref{prop:fixed_component_edge_safety}, this ranking gives the initial traversal order $1\prec3\prec2$.

Once the $1$--$3$ conflict clears, the remaining orders $1\prec2$ and $3\prec2$ allow agents~1 and~3 to progress concurrently. In the upper-branch snapshot after release, both have positive decisions, while agent~2 continues to yield. The rotation of $\boldsymbol\phi_{\max}$ and the decrease in $a_{\mathrm{up}}$ together change the direction and reduce the magnitude of the leading-order decision equilibrium. Under the equilibrium continuation and tracking conditions of Lemma~\ref{prop:smooth_dominant_eigenvector_branch} and Equation~\eqref{eq:supp_branch_tracking_error}, with positive fractional-speed gaps as in Equation~\eqref{eq:edge_fractional_speed_margin}, the decision--attention--speed states follow the evolving equilibrium while preserving the selected orders on both unresolved conflicts.

\emph{Lower-branch choice.} The lower-branch amplitudes are $a_{\mathrm{low}}^-\approx-1.132$ and $a_{\mathrm{low}}^+\approx-0.943$. The corresponding leading-order decision vectors have the ranking $z_2>z_3>z_1$ at both snapshots. Under the same response-map and temporal-clearance conditions, this gives the initial traversal order $2\prec3\prec1$. For comparison, applying the same $1$--$3$ clearance event leaves the orders $2\prec1$ and $2\prec3$, with agent~2 passing first in both remaining conflicts. The actual clearance sequence depends on the selected branch and the resulting motion.

The two branch choices illustrate how decision asymmetry selects a traversal order and how conflict release reshapes the selected response.

\clearpage
\begin{figure}[h!]
\centering
\includegraphics[width=\textwidth,height=0.52\textheight,keepaspectratio]{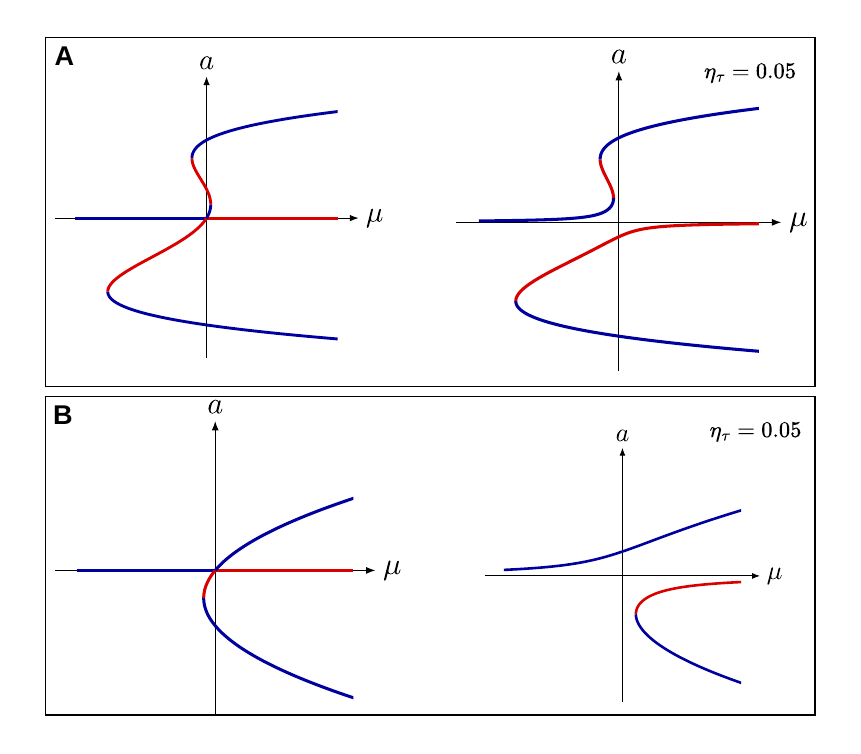}
\caption{\textbf{Bifurcation diagrams and unfolding by timing bias.} Equilibrium branches illustrate the \edit{scalar critical equation~\eqref{eq:supp_critical_equation}} with $\alpha=1$ and $\beta_2=-0.5$. \edit{The horizontal axis shows $\mu=\lambda_{\max}-\lambda^\ast$, the offset of the dominant eigenvalue of the interaction-sensitivity matrix $\boldsymbol A$ from the instability threshold $\lambda^\ast$. The vertical axis shows $a$, the signed amplitude of the differentiated decision pattern along the dominant eigenvector $\boldsymbol\phi_{\max}$; its equilibrium value $a^\dagger$ appears in Equation~\eqref{eq:main_differentiated_branch}.} Solid blue and solid red denote stable ($g_a<0$) and unstable ($g_a>0$) equilibria, respectively, where $g_a:=\partial g/\partial a$ is evaluated at each equilibrium. \textbf{(A)} The positive-cubic model $g=\eta_\tau+\mu a-0.5a^2+a^3-0.3a^5$ includes a negative quintic term that stabilizes the outer branches. Folds connect stable and unstable segments, allowing multiple stable responses to coexist. \textbf{(B)} The negative-cubic model $g=\eta_\tau+\mu a-0.5a^2-a^3$ provides cubic saturation without a quintic term. \textbf{Left column:} With timing bias $\eta_\tau=0$, the nonzero quadratic coefficient produces a transcritical exchange of stability at $(\mu,a)=(0,0)$ in both models. \textbf{Right column:} A positive timing bias $\eta_\tau=0.05$ unfolds the crossing, removes the equilibrium at $a=0$, and biases the response toward positive amplitude without eliminating all competing stable equilibria. These coefficients illustrate possible branch geometries and are not derived for a specific controller.}
\label{fig:supp_cubic_bifurcation_comparison}
\end{figure}

\clearpage

\subsection*{Notation reference}
\begingroup
\footnotesize
\linespread{1}\selectfont
\setlength{\LTcapwidth}{\textwidth}
\renewcommand{\arraystretch}{1.12}
\begin{longtable}{p{0.26\textwidth}p{0.66\textwidth}}
\caption{\textbf{Notation: agent states and controller parameters.} Bold symbols denote vectors or matrices; a dagger denotes a differentiated fast subsystem equilibrium.}\label{tab:principal_notation}\\
\hline
\textbf{Symbol} & \textbf{Definition}\\
\hline
\endfirsthead
\multicolumn{2}{l}{\textbf{Table~\thetable\ (continued)}}\\
\hline
\textbf{Symbol} & \textbf{Definition}\\
\hline
\endhead
\hline
\multicolumn{2}{r}{Continued on next page}\\
\endfoot
\hline
\endlastfoot
$i,j$; $f,\ell$ & Agent indices; instantaneous follower and leader, respectively, in an ordered interaction.\\
$\boldsymbol q_i$, $\boldsymbol v_i$ & Agent $i$'s position and velocity vector; $v_i$ is the scalar traversal speed.\\
$z_i$, $u_i$ & Agent $i$'s decision and attention states. Positive decisions favor speeding up; negative decisions favor slowing down.\\
$v_i^{\mathrm{nom}}$, $v_i^{\mathrm{cmd}}$ & Agent $i$'s nominal and commanded traversal speeds.\\
$\underline v_i$, $\overline v_i$ & Lower and upper admissible traversal speeds of agent $i$; $\overline v_i$ bounds realized motion, including tracking and actuation transients.\\
$\boldsymbol x=(\boldsymbol z,\boldsymbol u,\boldsymbol v)$ & State of the fast decision--attention--speed subsystem, comprising decisions, attention states, and realized traversal speeds.\\
$\mathcal X_i$, $\mathcal X_i^e$ & Agent $i$'s local observation state and its pairwise observations for interaction $e$.\\
$\mathcal X_{\boldsymbol z}$ & Observations evaluated at speeds $\boldsymbol v^{\mathrm{nom}}+\boldsymbol V(\boldsymbol z)$, with conflict geometry and interaction topology frozen on the fast time scale.\\
$h_i=h(\mathcal X_i)$ & Aggregate signed conflict evidence driving agent $i$'s decision.\\
$S$, $s_{\max}$ & Bounded odd decision sigmoid and its maximum absolute magnitude, $s_{\max}=\sup_x|S(x)|$.\\
$V_i(z_i)$ & Agent $i$'s decision-dependent speed adjustment, with $v_i^{\mathrm{cmd}}=v_i^{\mathrm{nom}}+V_i(z_i)$.\\
$\tau_z$, $\tau_u$, $\tau_v$ & Decision, attention, and speed-tracking time-scale parameters.\\
$\rho_z$, $\rho_u$ & Decision and attention damping parameters; the corresponding relaxation times are $\tau_z/\rho_z$ and $\tau_u/\rho_u$.\\
$u_0$, $\kappa_u$ & Basal attention input and decision-dependent attention reinforcement gain; basal attention is $u_0/\rho_u$.\\
$\kappa_V$, $\kappa_z$ & Maximum fractional-speed departure and decision sensitivity in the saturating speed-command law.\\
$z_{\max}$, $z_{\mathrm{loc}}$ & Bound $z_{\max}=s_{\max}/\rho_z$ defining the invariant interval $[-z_{\max},z_{\max}]$, and radius of the prescribed local decision interval.\\
$z_f^{\mathrm{stop}}$ & Follower decision threshold at or below which a stopping command is available.\\
\end{longtable}

\begin{longtable}{p{0.26\textwidth}p{0.66\textwidth}}
\caption{\textbf{Notation: interaction structure and conflict evidence.}}\label{tab:notation_interactions}\\
\hline
\textbf{Symbol} & \textbf{Definition}\\
\hline
\endfirsthead
\multicolumn{2}{l}{\textbf{Table~\thetable\ (continued)}}\\
\hline
\textbf{Symbol} & \textbf{Definition}\\
\hline
\endhead
\hline
\multicolumn{2}{r}{Continued on next page}\\
\endfoot
\hline
\endlastfoot
$\mathcal N_i(t)$, $\mathcal E_i(t)$ & Sensed neighbors and sensed directed interactions of agent $i$; $\mathcal E_i=\{(i,j):j\in\mathcal N_i\}$.\\
$e=(i,j)$, $\bar e=(j,i)$ & Reciprocal directed interactions. Under the edge convention $i\rightarrow j$, agent $j$'s motion influences agent $i$'s evidence.\\
$\mathcal P(t)$, $\mathcal E^{\mathrm{conf}}(t)$ & Unordered physical-conflict pairs and their associated directed interactions. These exclude cleared pairs, even if still sensed.\\
$\mathcal G(t)$, $\mathcal V(t)$ & Weighted directed physical-conflict graph and its agent set.\\
$\mathcal C(t)$, $\mathcal G^c(t)$ & Index set of strongly connected components of the conflict graph and component $c$. Each component is a locally engaged group.\\
$\mathcal V^c$, $\mathcal P^c$ & Agents and unordered physical-conflict pairs in group $c$.\\
$M_{\max}$ & Uniform bound on each agent's number of sensed interactions.\\
$p_i^e=p^e=P(\mathcal X_i^e)$ & Pair-invariant conflict pressure; $p_i^e=p_j^{\bar e}$ and $0<p^e\leq\bar p<1$ for the chosen pressure map.\\
$\bar p$, $p_{\mathrm c}$, $p_{\mathrm w}$, $p_{\mathrm r}$ & Uniform pressure upper bound, critical-pressure lower bound, weak-pressure upper bound, and release-pressure upper bound; $0<p_{\mathrm r}\leq p_{\mathrm w}<p_{\mathrm c}\leq\bar p<1$.\\
$\sigma(x;k)$ & Logistic sigmoid $(1+e^{-kx})^{-1}$ with gain $k>0$, used in the pressure map.\\
$\hat{p}^e$ & Arrival-order-tilted pressure for directed interaction $e=(i,j)$: $p_i^e[1-k_{\mathrm{tilt}}g_i^e/(1-p_i^e)]$, used as the softmax score.\\
$w_i^e=W(\mathcal X_i^e;\mathcal X_i)$ & Normalized importance of interaction $e$ among agent $i$'s sensed interactions.\\
$g_i^e=G(\mathcal X_i^e)$ & Signed arrival-order value: positive favors agent $i$ proceeding, negative favors yielding, and zero denotes a projected entry-time tie.\\
$\omega_i^e=p^e w_i^e$ & Agent $i$'s directed interaction weight on $e$; the signed evidence contribution is $\omega_i^e g_i^e$.\\
$k_{\mathrm d}$, $k_{\mathrm t}$, $k_{\mathrm g}$ & Spatial-pressure, temporal-pressure, and arrival-order sigmoid gains.\\
$T_{\mathrm{sm}}$, $k_{\mathrm{tilt}}$ & Softmax temperature and arrival-order pressure-tilt gain.\\
\end{longtable}

\begin{longtable}{p{0.26\textwidth}p{0.66\textwidth}}
\caption{\textbf{Notation: conflict geometry, timing, and safety.}}\label{tab:notation_geometry}\\
\hline
\textbf{Symbol} & \textbf{Definition}\\
\hline
\endfirsthead
\multicolumn{2}{l}{\textbf{Table~\thetable\ (continued)}}\\
\hline
\textbf{Symbol} & \textbf{Definition}\\
\hline
\endhead
\hline
\multicolumn{2}{r}{Continued on next page}\\
\endfoot
\hline
\endlastfoot
$d^{\mathrm{safe}}$, $R_{\mathrm{sens}}$ & Required minimum spatial separation and sensing radius.\\
$\boldsymbol q_{\mathrm c}^e$, $\alpha^e$ & Projected crossing point and crossing angle.\\
$r_{\mathrm c}^e$ & Half-length of a contested segment, $d^{\mathrm{safe}}/\sin\alpha^e$.\\
$d_{\mathrm{ca}}^e$, $t_{\mathrm{ca}}^e$ & Projected closest-approach distance and time.\\
$d_{\mathrm{prox}}$, $t_{\mathrm{prox}}$ & Spatial and temporal proximity buffers in the pressure map.\\
$d_{k^-}^e$ & Agent $k$'s remaining distance to its contested-segment entry boundary.\\
$t^e_{k^-}$, $t^e_{k^+}$ & Projected entry and exit times for agent $k$'s contested segment.\\
$t^e_{k^+}-t_{\mathrm{ca}}^e$ & Projected exit-time offset relative to closest approach; its rate is $\dot t^e_{k^+}-\dot t_{\mathrm{ca}}^e$.\\
$L_{\mathrm d}$, $L_{\mathrm t}$ & Upper bounds on $|\dot d_{\mathrm{ca}}^e|$ and $|\dot t^e_{k^+}-\dot t_{\mathrm{ca}}^e|$, respectively.\\
$L_p(k_{\mathrm d},k_{\mathrm t})$, $\varepsilon_p$ & Bound on the pressure rate $|\dot p^e|$ and allowed pressure change during negotiation, respectively.\\
$\mathcal I_k^e$ & Projected contested-segment occupancy interval $[t^e_{k^-},t^e_{k^+}]$.\\
$i\prec_e j$ & Traversal order in which agent $i$ clears its contested segment before agent $j$ enters its own.\\
$\pi_{i\prec j}^e$ & Temporal-clearance gap $t^e_{j^-}-t^e_{i^+}$ between agent $i$ leaving its contested segment and agent $j$ entering its own.\\
$\Delta_t^e$ & Total temporal-clearance design buffer for nominal or equilibrium predictions; strictly exceeds the certified maximum loss of temporal separation.\\
$\mathcal V_{\mathrm{safe}}^c$, $\mathcal Z_{\mathrm{safe}}^c$ & Group speed and decision sets with positive pairwise temporal-clearance gaps.\\
$T^{\mathrm{neg}}$, $\mathcal I^{\mathrm{neg}}$ & Negotiation horizon for establishing a traversal order and interval $[t_0,t_0+T^{\mathrm{neg}}]$, starting at conflict detection $t_0$.\\
$T_{\mathrm{fast}}$ & Characteristic response time of the decision--attention--speed subsystem, $\max\{\tau_z/\rho_z,\tau_u/\rho_u,\tau_v\}$.\\
$\varepsilon_{\mathrm{trav}}$ & Chosen small upper bound on the fraction of each agent's remaining distance to its contested segment traveled during negotiation; sets the detection lead needed for time-scale separation.\\
$T_{\mathrm{trav}}^e$, $T_{\mathrm{trav}}$ & Lower bound on the time from detection until either agent reaches its contested segment: $T_{\mathrm{trav}}^e=\min_{k\in\{i,j\}}d_{k^-}^e(t_0)/\overline v_k$ for a pair; $T_{\mathrm{trav}}$ is the corresponding group minimum.\\
$t_{\mathrm{safe}}^e(a)$ & Sufficient lower bound on $d_{f^-}^e/v_f^{\mathrm{nom}}$ at order formation for the selected equilibrium speeds to provide temporal separation $\Delta_t^e$ and positive clearance after errors.\\
$t_{\mathrm{safe}}^{e,k}$ & $t_{\mathrm{safe}}^e(a)$ evaluated on the candidate branch where $k$ follows; uncertainty allowances are included through $\Delta_t^e$.\\
$t_{\mathrm{req}}^e$ & Lower bound~\eqref{eq:complete_required_detection_time} on nominal detection lead, combining negotiation, yielding, stopping, temporal clearance, and time-scale separation.\\
\end{longtable}

\begin{longtable}{p{0.26\textwidth}p{0.66\textwidth}}
\caption{\textbf{Notation: spectral and bifurcation quantities.}}\label{tab:notation_spectral}\\
\hline
\textbf{Symbol} & \textbf{Definition}\\
\hline
\endfirsthead
\multicolumn{2}{l}{\textbf{Table~\thetable\ (continued)}}\\
\hline
\textbf{Symbol} & \textbf{Definition}\\
\hline
\endhead
\hline
\multicolumn{2}{r}{Continued on next page}\\
\endfoot
\hline
\endlastfoot
$\boldsymbol A$ & Sensitivity of conflict evidence to decision-induced speed changes at frozen conflict geometry and interaction topology; not a frozen interaction-weight matrix.\\
$t^{e,\ast}$, $\bar t$, $\underline t_{\mathrm c}$ & Common projected time until the agents reach their respective contested segments at a group tie, its absolute upper bound, and its positive lower bound for safety-critical reference observations, respectively.\\
$\lambda_{\max}$, $\boldsymbol\phi_{\max}$ & Dominant eigenvalue of $\boldsymbol A$ and associated eigenvector defining the leading-order collective decision pattern.\\
$\lambda^\ast$ & Spectral threshold for loss of stability of collective indecision.\\
$\mu=\lambda_{\max}-\lambda^\ast$ & Distance from the spectral threshold in the local critical equation.\\
$a$, $\boldsymbol x^\dagger$ & Critical-mode amplitude and differentiated equilibrium of the fast subsystem with state $(\boldsymbol z,\boldsymbol u,\boldsymbol v)$.\\
$\eta_\tau$ & Unfolding term contributed by projected entry-time offsets in the local critical equation.\\
$\vartheta_i^\dagger(a)$ & Equilibrium fractional speed $v_i^\dagger(a)/v_i^{\mathrm{nom}}$.\\
$\nu_i$, $\eta_i$ & Agent $i$'s local fractional-speed gain $V_i'(0)/v_i^{\mathrm{nom}}$ and critical-mode coefficient $\nu_i\phi_{\max,i}$; $\eta_i$ is distinct from $\eta_\tau$.\\
$\delta_\vartheta(a)$ & Minimum absolute equilibrium fractional-speed difference across unresolved physical-conflict pairs in a group.\\
\end{longtable}
\endgroup

\clearpage
\section*{{Captions for supplementary movies}}


\noindent{\textbf{Movie S1. Antipodal exchange in simulation.}
Simulated agents approach from positions around a circle and traverse a shared central region. The left panel shows agent motion and recent trajectories; the upper and lower right panels show decision and speed histories, respectively, over the displayed simulation time.}
\par\medskip

\noindent{\textbf{Movie S2. Multiple intersecting paths in simulation.}
Simulated agents traverse intersecting paths from multiple approach directions. The left panel shows agent motion and recent trajectories; the upper and lower right panels show the corresponding decision and speed histories, respectively, over the displayed simulation time.}
\par\medskip

\noindent{\textbf{Movie S3. A larger group traversing intersecting paths in simulation.}
Simulated agents distributed along the sides of a shared workspace move through its interior on intersecting paths. The left panel shows agent motion and recent trajectories; the upper and lower right panels show the corresponding decision and speed histories, respectively, over the displayed simulation time.}
\par\medskip

\noindent{\textbf{Movie S4. Crowded passage through a shared open area in simulation.}
Simulated agents initially distributed around an open area move through its shared interior and disperse toward the perimeter. The left panel shows agent motion and recent trajectories; the upper and lower right panels show the corresponding decision and speed histories, respectively, over the displayed simulation time.}
\par\medskip

\noindent{\textbf{Movie S5. Crossing and turning paths in a simulated road network.}
Simulated agents traverse a road network along straight and turning paths. The left panel shows agent motion and recent trajectories; the upper and lower right panels show the corresponding decision and speed histories, respectively, over the displayed simulation time.}
\par\medskip

\noindent{\textbf{Movie S6. Antipodal crossing with six robots, trial 1.}
Six TurtleBot3 robots use neuromorphic decision dynamics (NDD) to adjust forward speed along prescribed paths during an antipodal crossing. The robots coordinate passage without inter-robot communication or preassigned priorities. The nominal forward speed is $0.2\,\mathrm{m/s}$.}
\par\medskip

\noindent{\textbf{Movie S7. Antipodal crossing with six robots, trial 2.}
A second six-TurtleBot3 antipodal crossing trial shows coordinated passage under NDD. Each robot adjusts its forward speed along a prescribed path without inter-robot communication or preassigned priorities. The nominal forward speed is $0.2\,\mathrm{m/s}$.}
\par\medskip

\noindent{\textbf{Movie S8. Orthogonal crossings with six robots.}
Six TurtleBot3 robots negotiate orthogonal crossings using NDD. Continuous adjustments of forward speed coordinate passage along prescribed paths without inter-robot communication. The nominal forward speed is $0.2\,\mathrm{m/s}$.}
\par\medskip

\noindent{\textbf{Movie S9. Mixed-lane crossings with six robots.}
Six TurtleBot3 robots use NDD to coordinate passage through multiple intersections in a mixed-lane arrangement. The robots adjust forward speed along prescribed paths without inter-robot communication. The nominal forward speed is $0.2\,\mathrm{m/s}$.}
\par\medskip

\end{document}